\documentclass{article}

\usepackage{PRIMEarxiv}

\usepackage[utf8]{inputenc} 
\usepackage[T1]{fontenc}    
\usepackage{hyperref}       
\usepackage{url}            
\usepackage{booktabs}       
\usepackage{amsfonts}       
\usepackage{nicefrac}       
\usepackage{microtype}      
\usepackage{lipsum}
\usepackage{fancyhdr}       
\usepackage{graphicx}       
\graphicspath{{media/}}     
\usepackage{amsmath}
\usepackage{mathtools}
\usepackage{longtable}
\usepackage{makecell}
\usepackage{tabu}

\newcommand{\indep}{\perp \!\!\! \perp}

\DeclareMathOperator*{\argmin}{arg\,min}

\title{Evaluation of Active Feature Acquisition Methods for Static Feature Settings
}

\author{
\name \hspace{-5pt}  Henrik von Kleist$^{1,2,3}$ \email henrik.vonkleist@helmholtz-munich.de \\
       \name Alireza Zamanian$^{2,4}$ \email alireza.zamanian@iks.fraunhofer.de \\
      \name Ilya Shpitser$^{3}$ \email ishpits1@jhu.edu
      \\
        \name Narges Ahmidi$^{1,3,4}$ \email narges.ahmidi@helmholtz-munich.de
      \\
      \addr $^1$Institute of Computational Biology, Helmholtz Munich - German Research Center for Environmental Health,
Neuherberg, Germany
\\
\addr $^2$TUM School of Computation, Information and Technology, Technical University of Munich, 
Garching, Germany
\\
\addr $^3$Department of Computer Science, 
Johns Hopkins University
Baltimore, Baltimore, MD, 
USA
\\
\addr $^4$Fraunhofer Institute for Cognitive Systems IKS, 
Munich, Germany
      }

\begin{document}

\maketitle

\begin{abstract}%
Active feature acquisition (AFA) agents, crucial in domains like healthcare where acquiring features is often costly or harmful, determine the optimal set of features for a subsequent classification task. As deploying an AFA agent introduces a shift in missingness distribution, it's vital to assess its expected performance at deployment using retrospective data.
In a companion paper, we introduce a semi-offline reinforcement learning (RL) framework for active feature acquisition performance evaluation (AFAPE) where features are assumed to be time-dependent. Here, we study and extend the AFAPE problem to cover static feature settings, where features are time-invariant, and hence provide more flexibility to the AFA agents in deciding the order of the acquisitions. 
In this static feature setting, we derive and adapt new inverse probability weighting (IPW),  direct method (DM), and  double reinforcement learning (DRL) estimators within the semi-offline RL framework. These estimators can be applied when the missingness in the retrospective dataset follows a 
missing-at-random (MAR) pattern. They also can be applied to missing-not-at-random (MNAR) patterns in conjunction with appropriate existing missing data techniques. We illustrate the improved data efficiency offered by the semi-offline RL estimators in synthetic and real-world data experiments under synthetic MAR and MNAR missingness.
\end{abstract}


\section{Introduction}

Machine learning (ML) methods generally assume the ready availability of the complete set of input features at deployment, typically incurring little to no cost. 
However, this assumption does not hold universally, especially in scenarios where feature acquisitions are associated with substantial costs. 
In contexts like medical diagnostics, the cost of acquiring certain features, such as X-rays, biopsies, etc. encompasses not only financial costs but also poses potential risks to patient well-being. 
In such cases, the cost or harm of the feature acquisition should be balanced against the predictive value of the feature. 

Active 
feature acquisition (AFA) 
addresses this problem by training two AI components: 
i) the "AFA agent," an AI system tasked with determining which features should be observed, and 
ii) an ML prediction model that undertakes the prediction task based on the acquired feature set.  
While missingness was effectively determined by, for example, a physician during the acquisition of the retrospective dataset, the missingness at the deployment of the AFA agent is determined by the AFA agent, thereby leading to a missingness distribution shift.

In our companion paper \cite{von_kleist_evaluation_2023}, we formulate the problem of active feature acquisition performance evaluation (AFAPE) which addresses the task of estimating the performance an AFA agent would have at deployment, from the retrospective dataset. 
Consequently, upon completing the AFAPE problem, the physician will be well-informed about expected rates of incorrect diagnoses and the average costs associated with feature acquisitions when the AFA system is put into operation.

\subsection{Paper Goal} 
In this paper, we address,  as in the companion paper, the AFAPE problem.
We do so, however, by employing an additional static feature assumption. 
In the static feature setting, feature values do not change over time. This allows the agent to wait and consider the optimal order of acquisitions. Additionally, we assume said order is not known in the retrospective dataset, but only the overall acquired set of features is known. 
We also investigate how different missingness assumptions including the missing-at-random (MAR) and the missing-not-at-random (MNAR) assumptions affect the AFAPE problem. Note that these missingness assumptions always refer to the missingness present in the retrospective dataset.  
The paper aims to achieve two main objectives:
i) Identification, which entails pinpointing the assumptions that enable the unbiased estimation of costs and prediction performance from retrospective data; and 
ii) Estimation, which is centered on delivering accurate and precise cost estimates.

\subsection{Paper Outline and Contributions} 

We start the remainder of this paper by discussing relevant methods and background in Section \ref{sec_related_methods}
and define the AFAPE problem for static feature configurations in Section \ref{sec_afape}.
We then show in Section \ref{sec_missing_data} that the AFAPE problem can, also in static feature settings, be solved from a missing data (+online RL) viewpoint. 

In Section \ref{sec_semi_offline_RL}, we assume a special MAR scenario and extend the semi-offline RL viewpoint which was developed in the companion paper \cite{von_kleist_evaluation_2023} for general time-series settings, to the static feature setting. 
In the semi-offline RL setting, a retrospective data point is used to simulate the feature acquisition process of the AFA agent. The name "semi-offline RL" arises, because the AFA agent is allowed to try out different acquisitions (the online part), but the acquisition of features that are missing in the retrospective dataset is blocked (the offline part). The semi-offline RL viewpoint drastically reduces positivity requirements (i.e. requirements for which patterns of missingness exist in the data) compared to the missing data view. 
We further extend the semi-offline RL versions of the direct method (DM) \cite{von_kleist_evaluation_2023, levine_offline_2020}, inverse probability weighting (IPW) \cite{von_kleist_evaluation_2023,levine_offline_2020}, and the double reinforcement learning (DRL) estimator \cite{von_kleist_evaluation_2023,kallus_double_2020}.
The DRL estimator maintains its consistency, even in the presence of misspecifications in either the underlying Q-function or the propensity score model.

In Section \ref{sec_MNAR}, we extend the semi-offline RL framework to MNAR settings. We propose a hybrid missing data + semi-offline RL viewpoint that allows the trading of the benefits of both views.

In Section \ref{sec_experiments},
we demonstrate the improved data efficiency 
of the semi-offline RL estimators in synthetic and real-world data experiments with exemplary synthetically induced MAR, and MNAR missingness.

\section{Background and Related Methods}
\label{sec_related_methods}

In the following, we review the general AFA literature, how AFA methods have been evaluated previously, and give some general background on semi-parametric theory. 

\subsection{Active feature acquisition (AFA)} 

Research on active feature acquisition (AFA) and related problem formulations is largely scattered over different research communities.  Initially, scholars in the fields of economics and decision science explored the concept of "Value of Information" (VoI) \cite{lavalle_cash_1968,lavalle_cash_1968-1,gould_risk_1974,hilton_determinants_1979,hess_risk_1982,keisler_value_2014}.
Furthermore, similar concepts have been used to assess the cost-effectiveness of screening policies in the public health literature \cite{mushlin_is_1992,krahn_screening_1994,botteman_health_2003,force_screening_2009}. 
Additionally, AFA has been examined under the names "dynamic testing regimes" \cite{liu_efficient_2021, robins_estimation_2008} and "dynamic monitoring regimes" \cite{neugebauer_identification_2017,kreif_exploiting_2021} in the causal inference literature, although not with the task of feature acquisition for optimal prediction, but for optimal treatment recommendations. The applied causal inference concepts have, to our knowledge, not been extended to regular AFA and neither to static feature settings. They are thus not directly comparable with the semi-offline RL viewpoint developed in this paper. For a comparison with semi-offline RL in time-series AFA settings, see our companion paper \cite{von_kleist_evaluation_2023}.

The term "active feature acquisition" (AFA) is commonly found in the machine learning literature, yet alternative terms are also frequently employed. These encompass, among others, "active sensing" \cite{yoon_asac_2019, yoon_deep_2018, tang_adversarial_2020, jarrett_inverse_2020}, "active feature elicitation" \cite{natarajan_whom_2018,das_clustering_2021},"dynamic feature acquisition" \cite{li_dynamic_2021}, "dynamic active feature selection" \cite{zhang_novel_2019}, "element-wise efficient information acquisition" \cite{gong_icebreaker_2019}, "classification with costly features" \cite{janisch_classification_2020}, and "test-cost sensitive classification" \cite{xiaoyong_chai_test-cost_2004}.

This paper does not focus on the training of an AFA agent but on the evaluation of \textit{any} AFA agent under the inevitable missingness distribution shift at deployment. Nevertheless, we provide a short overview of some common approaches to train AFA agents in Appendix \ref{app_AFA_methods}. These encompass greedy information-theoretic and reinforcement learning-based approaches.

\subsection{Active Feature Acquisition Performance Evaluation (AFAPE)}

Our companion paper \cite{von_kleist_evaluation_2023} introduces the active feature acquisition performance evaluation (AFAPE) problem for the first time and shows three viewpoints that can be taken to solve AFAPE. In the following, we categorize the ways AFA performances have been previously reported into these categories. 

\textit{Offline RL view: } 
Offline RL can be applied to solve AFAPE, especially when feature measurements affect the underlying feature values \cite{von_kleist_evaluation_2023}. In this paper, we do, however, assume the absence of these effects (known as the no direct effect (NDE) assumption \cite{von_kleist_evaluation_2023, robins_estimation_2008}). Under the NDE assumption, offline RL will be inefficient and impose strong positivity requirements \cite{von_kleist_evaluation_2023}. As we additionally assume the order of acquisitions is not known (and thus the individual actions), it is not possible to apply offline RL in our static feature setting. 
Nevertheless, offline RL has still been applied in static feature AFA settings \cite{chang_dynamic_2019,cheng_optimal_2018}.
Chang \textit{et. al} \cite{chang_dynamic_2019}, for example, resolved the problem through the assumption of random acquisition orders, which does, however, only lead to unbiased estimation under the strong MCAR assumption.

\textit{Missing data + online RL view:}
Our companion paper, also establishes that a missing data view can be applied to solve AFAPE \cite{von_kleist_evaluation_2023}. We extend this viewpoint to the static feature setting and allow more complex MNAR scenarios in this paper. 
The missing data view has also been applied in AFA settings before, but to our knowledge only in the form of (conditional) mean imputation \cite{an_reinforcement_2022, erion_coai_2021, janisch_classification_2020}, which leads to estimation bias \cite{von_kleist_evaluation_2023}.

\textit{Semi-offline RL view:}
In the companion paper \cite{von_kleist_evaluation_2023}, we propose a novel semi-offline RL framework to solve AFAPE. We extend this framework in this paper to static feature settings. Semi-offline RL involves the simulation of an online interaction of the agent with its environment, but certain actions, where the underlying feature is not available in the retrospective dataset, are blocked. This blocking setup has been previously applied in AFA settings \cite{janisch_classification_2020, yoon_deep_2018} (including static feature settings), but the resulting bias, introduced by the blocking operation, has not been corrected before.


\subsection{Semi-parametric Theory}

The objective of AFAPE is to assess the anticipated costs and prediction performance associated with the deployment of an AFA system. In broader terms, this involves estimating a target (performance) parameter, denoted as $J = J(p)$, from an unknown distribution $p$, using a set of observed samples derived from this distribution (referred to as the retrospective dataset).
In the realm of semi-parametric theory, the aim is to identify suitable estimators for this target parameter $J$ while allowing at least some part of the data-generating process $p$ to be unrestricted. This flexibility of the data generating process allows for more more reliable and trustworthy estimates. 
For a more comprehensive overview of these concepts, please refer to Appendix \ref{app_semiparametric_theory}.

\section{AFAPE Problem Definition}
\label{sec_afape}

This section gives an introduction to the mathematical notation used in the context of the AFA setting and the AFAPE problem. A glossary of important terms and variables is shown in Appendix \ref{Appendix_glossary}.

\begin{figure}
\centering
\includegraphics[width=0.8\textwidth]
{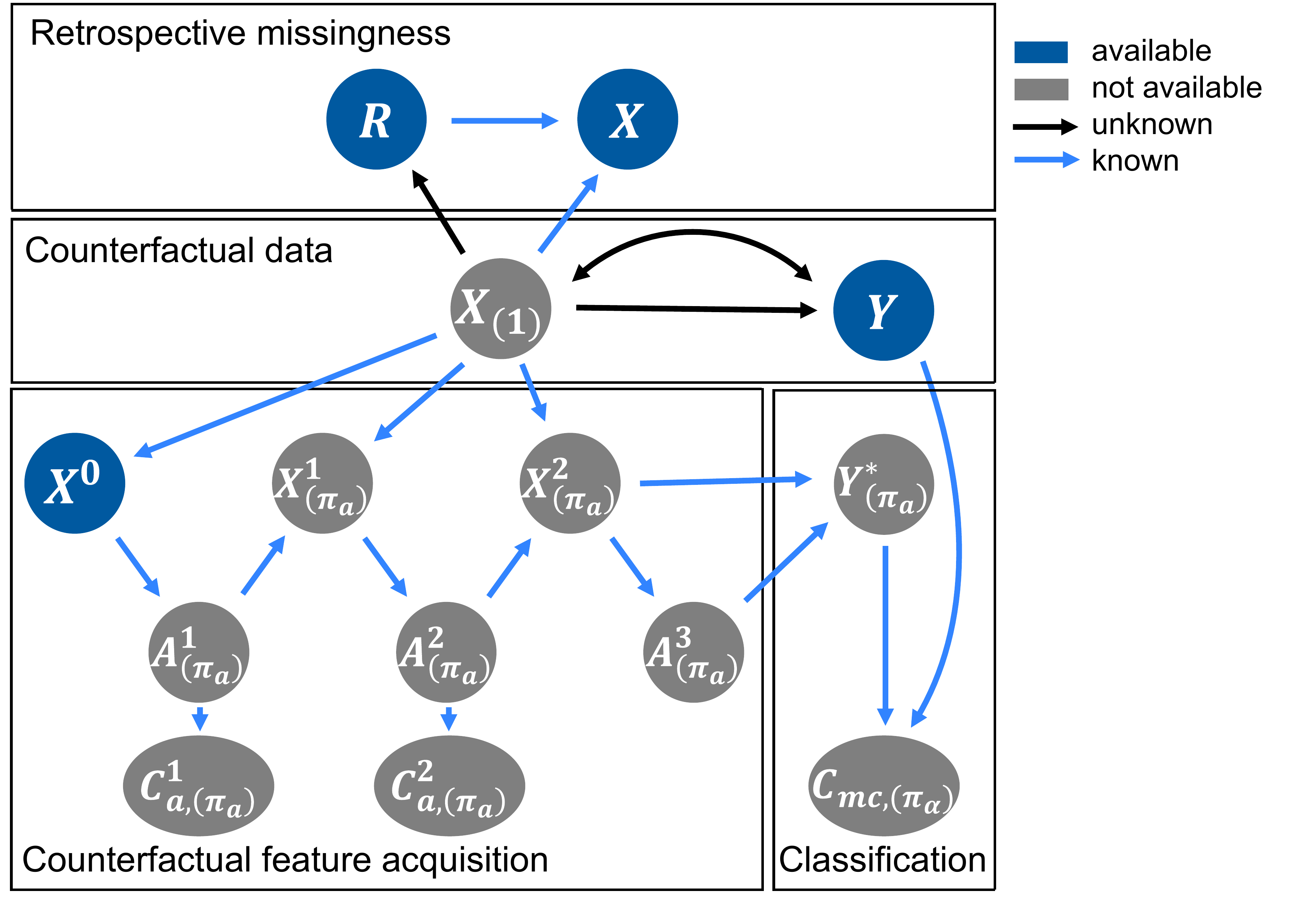}
\caption{ 
The causal graph depicting the AFA setting under a static feature assumption. 
The graph depicts the retrospective dataset $\mathcal{D}$, given by missingness indicators $R$, observed features $X$, and the label $Y$. It further shows the counterfactual feature acquisition process under the acquisition decisions of an AFA policy $\pi_\alpha$: Based on a set of always observed, free features $X^0$, $\pi_\alpha$ would have determined the first counterfactual acquisition decision $A^1_{(\pi_\alpha)}$ and produced an acquisition cost $C_{a,(\pi_\alpha)}^1$. The corresponding feature $X^1_{(\pi_\alpha)}$ would have been revealed (from the total set of counterfactual features $X_{(1)}$). After a certain number of such acquisitions, the acquisition process would have concluded with the action $A^T_{(\pi_\alpha)}$ = "stop \& predict" and a prediction $Y^*_{(\pi_\alpha)}$ would have been performed. A mismatch between $Y$ and the $Y^*_{(\pi_\alpha)}$ would have resulted in a misclassification cost $C_{mc,(\pi_\alpha)}$. 
The following edges showing long-term dependencies were excluded from the graph for visual clarity: 
$\underline{X}^{t-1}_{(\pi_\alpha)}, \underline{A}^{t-1}_{(\pi_\alpha)} \rightarrow  A^{t}_{(\pi_\alpha)}$ and $\underline{X}^{T-1}_{(\pi_\alpha)}, \underline{A}^{T-1}_{(\pi_\alpha)} \rightarrow  Y^*_{(\pi_\alpha)}$. }
\label{graph_general}
\end{figure}

\subsection{Retrospective Data}
The available retrospective dataset $\mathcal{D}$ contains always observed (categorical) labels $Y \in \{0,..., Y_{K-1} \}$, measured feature values $X \in ( \mathbb{R} \cup \{ "?" \} )^{d_x}$ (where $"?"$ denotes a special value to represent that a certain feature was not acquired) and missingness indicators $R \in \{0,1\}^{d_x}$. We assume no measurement error and that measurements do not affect the underlying feature values (known as the no direct effect assumption (NDE) \cite{von_kleist_evaluation_2023,robins_estimation_2008}. This establishes the following relationship between the underlying counterfactual feature values $X_{(1)}$, the measured features $X$, and the missingness indicators $R$: 
\begin{align}
\label{eq_consistency}
X_i =
\begin{cases}
X_{(1),i} & \text{ if } R_i = 1\\
"?" & \text{ if } R_i = 0.
\end{cases}
\end{align}
\noindent 
$X_{(1)}$ denotes the potential outcome of $X$, had $R= \vec{1}$ been true (possibly counter to the fact). 
The missingness mechanism is denoted by $\pi_\beta(R|X_{(1)})$. We assume the label $Y$ does not have an effect on $R$ (i.e. $A \not \rightarrow R$), but the developed concepts can be easily extended to allow for such an effect. 
The retrospective data is visualized in the top and middle parts of the causal graph in Figure \ref{graph_general}.
The graph shows an arrow  $X_{(1)} \rightarrow Y$, but the results in this paper also hold under a reversed causal relationship. We allow for additional unobserved confounding between $X_{(1)}$ and $Y$ (depicted as $X_{(1)}\leftrightarrow Y$), but not between $X_{(1)}$ and $R$.

\subsection{AFA Process}

In AFA, we are interested in what would have happened in the acquisition process of the dataset, had, instead of the missingness mechanism $\pi_\beta$, an AFA agent, characterized by the AFA policy $\pi_\alpha$, decided which features to acquire. We denote the counterfactual acquisition actions and observed features as $A_{(\pi_\alpha)}$ and $X_{(\pi_\alpha)}$, respectively. The AFA policy does not decide on all feature acquisitions at once but goes through a step-by-step feature acquisition process (Figure \ref{graph_general} bottom left). The AFA agent would have started with a set of always observed (free) features $X^0$ and would have chosen the first acquisition action $A^1_{(\pi_\alpha)}$. A specific feature acquisition cost $C^1_{a, (\pi_\alpha)}$ would have been produced. The corresponding feature value $X^1_{(\pi_\alpha)}$ would have been revealed to the agent, and the next acquisition action $A^2_{(\pi_\alpha)}$ would have been chosen. The acquisition process would have ended at step $T$ once the AFA policy would have chosen to stop acquisitions. 

An acquisition action $A^t =i$ (with $A^t \in \{1,..., d_x, d_x+1\}$ and $t \in \{1, ..., T\}$) defines which feature $X_{(1),i}$ will observed at step $t$ (if $i \leq d_x$). The action $A^t = d_x+1$ represents a "stop \& predict" action which concludes the acquisition process and initializes the classification. We further let $R_{(\pi_\alpha)}^t$ denote the set of acquired features up to time t (i.e. the set version of $A_{(\pi_\alpha)}^1, ..., A_{(\pi_\alpha)}^t$).  
We define $\underline{X}^{t} = \{X^0, ..., X^t\}$ as the set of acquired features until step $t$, and $\overline{X}^{t} = \{X^t, ..., X^{T-1}\}$ as the set of acquired features from $t$ until the end of the acquisition process (and similarly for other variables). The total set of acquired features can thus be written as $X = 
\underline{X}^{T-1} = \overline{X}^{0}$. 
The AFA policy can now be defined as the distribution $\pi_\alpha(A^t| \underline{X}^{t-1}, \underline{A}^{t-1})$ which decides on the next acquisition based on all previously acquired features. We assume $\pi_\alpha(A^t| \underline{X}^{t-1}, \underline{A}^{t-1}) = \pi_\alpha(A^t| \underline{X}^{t-1}, R^{t-1})$ as it does not make sense to choose the next feature based on the order of previously acquired features.  

\subsection{Classification Process}

Once the AFA process would have concluded, a classification of the label $Y$ would have been performed based on the acquired set of features. If the true label $Y$ would have differed from its prediction $Y^*_{(\pi_\alpha)}$, a misclassification cost $C_{mc,(\pi_\alpha)}$ would have been produced (Figure \ref{graph_general} bottom right). We denote the (deterministic) classifier used for the prediction as $g(Y^*| \underline{X}^{T-1},\underline{A}^T)$. Throughout the paper we will denote (known) deterministic distribution by $g(.)$. 
Similarly as for the AFA policy, we assume 
$g(Y^*| \underline{X}^{T-1},\underline{A}^T) = g(Y^*| \underline{X}^{T-1}, R^T)$ as the optimal prediction of the label does not depend on the order of feature acquisitions.

\subsection{Problem Definition: Active Feature Acquisition Performance Evaluation (AFAPE)}

When provided with an AFA policy $\pi_{\alpha}(A^t|\underline{X}^{t-1},\underline{A}^{t-1})$ and a classifier $g(Y^* |\underline{X}^{T-1}, \underline{A}^T)$, the objective of AFAPE is to calculate the expected acquisition and misclassification costs that would occur if the AFA policy $\pi_{\alpha}$ and the classifier $g$ were deployed. The estimation challenge for this expected counterfactual cost can be framed as the task of estimating 

\begin{equation}
\label{eq:AFAPE_objective}
J_a = \mathbb{E}\left[\sum_{t=1}^T C_{a, (\pi_\alpha)}^t  \right], \text{ and } J_{mc} = \mathbb{E}\left[ C_{mc, (\pi_\alpha)}  \right],
\end{equation}

The objective of this paper is twofold:
i) To perform identification, meaning to establish the conditions under which it becomes feasible to estimate the target parameters $J_a$ and $J_{mc}$ from the retrospective dataset; and 
ii) To develop unbiased estimators for $J_a$ and $J_{mc}$.

Since the AFAPE problem exhibits similarities between $J_a$ and $J_{mc}$, we will primarily concentrate on $J_{mc}$ in the main sections of this paper. For brevity, we use the notations $J_{mc} \equiv J$ and $C_{mc} \equiv C$. Detailed estimation formulas for the acquisition costs are provided in the relevant appendix.

\subsection{Problem Definition: Optimization of Active Feature Acquisition Methods}
\label{sec_afa_optimization_problem}

While we only address the evaluation of AFA methods in this paper, we include the definition of the AFA optimization problem for completeness. In AFA, the objective is to identify the optimal AFA policy 
$\pi_{\alpha}(A^t|\underline{X}^{t-1},\underline{A}^{t-1};\phi_1)$ parameterized by 
$\phi_1$, as well as the optimal classifier 
$g(Y^* |\underline{X}^{T-1}, \underline{A}^T; \phi_2)$ parameterized by $\phi_2$. 
The aim is to utilize these jointly in a way that minimizes the expected cumulative cost comprising both counterfactual acquisition and misclassification costs:
\begin{align}
\label{eq:AFA_objective}
   \phi_1^{*}, \phi_2^{*} = 
   \argmin_{\phi_1,\phi_2}  
   J_\textit{total}(\phi_1, \phi_2)= 
   \argmin_{\phi_1,\phi_2}  
   \mathbb{E}\left[
   \sum_{t=1}^T C_{a,(\pi_\alpha)}^t +
   C_{mc,(\pi_\alpha)} \Big\vert \phi_1, \phi_2  \right].
\end{align}

\subsection{Assumptions}

We provide here a list of assumptions made throughout this paper. 
These include: 
\begin{itemize}
    \item \textit{No measurement noise:} Features are perfectly measured. 
    \item \textit{No direct effect (NDE)} \cite{von_kleist_evaluation_2023,robins_estimation_2008}: The measurement of any feature does not affect any features or the label. 
    \item \textit{Consistency:} Relates the counterfactual and observational distribution through Eq. \ref{eq_consistency}. 
    \item \textit{No non-compliance:} Assumes the AFA policy $\pi_\alpha$ solely determines which features are being acquired. This excludes, for example, scenarios where patients refuse certain tests or miss appointments. 
    \item  \textit{No interference: } Assumes actions on one individual do not affect others. This excludes, for example, scenarios where the hospital staff is overwhelmed by a high volume of simultaneous test requests. 
    \item \textit{Positivity / experiment treatment assignment: } 
    Assumes that certain sets of features had positive probability of being acquired during the acquisition of the retrospective dataset. The necessary positivity assumption for identification differs based on the applied viewpoint to solve AFAPE and will thus be specified/ derived during the discussion of the respective viewpoints. 
\end{itemize}

\noindent 
In addition to these assumptions, we also make varying assumptions about the missingness process (i.e. about  $\pi_\beta(R|X_{(1)})$). These differ between sections of the paper and can be categorized as follows: 

\textit{Missing-completely-at-random (MCAR)}:
Under MCAR, the reason for missingness is completely independent of any feature values: $\pi_\beta(R|X_{(1)}) = \pi_\beta(R)$. 

\textit{Missing-at-random (MAR)}:
Under MAR, the reason for missingness only depends on the observed features: 
$p(R=r | X_{(1)}) = p(R=r | \{ X_{(1),i} : r_i = 1 \})$. We will in this paper, however, only consider the following simpler MAR assumption: 
\begin{equation}
\label{eq_MAR}
    p(R=r | X_{(1)}) = p(R=r | X_{o} )
\end{equation}
\noindent 
where $X_{o}$ is a set of always observed features. Note that $X_{o}$ is only always observed in the retrospective data and does not need to be always observed under $\pi_\alpha$. 

\textit{Missing-not-at-random (MNAR)}:
All scenarios that are not MCAR or MAR and feature a dependence of $R$ on potentially unobserved variables, are denotes as MNAR. 

\section{Missing Data (+ Online Reinforcement Learning) View}
\label{sec_missing_data}

The AFAPE problem can, in time-series settings, be formulated as a combination of a missing data and an online RL problem \cite{von_kleist_evaluation_2023}. 
In this section, we show that this is also the case for the static feature setting. 

\subsection{Identification}

We start with the following theorem which decomposes the AFAPE problem into a missing data and an online RL part. 

\begin{theorem}
\label{theorem_problem_missing} 
(AFAPE problem decomposition into missing data and online RL).
The AFAPE problem of estimating $J$ (Equation \ref{eq:AFAPE_objective})
can be decomposed under the no interference assumption into 
\begin{align}
\label{eq:AFAPE_objective_miss}
    J  
    = 
    \sum_{X_{(1)}, Y} 
    \underbrace{
    \mathbb{E}
    \left[
    C_{(\pi_{\alpha})}|X_{(1)},Y
    \right]}_{\text{online RL}}
    \underbrace{p(X_{(1)}, Y)}_{\text{missing data}}.
\end{align}
Furthermore, $J$ is identified if $p(X_{(1)}, Y)$ is identified. 
\end{theorem}

\begin{proof}
The inner expected value is identified, since all conditionals of the following factorization are known functions: 
\begin{align}
\label{eq_identification_online_RL}
 \mathbb{E}
    [
    C_{(\pi_{\alpha})}|X_{(1)},Y
    ]
     =  \sum_{\underline{X}_{(\pi_{\alpha})}^T,\underline{A}_{(\pi_{\alpha})}^T,Y_{(\pi_{\alpha})}^*,C_{(\pi_{\alpha})}}
    C_{(\pi_{\alpha})}
    q(C_{(\pi_{\alpha})}, Y^*_{(\pi_{\alpha})}, X_{(\pi_{\alpha})}, A_{(\pi_{\alpha})}|
    X_{(1)},Y) 
\end{align}
where 
\begin{align*}
    q(C_{(\pi_{\alpha})},  Y^*_{(\pi_{\alpha})}, & X_{(\pi_{\alpha})}, A_{(\pi_{\alpha})}|
    X_{(1)},Y) = 
    \\ &  = 
    \prod_{t=0}^{T-1}
    \underbrace{
    g(
    X_{(\pi_{\alpha})}^t
    |
    A_{(\pi_{\alpha})}^{t},
    X_{(1)}
    ) 
    }_{\text{feature revelations}}
    \prod_{t=1}^T
    \underbrace{
    \pi_{\alpha}(
    A_{(\pi_{\alpha})}^t|
    \underline{X}_{(\pi_{\alpha})}^{t-1},
    \underline{A}_{(\pi_{\alpha})}^{t-1}) 
    }_{\text{acquisition decisions}}
    \underbrace{
    g(Y_{(\pi_{\alpha})}^*|
    \underline{X}_{(\pi_{\alpha})}^{T-1}, \underline{A}_{(\pi_{\alpha})}^T 
    )}
    _{\text{label prediction}}
    \underbrace{
    g(C_{(\pi_{\alpha})}|
    Y, 
    Y_{(\pi_{\alpha})}^* 
    )}
    _{\text{cost computation}}.
\end{align*}
\end{proof}

\noindent 
The inner expected value $\mathbb{E}
[
C_{(\pi_{\alpha})}|X_{(1)},Y
]$ represents the online RL part as it features the evaluation of a policy in a known environment. The outer expected value represents a missing data problem as it necessitates the identification of the counterfactual feature distribution $p(X_{(1)},Y)$. 

The identification of $p(X_{(1)}, Y)$) depends on the assumed model for the missingness mechanism. While identification is not possible for all MNAR scenarios, there exists a large class of MNAR submodels that are identified \cite{bhattacharya_identification_2020, nabi_full_2020}. We provide a review of identification in missing data problems in Appendix \ref{app_missing_data} and provide an example. 

Identification of $p(X_{(1)}, Y)$ also requires at least the following positivity assumption:

\vspace{10pt}
\noindent
\textit{Positivity assumption (missing data): } 
\noindent
\begin{align}
\label{eq_positivity_missing_data}
\pi_{\beta}(R =\vec{1}| 
X_{(1)} = x)
>0 
\quad \quad \quad \quad 
\forall  x  
\end{align}

\noindent 
This positivity assumption states that there must be support for complete cases for all subpopulations in the data. This assumption is very strong and it may be easily violated in real-world settings as physicians rarely perform all possible feature acquisitions for every group of patients.

\subsection{Estimation}

The online RL problem of finding an estimate for $\mathbb{E}\left[C_{(\pi_{\alpha})}|X_{(1)},Y\right]$,  denoted as $\hat{\mathbb{E}}\left[C_{(\pi_{\alpha})}|X_{(1)},Y\right]$, can be solved using simple Monte Carlo integration to solve the expected value in Eq. \ref{eq_identification_online_RL}. 
The missing data problem allows the use of the following two estimators: 

\vspace{10pt}
\noindent
\textit{1) Inverse probability weighting (IPW):} 

\noindent
The IPW estimator \cite{seaman_review_2013} has the following form: 
\begin{align}
    J_{\textit{IPW-Miss}} = \hat{\mathbb{E}}_n \left[ \rho_{\textit{Miss}}\text{ } \hat{\mathbb{E}}\left[C_{(\pi_{\alpha})}|X_{(1)},Y\right]\right], \text{  where  } \rho_{\textit{Miss}} =  
    \frac{\mathbb{I}(R= \vec{1}) }{
    \hat{\pi}_{\beta}(
    R=\vec{1}| 
    X_{(1)})}, 
\end{align}

\noindent 
where $\mathbb{I}(.)$ denotes the indicator function.
A model for the propensity score $\pi_{\beta}$, denoted as $\hat{\pi}_{\beta}$, needs to be learned from the data. This estimator cannot make use of large parts of the available data, since only complete cases are (where $\mathbb{I}(R= \vec{1})$) are reweighted.

\vspace{10pt}
\noindent
\textit{2) Plug-in of the G-formula:} 

\noindent
The estimator based on the plug-in of the G-formula \cite{hernan_causal_2020} has the following form: 
\begin{align} \label{eq:MI}
    J_{\textit{G-Miss}} = \sum_{X_{(1)},Y} \hat{\mathbb{E}}\left[C_{(\pi_{\alpha})}|X_{(1)},Y\right]  \hat{p}(X_{(1)},Y).  
\end{align}

\noindent 
To employ this estimator, one must estimate the counterfactual data distribution $p(X_{(1)}, Y)$. Typically, this distribution is not modeled entirely; instead, the empirical distribution of the existing data is supplemented with samples (imputations) generated from a model for the missing data. This strategy is referred to as multiple imputation (MI) \cite{sterne_multiple_2009}. For a more in-depth exploration of MI and potential challenges within the context of AFA, we direct readers to our companion paper \cite{von_kleist_evaluation_2023}.

\section{Semi-offline Reinforcement Learning View}
\label{sec_semi_offline_RL}

In this section, we consider the special MAR assumption $\pi_\beta(R|X_{(1)}) = \pi_\beta(R|X_{o})$ where $X_{o}$ corresponds to a subset of features that are always observed in the retrospective dataset.  

To motivate the semi-offline RL viewpoint, we now examine a simple scenario to exemplify that the missing data + online RL view imposes unnecessarily strong positivity requirements and leads to inefficient estimation. 
Suppose an AFA policy $\pi_\alpha$ always acquires only features $X_1$ and $X_2$ out of a possible set of 4 features. 
The missing data view still requires the identification of $p(X_{(1)},Y)$ for all features, even though the target cost is independent of the features that are not acquired. 
This scenario suggests that it should be possible to relax the strong positivity requirement of complete cases to a weaker requirement of positive support for data points where $X_1$ and $X_2$ are observed. Likewise, all datapoints where $X_1$ and $X_2$ were observed should be used for reweighting in the IPW estimator, not just complete cases. The following semi-offline RL view achieves both of these improvements. 

The concept of semi-offline reinforcement learning is closely related to off-policy reinforcement learning. In off-policy RL, even though it is possible to directly run the target policy $\pi_{\alpha}$ in a known environment, one instead uses a different simulation policy $\pi_{sim}$ to sample the actions. To prevent introducing bias, adjustments are made to account for the differences in the distributions.
In our semi-offline RL approach, we follow a similar strategy but with an additional constraint on the simulation policy. Specifically, the simulation policy is restricted to selecting only those features $i$ that are available in the current retrospective datapoint (i.e. the features $i$ s.t. $R_i=1$). As only available features (for which $X_{{(1)},i} = X_i$ holds) are sampled, one does not need to explicitly need to know $X_{(1)}$ during the sampling process. As in off-policy RL, one must, however, adjust for this blocking of the acquisition of non-available features.  We refer to this approach as "semi-offline" because the simulation policy is allowed to freely sample from the available features (the online part), but it is not permitted to acquire features that are unavailable (the offline part).



%

\subsection{The Semi-offline Sampling Policy}

Firstly, we construct the semi-offline sampling distribution $p'$ by defining a blocked policy that ensures that no unavailable features are sampled.  
\begin{definition}
\label{def_blocked_policy}
(Blocked Policy)
A policy $\pi'(A'^t|\underline{X}'^{t-1},\underline{A}'^{t-1}, R)$ is called a 'blocked policy' of the policy $\pi(A'^t|\underline{X}'^{t-1},\underline{A}'^{t-1})$  
if it satisfies the following conditions: 

\noindent 
1) Blocking of acquisitions of non-available features:
\begin{flalign*}
\text{if }  
& a'^t = i, i \leq d_x,\text{and }  r_i = 0,  
\quad \text{then } 
\pi'(
a'^t|
\underline{x}'^{t-1},
\underline{a}'^{t-1},
R = r) = 0
\nonumber
\quad \quad
&&
\forall  t,a'^t,r, \underline{x}'^{t-1}, \underline{a}'^{t-1}  
\end{flalign*}


\noindent 
2) No blocking of acquisitions of available features:
\begin{flalign*}
\text{if }  \hspace{11 pt} \quad \quad  \quad \quad \quad \quad \quad \quad
& a'^t  = i, r_i =1 \text{ (if $i\leq d_x$)},\text{and }  
\pi(
a'^t|
\underline{x}'^{t-1},
\underline{a}'^{t-1}) > 0, 
\nonumber
&&
\\
\text{then } \quad \quad \quad \quad \quad \quad \quad \quad
& \pi'(
a'^t|
\underline{x}'^{t-1},
\underline{a}'^{t-1}, 
R = r) > 0
\nonumber
&&
\\ 
&
\forall  t,a'^t,r, \underline{x}'^{t-1}, \underline{a}'^{t-1}  &&
\end{flalign*}
%
%
\end{definition}

\noindent
Condition 1 ensures that the acquisition of unavailable features is blocked, while condition 2 ensures that other desired a
actions are still performed. 

The definition of the blocked policy allows us to define the semi-offline RL sampling policy, denoted by $p'$ which replaces the desired AFA policy $\pi_\alpha$ that can't be sampled (due to unavailable counterfactual features) with a blocked policy: 
\begin{align}
\label{eq_def_p_prime}
p'(C', & A', X'|X_{(1)}, Y, R)  
= 
\prod_{t=0}^{T-1}
g(
X'^t
|
A'^{t},
X_{(1)}^t
)
\prod_{t=1}^T
\underbrace{
\pi'_{sim}(
A'^t|
\underline{X}'^{t-1},
\underline{A}'^{t-1}, 
R) 
}_{\text{actions under blocking}}
g(C'|
Y, 
\underline{X}'^{T-1}, \underline{A}'^T 
)
\end{align}

\noindent 
where we let $C'$, $X'$ and $A'$ denote the resulting "simulated" costs, observed features and actions. $\pi'_{sim}$ denotes a blocked simulation policy $\pi_{sim}$. If $\pi_{sim}$ is different from $\pi_\alpha$, this adds a conventional off-policy aspect to the sampling distribution. We require $\pi_{sim}$ to fulfill the standard off-policy positivity assumption: 

\vspace{10pt}
\noindent
\textit{Positivity assumption (off-policy RL): } 
\begin{flalign*}
\text{if }  \hspace{11 pt} \quad \quad \quad \quad \quad \quad
& p( 
\underline{X}^{t-1}_{(\pi_\alpha)} = \underline{x}^{t-1}, 
\underline{A}^{t-1}_{(\pi_\alpha)} = \underline{a}^{t-1}  
) \pi_\alpha(a^t|     
\underline{x}^{t-1} ,
\underline{a}^{t-1} )
> 0,
\nonumber
&&
\\
\text{then } \quad \quad \quad \quad \quad \quad
& p(
\underline{x}^{t-1}, 
\underline{a}^{t-1}
)
\pi_{sim}(a^t| 
\underline{x}^{t-1}, 
\underline{a}^{t-1})
>0 
\nonumber
&&
\\ 
&
\forall  t, a^t,
\underline{x}^{t-1}, \underline{a}^{t-1}  &&
\end{flalign*}
\noindent where $q(\underline{X}^{t-1}_{(\pi_\alpha)}, \underline{A}^{t-1}_{(\pi_\alpha)} )$ denotes the marginal counterfactual distribution of states and actions under $\pi_\alpha$. The positivity assumption states that 
$\pi_{sim}$ needs positive support for actions where $\pi_\alpha$ has positive support. 

The resulting semi-offline RL sampling distribution $p'$ ensures that $p'(C',X',A'|X_{(1)},Y, R) = p'(C',X',A'|X,Y, R)$ (i.e. conditional independence of $X_{(1)}$)  and thus the possibility to sample $p'$ for every datapoint $X,Y, R$. The semi-offline RL sampling policy can then be used to generate a new dataset $\mathcal{D}'$ based on which $J$ can be estimated. 
The causal graph showing the new semi-offline sampling distribution $p'$ is shown in Figure \ref{graph_semi_offline_RL}. 

\begin{figure}[ht]
\centering
\includegraphics[width=0.8\textwidth]
{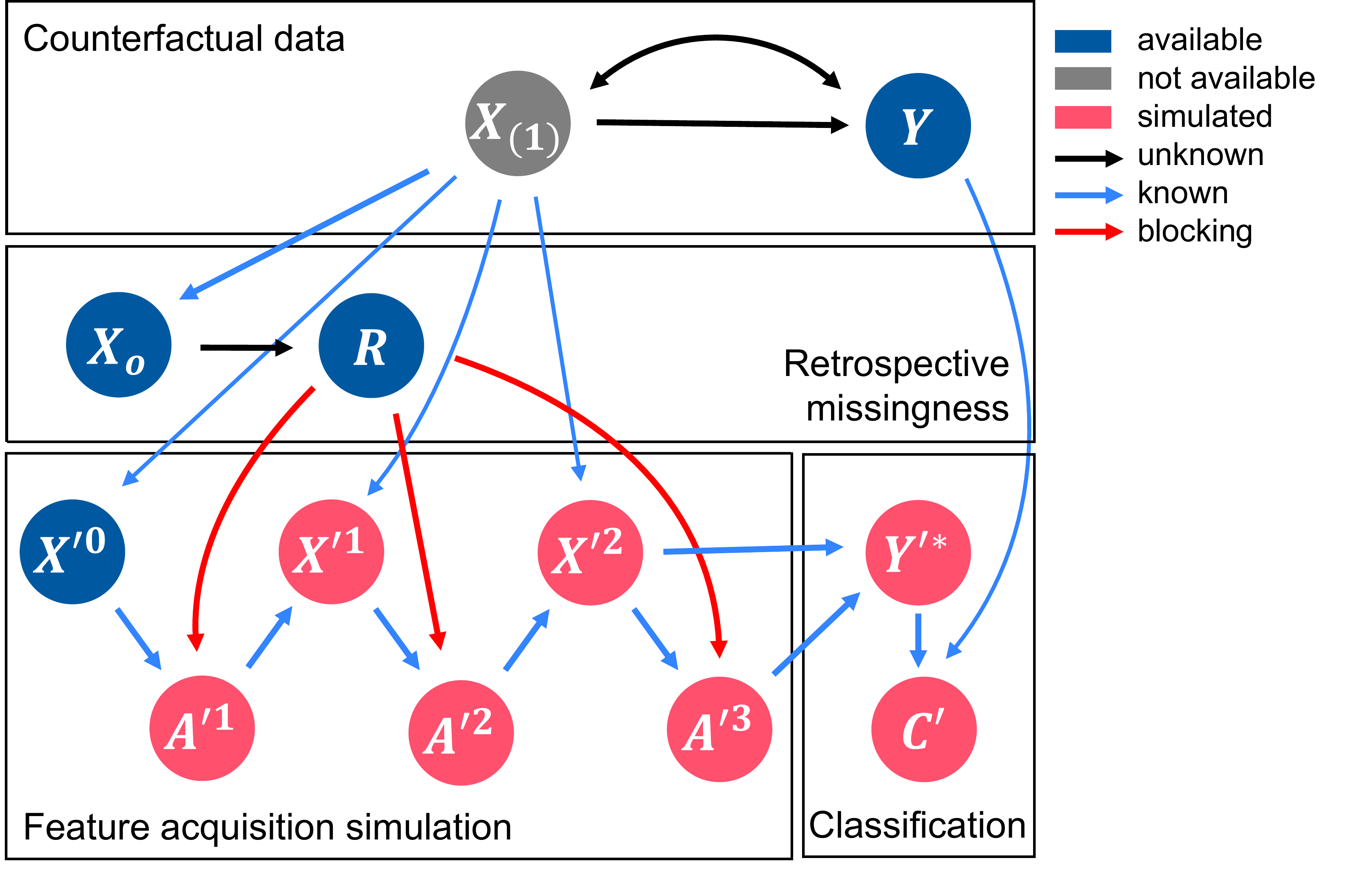}
\caption{ 
Causal graph depicting the semi-offline RL sampling distribution. The simulated actions $A'^t$ depend on the retrospective missingness indicator $R$ through a blocking operation that prevents the acquisition of non-available features. 
The following edges showing long-term dependencies were excluded from the graph for visual clarity:
$\underline{X}'^{t-1}, \underline{A}'^{t-1} \rightarrow  A'^{t}$ and $\underline{X}'^{T-1}, \underline{A}'^{T-1} \rightarrow  Y'^*$. }
\label{graph_semi_offline_RL}
\end{figure}

\subsection{Problem Reformulation}

Firstly, we show that one can still address the original AFAPE problem from the semi-offline RL viewpoint by providing the following theorem. 

\begin{theorem}
\label{theorem_problem_semi_offline_RL} 
(AFAPE problem reformulation under the semi-offline RL view).
The AFAPE problem of estimating $J$ (Eq. \ref{eq:AFAPE_objective} or Eq. \ref{eq:AFAPE_objective_miss})
is under the no direct effect (NDE), the no interference and the static feature assumption equivalent to estimating 
\begin{align}
\label{eq:AFAPE_objective_semi_offline_RL}
    J  
    = 
    \mathbb{E}_{p'}[C'_{(\pi_\alpha)}].
\end{align}
\end{theorem}

\noindent 
$C'_{(\pi_\alpha)}$ denotes the potential outcome of $C'$, had, instead of the blocked simulation policy $\pi'_{sim}$, the AFA policy $\pi_\alpha$ been employed. 

\vspace{10pt}
\begin{proof}
We begin from Eq. \ref{eq:AFAPE_objective_miss} to show the equivalence: 
\begin{align*}
    J & 
    =
     \sum_{X_{(1)}, Y} \mathbb{E}
    \left[
    C_{(\pi_{\alpha})}|X_{(1)},Y
    \right]
    p(X_{(1)}, Y)
     \overset{*_1}{=} 
    \sum_{X_{(1)}, Y} \mathbb{E}_{p'}
    \left[
    C'_{(\pi_\alpha)}|X_{(1)},Y
    \right]
    p(X_{(1)}, Y)
    = 
       \mathbb{E}_{p'}[C'_{(\pi_\alpha)}]
\end{align*}
\noindent 
where $*1)$ uses the following equivalence: 
\begin{align*}
 \mathbb{E}_{p'} [C'_{(\pi_\alpha)}| X_{(1)},Y]
&\overset{*_{1.1}}{=} 
\sum_{C', X', Y' } 
C'
\prod_{t=0}^{T-1}
g(X'^t|A'^t, X_{(1)})
\prod_{t=1}^T
\pi_{\alpha}(
A'^t|
\underline{X}'^{t-1},
\underline{A}'^{t-1})  
g(C'|
Y, 
\underline{X}'^{T-1}, \underline{A}'^T 
)
\\ & \overset{*_{1.2}}{=}  \mathbb{E}
    \left[
    C_{(\pi_{\alpha})}|X_{(1)},Y
    \right]
\end{align*}
where $*1.1)$ holds by the G-formula, and $*1.2)$ is the same factorization as in Eq.  \ref{eq_identification_online_RL}. 
\end{proof}

\subsection{Identification}

We make the following positivity assumption to allow identification: 

\vspace{10pt}
\noindent 
\textit{Positivity assumption (semi-offline RL): } 
\begin{flalign}
\label{eq_positivity_semi_offline_RL}
\text{if }  \hspace{11 pt} \quad \quad \quad \quad \quad \quad
& 
\prod_{t=1}^{T}  
\pi_{\alpha}(a'^t|             
\underline{x}'^{t-1},
\underline{a}'^{t-1})  
\prod_{t=0}^{T-1}
p'(x'^t|   
\underline{x}'^{t-1},        
\underline{a}'^{t},
x_o
) 
p(x_o)  
>0
\nonumber
&&
\\
\text{then } \quad \quad \quad \quad \quad \quad
& 
\pi_\beta( 
r \geq r'|
x_o)
>0
\nonumber
&&
\\ 
&
\forall 
\underline{x}'^{T-1},
\underline{a}'^T,
x_o,
r
\end{flalign}

\noindent 
where $r'= r'^T$ denotes the set notation of $\underline{a}'^T$, i.e. it indicates which features where acquired until step $T$. Furthermore, we let $r \geq r'$ be a shorthand notation for the element-wise comparison $r_i \geq r'_i$ $\forall i$. 
Note that the positivity assumption is independent of the simulation process as one can optionally rewrite: 
$p'(X'^t|   
\underline{X}'^{t-1},        
\underline{A}'^{t},
X_o
)  = p(X_{(1), a'^t}|   
X_{(1), r'^{t-1}}, 
x_o
)$
where we let $X_{(1), a'^t}$ and $X_{(1), r'^{t-1}}$ denote the counterfactual feature values $X_{(1)}$ at the indices of the current acquisition $a'^t$ and at the indices of all previous acquisitions $r'^{t-1}$, respectively. The requirement for this positivity assumption is shown in the proofs of the subsequently following identificiation theorems. 

The positivity assumption is much weaker than the positivity assumption from the missing data view. It only requires for all desired acquisitions $\underline{a}'^T$, that there is a datapoint, with at least as many features, that has positive support. 

Given the reformulated positivity assumption, one can perform identification as stated by the following theorem.   

\begin{theorem}
\label{theorem_identification_semi_offline_RL} 
(Identification of $J$ for the semi-offline RL view).
The reformulated AFAPE problem of estimating $J$ under the semi-offline RL view  (Eq. \ref{eq:AFAPE_objective_semi_offline_RL}) is under the no direct effect (NDE) assumption, the MAR assumption from Eq. \ref{eq_MAR}, the consistency assumption, the no interference assumption, the static feature assumption and the positivity assumption from Eq. \ref{eq_positivity_semi_offline_RL} identified by 
\begin{align}
\label{eq_identificiation_semi_offline_RL_1}
    J  
     = 
    \mathbb{E}_{p'}[C'_{(\pi_\alpha)}]
     = 
    \sum_{C',Y, X',A',R,X_{o}} 
    C' 
    g(C'|           
\underline{X}'^{T-1},            
\underline{A}'^{T},         
Y)
p'(Y|                   
\underline{X}'^{T-1}, 
\underline{A}'^{T},
X_{o})
q'(
\underline{X}'^{T-1},   
\underline{A}'^{T},        R,
X_{o}           
) 
\end{align}
\noindent
with the distribution
\begin{align} 
\label{eq_identificiation_semi_offline_RL_2}
q'(\underline{X}'^{T-1},            
\underline{A}'^{T},         R,
X_{o})=
\pi_\beta(R| R \geq R', X_o)
\prod_{t=1}^{T}  
\pi_{\alpha}(A'^t|             
\underline{X}'^{t-1},             
\underline{A}'^{t-1}) 
\prod_{t=0}^{T-1}
p'(X'^t|   
\underline{X}'^{t-1},        
A'^{t},
X_{o})
p(X_{o}) 
\end{align}
\end{theorem}

\noindent 
We can also extend the Bellman Equation from RL to the semi-offline RL setting:

\begin{theorem}
\label{theorem_Bellman_equation} 
(Bellman equation for semi-offline RL). 
The semi-offline RL view admits under the no direct effect (NDE) assumption, the MAR assumption from Eq. \ref{eq_MAR}, the consistency assumption, the no interference assumption, the static feature assumption, and the positivity assumption from Eq. 
\ref{eq_positivity_semi_offline_RL} the following semi-offline RL version of the Bellman equation:   
\begin{align}
\label{eq_bellman_1}
Q_{\textit{Semi}}(\underline{X}'^{t-1}, &
\underline{A}'^{t},            
X_o ) 
=
\sum_{\mathclap{X'^t}}                          V_{\textit{Semi}}
(              \underline{X}'^{t}, \underline{A}'^{t}, 
X_o) 
p(X'^t|                   
\underline{X}'^{t-1},          
A'^{t}, 
X_o)
\\
 V_{\textit{Semi}}
(\underline{X}'^{t}, 
\underline{A}'^{t},&
X_o)  
=
\sum_{A'^{t+1}}  
Q_{\textit{Semi}}(\underline{X}'^{t}, \underline{A}'^{t+1},      X_o  )
\pi_{\alpha}(A'^{t+1}|
\underline{X}'^{t},
\underline{A}'^{t})
\label{eq_bellman_2}
\end{align}
with semi-offline RL versions of the state-action value function $Q_{\textit{Semi}}$ and state value function $V_{\textit{Semi}}$:
\begin{align*}
 Q_{\textit{Semi}}^t 
 & \equiv 
 Q_{\textit{Semi}}(\underline{X}'^{t-1},
\underline{A}'^{t},
X_o) 
\equiv  
\mathbb{E}_{p'}[C'_{(\overline{\pi}^{t+1}_\alpha)}|
\underline{X}'^{t-1},
\underline{A}'^{t},
X_o]  
\\
V_{\textit{Semi}}^t
& \equiv 
V_{\textit{Semi}}(\underline{X}'^t,
\underline{A}'^{t},
X_o) 
\equiv 
\mathbb{E}_{p'}[C'_{(\overline{\pi}^{t+1}_\alpha)}|
\underline{X}'^t,
\underline{A}'^{t},
X_o]
\end{align*}
where $C'_{(\overline{\pi}^{t+1}_\alpha)}$ denotes the potential outcome of $C'$ under interventions from time step $t+1$ onwards. 
\end{theorem}

\noindent 
Appendix \ref{app_theorem_identification} contains the proofs for Theorems \ref{theorem_identification_semi_offline_RL} 
and \ref{theorem_Bellman_equation} and the derivation of the factorization of the "observational" (i.e. simulated) distribution given in the following remark:

\begin{remark}
\label{remark_semi_offline_RL_observational}
The observational data under the simulations given by $p'$ factorizes as:
\begin{align} 
\label{eq_fact_obs_p_prime}
p'(          
\underline{X}'^{T-1},            
\underline{A}'^{T},          
R, 
X_o) 
 = 
\prod_{t=1}^{T}  
\pi_{sim}'(
A'^{t} | 
\underline{X}'^{t-1},
\underline{A}'^{t-1}, 
R)
\prod_{t=0}^{T-1}
p(X'^t|                 
\underline{X}'^{t-1},      
A'^{t},
X_o) 
\pi_{\beta}(
R | 
X_o)
p(X_o)
\end{align}
\end{remark}

\subsection{Estimation}
\label{sec_semi_offline_RL_estimation}

The semi-offline RL view leads to the following new estimators: 

\vspace{10pt}
\noindent 
\textit{1) Inverse probability weighting (IPW):} 

\noindent 
The semi-offline RL IPW estimator has the following form: 
\begin{align} 
\label{eq:semi-off-ipw_pi_id}
    J_{\textit{IPW-Semi}}
    & = 
    \hat{\mathbb{E}}_{n'}
    \left[ \rho_{
    \textit{Semi}}^T
    \text{ } 
    C'\right], 
\end{align} 
where $\hat{\mathbb{E}}_{n'}[.]$ denotes the empirical average over the dataset $\mathcal{D}'$ and 
\begin{align} 
\label{eq_ipw_semi_offline_RL_weights_2}
    \rho^t_{\textit{Semi}} & = 
    \prod_{\tau=1}^t
    \frac{
    \pi_\alpha(A'^{\tau}| \underline{X}'^{\tau-1}, \underline{A}'^{\tau-1}) 
    }
    {
    \pi'_{sim}(
    A'^{\tau}| \underline{X}'^{\tau-1}, \underline{A}'^{\tau-1},
    R
    )
    }
    \frac{
    \mathbb{I}(R \geq R'^t)
    }
    {
    \hat{\pi}_\beta(
    R \geq R'^t| 
    X_o
    ).
    }
\end{align}

\noindent 
The IPW estimator $J_{\textit{IPW-Semi}}$ demonstrates the large benefits of the semi-offline RL view over the missing data view. Its second fraction shows that not only datapoints where $R = \vec{1}$ are used (i.e. have positive weight), as in the missing data view, but all datapoints where $R \geq R'$ can be used.

\vspace{10pt}
\noindent
\textit{2) Direct method (DM):} 

\noindent 
The semi-offline RL DM estimator has the following form: 
\begin{align} \label{eq:semi-off-dm}
    J_{\textit{DM-Semi}} = 
    \hat{\mathbb{E}}_{n'}[V_{\textit{Semi}}^0]
\end{align}
This estimator requires learning $Q_{\textit{Semi}}$ using the semi-offline version of the Bellman equation (Eqs. \ref{eq_bellman_1} and \ref{eq_bellman_2}). $V_{\textit{Semi}}$ can then be inferred as $V_{\textit{Semi}}^t = \mathbb{E}_{\pi_{\alpha}}[Q_{\textit{Semi}}^{t+1}]$.

\vspace{10pt}
\noindent 
\textit{3) Double reinforcement learning (DRL):} 

\noindent
The semi-offline DRL estimator has the following form: 
\begin{align}
    J_{\textit{DRL-Semi}} = 
    \hat{\mathbb{E}}_{n'}
    \left[\rho_\textit{Semi}^T C' + 
 \sum_{t=1}^{T} 
 \left(- \rho_{\textit{Semi}}^{t}  
 Q_\textit{Semi}^t
 +
 \rho_{\textit{Semi}}^{t-1}  
 V_\textit{Semi}^{t-1}
 \right)
 \right].
\end{align}

\noindent 
Similar to the DLR estimator from offline RL \cite{kallus_double_2020}, this approach combines the other two estimators  (Eqs. \ref{eq:semi-off-ipw_pi_id} and \ref{eq:semi-off-dm}). 

Consistency properties of the new estimators are given by the following theorems. 

\begin{theorem}
\label{theorem_consistency_IPW} (Consistency of $J_{\textit{IPW-Semi}}$). 
The estimator $J_{\textit{IPW-Semi}}$ is consistent if the propensity score model $\hat{\pi}_\beta$ is correctly specified.
\end{theorem}

\begin{proof}
The consistency of the IPW estimator follows from the standard inverse probability weighting approach $\mathbb{E}_{q'}[C']=\mathbb{E}_{p'}[\frac{q'}{p'}C']$ and the use of the factorizations for $q'$ and $p'$ from Eqs. \ref{eq_identificiation_semi_offline_RL_2} (Theorem \ref{theorem_identification_semi_offline_RL}) and \ref{eq_fact_obs_p_prime} (Remark \ref{remark_semi_offline_RL_observational}).
\end{proof}

\begin{theorem}
\label{theorem_consistency_DM} (Consistency of $J_{\textit{DM-Semi}}$). 
The estimator $J_{\textit{DM-Semi}}$ is consistent if the Q-function $Q_\textit{Semi}$ is correctly specified.
\end{theorem}

\begin{proof}
    We can simply apply the law of total expectation for the first step of the semi-offline Bellman equation (Theorem \ref{theorem_Bellman_equation}). 
\end{proof}

\begin{theorem}
\label{theorem_double_robustness}
(Double robustness of $J_{\textit{DRL-Semi}}$).
The estimator $J_{\textit{DRL-Semi}}$ is 
doubly robust, in the sense that it is consistent if either the Q-function $Q_\textit{Semi}$ or the propensity score model $\hat{\pi}_\beta$ is correctly specified. 
\end{theorem}

\noindent 
We defer the proof to Appendix \ref{app_theorem_double_robustness}.
The DRL estimator is regular and asymptotically linear (RAL) as it was derived from the influence function stated in the following theorem: 

\begin{theorem}
\label{theorem_efficient_if} 
(An influence function under the semi-offline RL view). An influence function of $J$ 
is:  
\begin{align}
     \Psi
     = -J + \rho_\textit{Semi}^T C' + 
 \sum_{t=1}^{T} 
 \left(
 - \rho_\textit{Semi}^t  Q_\textit{Semi}^t
+
 \rho_\textit{Semi}^{t-1}  
 V_\textit{Semi}^{t-1}
 \right).
\end{align}

\noindent 
\end{theorem}
%

\noindent 
We defer the proof to Appendix \ref{app_theorem_efficient_if}. 
In Appendix \ref{app_other_variants}, we extend the identification and estimation results of this section to the target parameter $J_{total}$ which includes the estimation of counterfactual acquisition and misclassification costs. 

\section{Semi-offline RL for Missing-not-at-random (MNAR) Scenarios}
\label{sec_MNAR}

In the following, we discuss how to combine the semi-offline RL and missing data views in MNAR scenarios. 
In particular, we examine the scenario, where instead of $R \perp\!\!\!\perp X_{(1)}|X_{o}$, we assume that a more general  $ R \perp\!\!\!\perp X_{(1)}|X_{adj,{(1)}}$ holds. 
The new subset of features needed for adjustment $X_{adj,(1)}$ now includes features that are not always observed. 

As described in Section \ref{sec_missing_data}, the missing data view can still be applied if  $\pi_{\beta}(R|X_{adj,(1)})$ is identified.
Semi-offline RL, however, encounters limitations when faced with such MNAR scenarios.
The pure semi-offline RL view cannot be applied, as, for example, the needed propensity score $\pi_\beta(R \geq R'|X_{adj,(1)})$ cannot be evaluated on datapoints where $R_{adj} \neq \vec{1}$, i.e. where the features needed for confounding adjustment are missing. 
Interestingly, such challenges are absent in the missing data perspective, since the propensity score model is strictly assessed for fully observed cases only (where $X_{(1)} = X$).


To address this issue, we propose a new hybrid semi-offline RL / missing data viewpoint, that lies in between the pure semi-offline RL view and the missing data view. In particular, we propose to first solve the missing data problem, but only for $p(X_{adj,{(1)}})$. Then, one can treat the known $X_{adj,{(1)}}$ just like $X_{o}$ and apply the semi-offline RL view.

\subsection{Problem Reformulation}
We reformulate the AFAPE problem in terms of the new hybrid semi-offline RL + missing data view:
\begin{theorem}
\label{theorem_problem_hybrid} 
(AFAPE problem reformulation under the hybrid semi-offline RL + missing data view).
The AFAPE problem of estimating $J$ (Equation \ref{eq:AFAPE_objective}) 
is under the no interference, no direct effect (NDE), and static feature assumption equivalent to estimating 
\begin{align}
\label{eq:AFAPE_objective_hybrid}
    J  
    = 
    \sum_{ X_{adj,(1)}} 
    \underbrace{
    \mathbb{E}_{p'}
    \left[
    C'_{(\pi_{\alpha})}|X_{adj,{(1)}}
    \right]}_{\text{semi-offline RL}}
    \underbrace{p(X_{adj,{(1)}})}_{\text{missing data}}.
\end{align}
\end{theorem}

\begin{proof}
Starting from Eq. \ref{eq:AFAPE_objective_miss}, we find: 
\begin{align*}
J & 
 \overset{*_1}{=}
 \sum_{X_{(1)}, Y} 
 \mathbb{E}
\left[
C_{(\pi_{\alpha})}|X_{(1)},Y
\right]
p(X_{(1)}, Y)
 \\ & 
 \overset{*_2}{=}
\sum_{X_{(1)}, Y} 
 \mathbb{E}_{p'}
\left[
C'_{(\pi_{\alpha})}|X_{(1)},Y
\right]
p(X_{(1)}, Y)
\\ & 
 \overset{*_3}{=}
 \sum_{
 X_{adj,(1)},
 X_{-adj,(1)}, Y} 
 \mathbb{E}_{p'}
\left[
C'_{(\pi_{\alpha})}
|
X_{adj,(1)},
X_{-adj,(1)},
Y
\right]
p(
X_{-adj,(1)},
Y|
X_{adj})
p(X_{adj,(1)}) 
= 
\\ & 
=
 \sum_{ X_{adj,(1)}} 
\mathbb{E}_{p'}
\left[
C'_{(\pi_{\alpha})}|
X_{adj,{(1)}}
\right]
p(X_{adj,{(1)}})
\end{align*}
where 
\begin{itemize}
    \item $1*)$ is Eq. \ref{eq:AFAPE_objective_miss} from Theorem \ref{theorem_problem_missing} 
    \item $2*)$ is the decomposition derived in the proof of Eq. \ref{eq:AFAPE_objective_semi_offline_RL}
    \item $3*)$ is the separation of $X_{(1)}$ into $X_{adj,{(1)}}$ and its complement $X_{-adj,{(1)}}$
\end{itemize}
\end{proof}

\noindent 
Eq. \ref{eq:AFAPE_objective_hybrid} demonstrates the separation of the AFAPE problem in two parts of which one can be solved using semi-offline RL, while the other needs to be solved using missing data methods.

\subsection{Identification}

We make the following two positivity assumptions to allow identification under the hybrid semi-offline RL + missing data view. 

\vspace{10pt}
\noindent 
\textit{Positivity assumption (missing data): } 
\begin{align}
\label{eq_positivity_missing_data_for_hybrid}
\pi_{\beta}(R =\vec{1}| 
X_{adj,(1)} = x_{adj})
>0 
\quad \quad \quad \quad 
\forall  x_{adj}  
\end{align}

\noindent 
\textit{Positivity assumption (semi-offline RL): }
\begin{flalign}
\label{eq_positivity_semi_offline_RL_for_hybrid}
\text{if }  \hspace{11 pt} \quad \quad \quad \quad \quad \quad
& 
\prod_{t=1}^{T}  
\pi_{\alpha}(a'^t|             
\underline{x}'^{t-1},
\underline{a}'^{t-1})  
\prod_{t=0}^{T-1}
p'(x'^t|   
\underline{x}'^{t-1},        
\underline{a}'^{t},
x_{adj}
) 
p(x_{adj})  
>0
\nonumber
&&
\\
\text{then } \quad \quad \quad \quad \quad \quad
& 
\pi_\beta( 
r \geq r'|
x_{adj})
>0
\nonumber
&&
\\ 
&
\forall 
\underline{x}'^{T-1},
\underline{a}'^T,
x_{adj},
r
\end{flalign}

\noindent 
which is a simple combination of the two individual positivity assumptions. 
The positivity requirement can be interpreted as the need for the standard missing data positivity requirement for the identification of $X_{adj,{(1)}}$ and the semi-offline RL positivity requirement for the remaining variables. 



\subsection{Estimation}

One can in general apply any combination of estimators from the semi-offline RL and missing data viewpoints to estimate the respective parts of the reformulated AFAPE problem under the hybrid semi-offline RL + missing data view. 
Here, we show as an example a novel hybrid IPW estimator, that solves both parts using inverse probability weighting: 

\vspace{10pt}
\noindent 
\textit{Inverse probability weighting (IPW):} 

\noindent 
The target cost that is estimated by the hybrid semi-offline + missing data IPW estimator for static feature settings under MNAR scenarios is
\begin{align*}
J_{\textit{IPW-Semi-Miss}} = 
\mathbb{E}_{p'} 
\left[ 
\rho_{\textit{Semi}}^T
\rho_{\textit{Miss}} 
C'
\right]
\end{align*}

\noindent 
where  
\begin{align*}     \rho^t_{\textit{Semi}} & = 
    \prod_{\tau=1}^t
    \frac{
    \pi_\alpha(A'^{\tau}| \underline{X}'^{\tau-1}, \underline{A}'^{\tau-1}) 
    }
    {
    \pi'_{sim}(
    A'^{\tau}| \underline{X}'^{\tau-1}, \underline{A}'^{\tau-1},
    R
    )
    }
    \frac{
    \mathbb{I}(R \geq R'^t)
    }
    {
    \hat{\pi}_\beta(
    R \geq R'^t| 
    X_{adj,(1)}, R_{adj} = \vec{1}
    ).
    }
\end{align*}
and 
\begin{align*}
\rho_{\textit{Miss}}
 =  
\frac{\mathbb{I}(R_{adj} = \vec{1}) }{\hat{\pi}_{\beta}(R_{adj}=\vec{1}| X_{adj, (1)})}
\end{align*}

\section{Experiments}
\label{sec_experiments}

We evaluate different estimators on synthetic and real-world datasets (Heloc \cite{fico_fico_2018} 
and Income \cite{newman_uci_1998}) under examples of synthetic MAR and MNAR missingness to allow the validation of the developed estimators through knowledge about the ground truth.

\subsection{Experiment Design}

We assess different estimators for the evaluation of random AFA policies and a standard deep Q-network (DQN) RL agent \cite{mnih_human-level_2015}. For classifiers, we employ impute-then-regress classifiers \cite{le_morvan_whats_2021} with unconditional mean imputation and a random forest classifier.
The Q-function $Q_\textit{Semi}$ is trained using a multi-layer perceptron. The propensity score model $\hat{\pi}_\beta$ is trained using a logistic regression model which corresponds to the model that induced the retrospective missingness. 
Our analysis involves comparing the performance of the following estimators:

\begin{itemize}
    \setlength{\itemsep}{0pt}
    \item \textit{Imp-Mean:} Uses simple mean imputation for the missing features and is thus biased. 
    \item  \textit{Blocking:} Averages the costs from the semi-offline sampling distribution $p'$ without adjustment and thus gives biased estimates. 
    \item \textit{CC:} Averages the cost for complete cases and is therefore only unbiased for MCAR experiments. 
    \item \textit{IPW-Miss}/\textit{IPW-Miss-gt:} Missing data IPW estimator with normalized weights. \textit{IPW-Miss-gt} is further using the ground truth propensity score model $\pi_\beta$ instead of learning it.
    \item \textit{IPW-Semi}/\textit{IPW-Semi-gt:} Semi-offline RL IPW estimator using normalized weights and a learned or ground truth $\pi_\beta$.
    \item \textit{IPW-Semi-Miss}/\textit{IPW-Semi-Miss-gt:} Hybrid IPW estimator (for MNAR scenarios) with normalized weights and a learned or ground truth $\pi_\beta$.
    \item \textit{DM-Semi:} Semi-offline RL version of the direct method.
    \item \textit{DRL-Semi}/\textit{DRL-Semi-gt:} Semi-offline DRL estimator with normalized weights and a learned or ground truth $\pi_\beta$.
    \item  \textit{J:} This "estimator" is considered as the ground truth. The AFA policy is evaluated on the fully observed dataset. It thus involves estimating $J$ based on Eq. \ref{eq:MI} with a Monte Carlo estimate $\hat{\mathbb{E}}\left[ C_{(\pi_{\alpha})}|X_{(1)},Y \right]$ based on samples from the ground truth data without missingness. 
\end{itemize}
Appendix \ref{Appendix_experiments} contains the full experiment details.

\begin{figure}
\centering	\includegraphics[width= \textwidth]{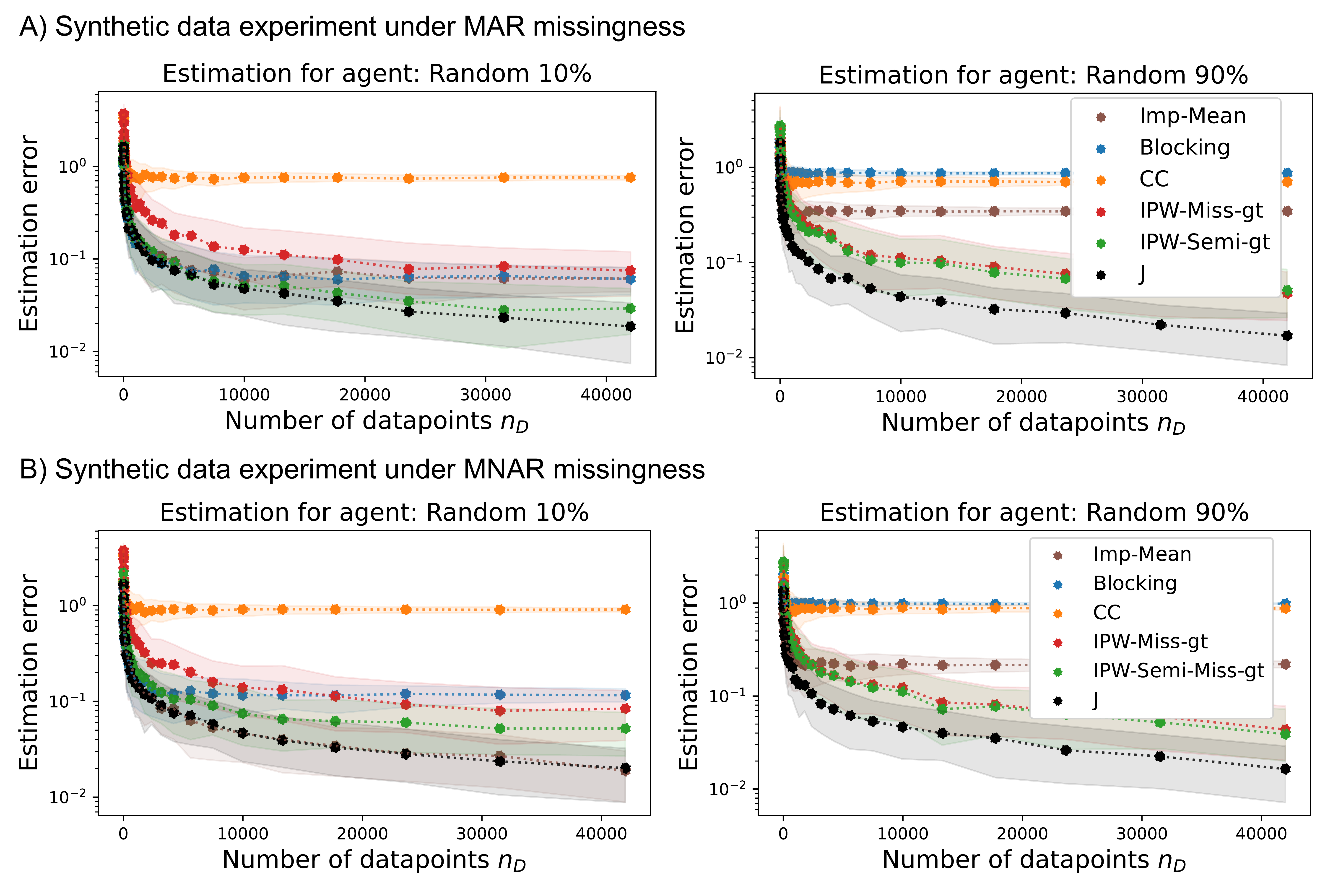}
\caption{Convergence plots for sampling-based estimators from two synthetic data experiments.
Plots show two agents that acquire each costly feature with a probability of $10\%$ and $90\%$, respectively. 
A) MAR experiment. B) MNAR experiment.
}
\label{figure_synthetic_convergence}
\end{figure}

\subsection{Results}

Figure \ref{figure_synthetic_convergence} illustrates convergence plots of sampling-based estimators in synthetic data experiments under both MAR and MNAR missingness. The estimation pertains to two random AFA policies that acquire costly features with a probability of 10\% and 90\%, respectively.

The experiments reveal that the mean imputation, blocking, and complete case analysis estimators, as anticipated, exhibit bias and fail to converge to the true value of $J$. The missing data IPW estimator (\textit{IPW-Miss-gt}) does converge to the ground truth, albeit at a slow pace, as it exclusively reweights complete cases.

The performance of the semi-offline RL IPW estimator (\textit{IPW-Semi-gt}/ \textit{IPW-Semi-Miss-gt}) is contingent on the specific AFA policy under evaluation. Notably, for the 'Random $10\%$' agent, the convergence plot highlights the estimator's pronounced advantages in terms of data efficiency, converging much faster than \textit{IPW-Miss-gt}. This effect is more prominent in the MAR experiment, as confounding adjustment necessitates in MNAR settings the knowledge of additional feature values which reduces the amount of reweighted datapoints. The benefit of using \textit{IPW-Semi-gt}/ \textit{IPW-Semi-Miss-gt} over \textit{IPW-Miss-gt} diminishes when dealing with "data-hungry" agents that acquire many features, as evidenced by nearly identical convergence curves for \textit{IPW-Miss-gt} and \textit{IPW-Semi-gt}/ \textit{IPW-Semi-Miss-gt} in the results for the 'Random $90\%$' policy.

\begin{figure}[h]
\centering	
\vspace{-5pt}
\includegraphics[width= \textwidth]{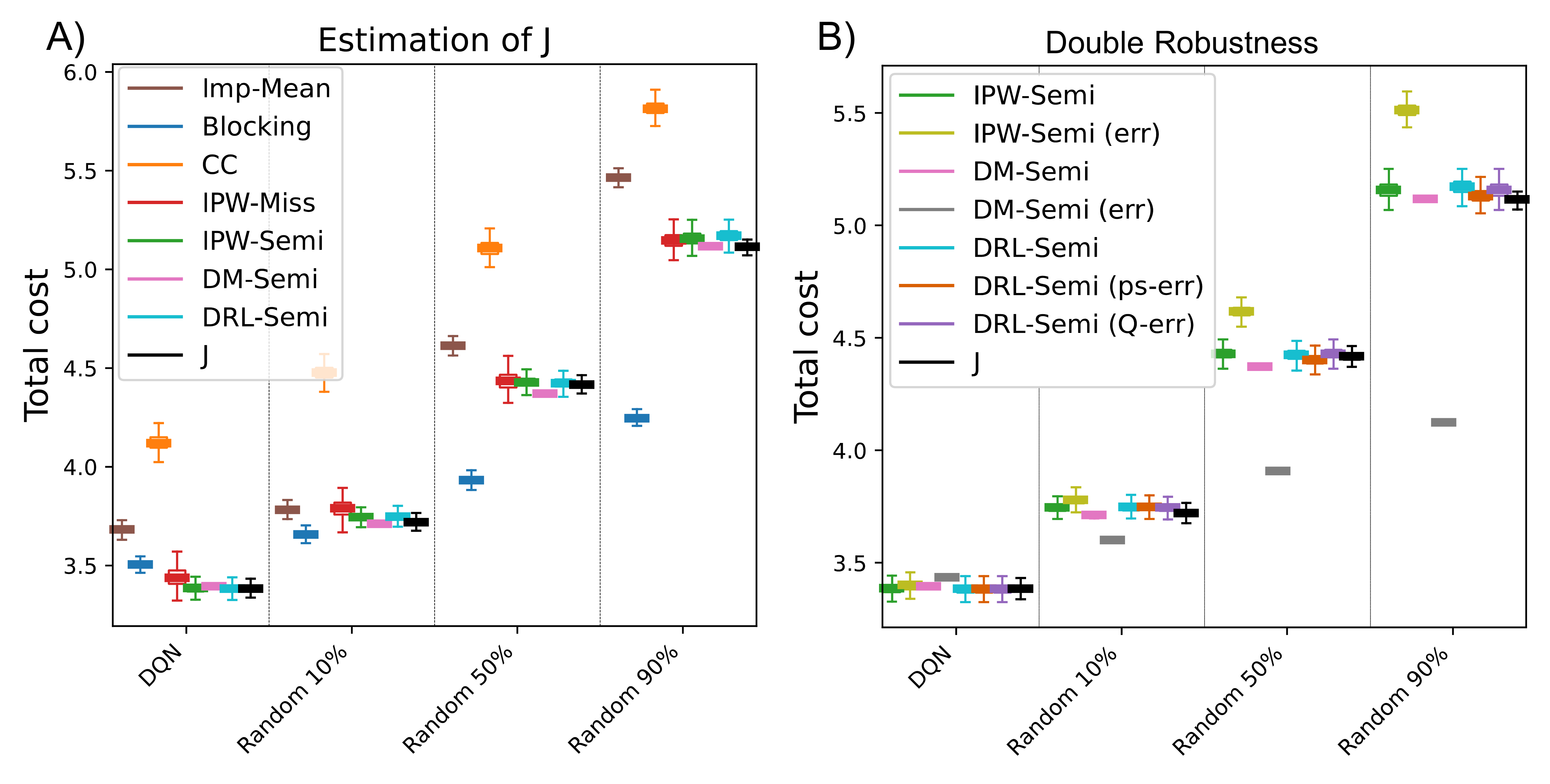}
\vspace{-10pt}
\caption{
A) Different estimators for the MAR experiment. The semi-offline RL estimators accurately estimate the ground truth $J$. 
B) Estimation results from the MAR experiment showcase the double robustness of the DRL estimator. 
}
\label{figure_synthetic_dr}
\end{figure}

Figure \ref{figure_synthetic_dr}A) illustrates the overall performance of various estimators in the synthetic data MAR experiment. Confidence intervals are derived using non-parametric bootstrapping, but due to the computational complexity, nuisance function retraining is excluded. As a result, the obtained confidence intervals appear disproportionately narrow, especially noticeable for the DM estimator. The experiments reveal that all semi-offline RL estimators provide a commendable approximation of the true target parameter $J$. However, biased estimators such as mean imputation, blocking, and complete case analysis consistently fail to accurately estimate $J$.

Figure \ref{figure_synthetic_dr}B) emphasizes the double robustness property of the DRL estimator in the same experiment. This figure highlights that even when one of the nuisance functions is misspecified, the DRL estimator still yields approximately correct estimates for the ground truth $J$.

General estimation results for the real-world MAR experiments are shown in Figure \ref{figure_realworld_MAR}A) for the Heloc and in Figure \ref{figure_realworld_MAR}B) for the Income dataset. These results showcase a substantial alignment between the estimates from the semi-offline RL estimators and the ground truth. A slight bias, can, however, be seen for the semi-offline DM estimator, potentially due to a misspecification of $Q_{Semi}$.

\begin{figure}
	\centering
	\includegraphics[width= \textwidth]{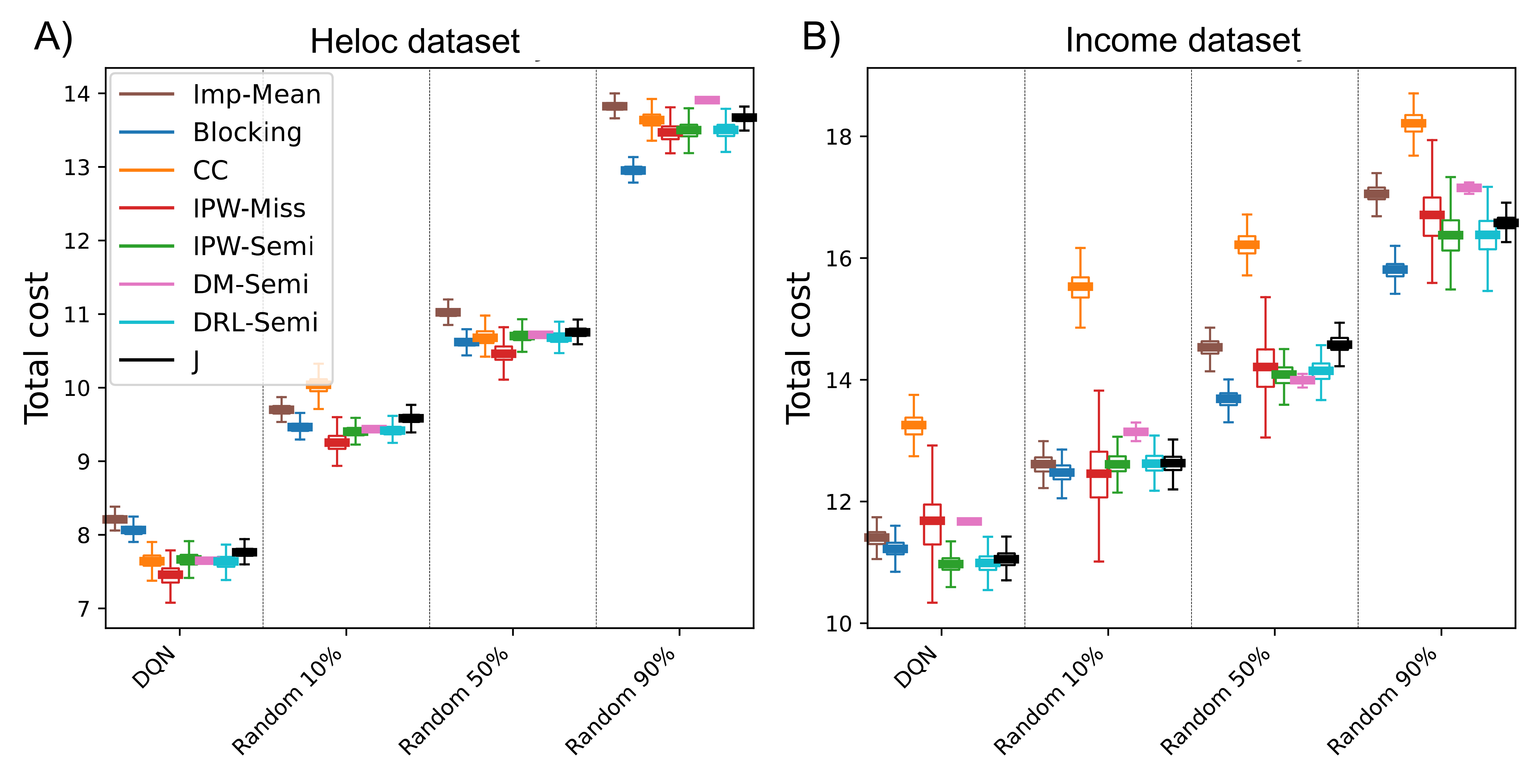}
	\caption{Results for the A) Heloc and B) Income datasets in experiments with induced MAR missingness. The semi-offline RL estimators (\textit{IPW-Semi-gt}, \textit{DM-Semi} and \textit{DRL-Semi}) provide in general accurate estimates for the target parameter $J$. Some bias might, however, be introduced for \textit{DM-Semi} estimator due to a potentially wrong parametric assumption of the utilized model for the Q-function. } 
 \label{figure_realworld_MAR}
\end{figure}
 
\section{Discussion and Future Work}
\label{sec_discussion}

In this study, we extend the developed semi-offline RL concepts, developed for the time-series AFAPE problem in our companion paper \cite{von_kleist_evaluation_2023}, to an AFA setting where a static feature assumption holds. Here, we discuss a few key questions that should be answered before tackling the AFAPE problem. 

\vspace{10pt}
\textit{1) Is the static feature assumption reasonable?} 
In this work, we assume the feature values do not change over time. This allows several benefits in terms of identification, estimation and the general design of AFA agents. Firstly, as shown in this work, it is possible to identify the target parameter even if the order of acquisitions in the retrospective dataset is not known. Secondly, one can design AFA agents that wait for feature acquisition results in order to make decisions on which subsequent feature acquisitions should be taken. 

One may, however, also imagine scenarios where the static feature assumption only holds partially. If the diagnosis of a patient happens, for example, during multiple appointments in the period of weeks or months, one may assume the static feature assumption only holds during the appointments, but not in between appointments. One may then, combine methods discussed here, with the methods for the time-series setting discussed in the companion paper \cite{von_kleist_evaluation_2023}. 

\textit{Conclusion: } The static feature assumption, if reasonable, allows several benefits in terms of identification requirements and estimation efficiency for the AFAPE problem and will allow the design of better AFA systems. 


\vspace{10pt}
\textit{2) What (conditional) independences hold in the data?}
Before choosing a viewpoint, the set of conditional independence assumptions that can be made should be carefully considered. In this work, we make a no direct effect (NDE) assumption that states that feature measurements do not affect the underlying feature. We discuss in our companion paper \cite{von_kleist_evaluation_2023} that one can apply offline RL in general time-series settings if this is not the case. Under NDE, one can, however, not solve the AFAPE problem when the order of feature acquisition is unknown. 

In addition to the NDE assumption, the choice of independence assumptions that make up the missingness process has to be carefully considered. While many MNAR assumptions are identified from the missing data view, they do require stronger positivity assumptions and lead to less efficient estimators compared to the semi-offline RL view. The hybrid missing data + semi-offline RL view allows this trade-off between missingness mechanism assumptions and positivity/ data efficiency to be carefully adjusted. 

\textit{Conclusion:} Under weak MNAR assumptions one has to apply the missing data view. Under the strong MAR assumption, one can apply semi-offline RL. In between, one can consider adopting a hybrid missing data + semi-offline RL view.

\vspace{10pt}
\textit{3) How much exploration was performed by the retrospective missingness policy $\pi_\beta$?}
The missing data view requires a strong requirement of positivity for complete cases. The semi-offline RL view, on the other hand only requires there to be at least as many feature acquisitions in the retrospective data, as are desired under the AFA policy $\pi_\alpha$.
This benefit is therefore stronger for AFA policies that acquire only small portions of the data, while it is less noticable for data hungry AFA policies that acquire almost of all of the data. 

\textit{Conclusion: } The positivity assumption of the semi-offline RL view is significantly weaker than the positivity assumption for the missing data (+ online RL) view. 

\vspace{10pt}
\textit{4) Can the nuisance models be correctly specified and trained?} 
The effectiveness of all estimators hinges upon the precise specification and accurate training of nuisance functions. Consequently, if justified parametric assumptions about the nuisance model can be made that will improve the training, this will also result in better estimation. From the viewpoint of nuisance function modeling, there is therefore no clear best estimator.
For instance, if the multiple imputation model can be readily formulated based on a thorough understanding of the interplay between variables, incorporating justified smoothness properties, the resulting multiple imputation (MI) estimator may surpass all alternatives from the semi-offline RL viewpoint. Similarly, correct parametric assumptions about the propensity score model or the Q-function can confer advantages for semi-offline RL estimators.

\textit{Conclusion: } The performance of any estimator depends on how well the nuisance functions can be modeled and there is no clear best estimator from the viewpoint of nuisance function modeling.

\vspace{10pt}
\textit{5) Is the available dataset size sufficient?} 

Our research shows that adopting the semi-offline RL viewpoint yields notably greater data efficiency when contrasted to the missing data (+ online RL) view. 
To illustrate this effect, consider the amount of trajectories (denoted as $n_\textit{traj}$) that can be simulated from a single data point $X, Y, R$ in the retrospective dataset. 
In the case of the missing data IPW estimator, data points that aren't complete cases receive a weight of 0 ($n_{\textit{traj-Miss}}=0$ if $A \neq \vec{1}$). Therefore, there are no meaningful trajectory simulations possible from a non-complete case. 
However, a complete case allows the simulation of up to $n_{\textit{traj-Miss}} = \sum_{i=0}^{d_x}i \! 
\binom{d_x}{i}$
%
trajectories with possibly non-zero IPW weights. 

In contrast, within the semi-offline RL view, every data point can be leveraged for trajectory simulation. Specifically, there are $n_\textit{traj-Semi} = \sum_{i=0}^{\lVert R \rVert_{1}}i\! \binom{\lVert R \rVert_{1}}{i}$
different trajectories, with $\lVert R \rVert_{1}$ denoting the number of available features. To provide a concrete example, for a data point with $\lVert R \rVert_{1} = 10$ distinct features, this results in a total of approximately $n_\textit{traj-Semi} \approx 10$ million different trajectories.

\textit{Conclusion: } Estimators derived from the semi-offline RL view demonstrate higher data efficiency compared to estimators from the missing data (+ online RL) viewpoint.

\vspace{10pt}
\noindent 
Several general extensions are required for the semi-offline RL framework, and these extensions also need to be applied to the static AFA setting.
Firstly, although we have successfully derived an influence function for the MAR setting, we have not delved into the efficiency of this derived influence function. The task of identifying the efficient influence function for the AFAPE problem is therefore left for future work.

Additionally, the AFAPE problem that has been discussed represents just the initial step toward the development of reliable and optimal AFA agents. Hence, we plan to extend the developed semi-offline RL framework to be able to also address the AFA optimization problem, briefly outlined in Section \ref{sec_afa_optimization_problem}.

\section{Conclusion}
\label{sec_conclusion}

We study the problem of active feature acquisition performance evaluation (AFAPE) in static feature settings.  AFAPE entails estimating the acquisition and misclassification costs that an AFA agent would incur upon deployment, based on retrospective data. We focus on the static feature setting, where the assumption is that feature values do not change over time, enabling the AFA agent to decide on the acquisition of each feature individually, based on the set of previously acquired features. 

We illustrate that even when the sequence of feature acquisitions in the retrospective data is unknown, the AFAPE problem can be solved by two different viewpoints: the missing data (+ online RL) viewpoint, or the novel semi-offline RL approach. Notably, we showcase that the semi-offline RL perspective yields estimators with superior data efficiency and reduced positivity requirements compared to the missing data view. Furthermore, we uncover that both viewpoints can be combined when the underlying missingness mechanism in the retrospective dataset is MNAR.

In conclusion, we substantiate our findings with synthetic data experiments, underscoring the critical importance of employing unbiased estimators for AFAPE. This practice not only safeguards the reliability but also enhances the safety of AFA systems.

\section*{Acknowledgments and Disclosure of Funding}
The present contribution is supported by the Helmholtz Association under the joint research school “HIDSS-006 - Munich School for Data Science @ Helmholtz, TUM \& LMU". Henrik von Kleist received a Carl-Duisberg Fellowship by the Bayer Foundation.

\newpage

\appendix
\section{Literature Review for Active Feature Acquisition (AFA)}
\label{app_AFA_methods}

This appendix contains more details about some common approaches for AFA in static feature settings. These can be generally divided into greedy, information-theoretic approaches, and approaches based on reinforcement learning. 

\textit{Greedy AFA policies: } 
Many AFA approaches are based on a greedy feature acquisition strategy wrapping a subsequent classification task. An idea is to employ decision tree classifiers and to acquire features sequentially by traversing the branch of the decision tree \cite{ling_decision_2004, sheng_feature_2006}, while the splitting criteria minimizes the combined cost of feature acquisition and misclassification  \cite{ling_decision_2004}.
Another approach, the test-cost sensitive Naive Bayes (csNB) classifier \cite{xiaoyong_chai_test-cost_2004}, exploits the Naive Bayes assumption of independence among the predictive power of features. This allows for an efficient exploration of features, acquiring of which can reduce costs.  
Das \textit{et. al} \cite{das_clustering_2021} propose clustering-based cost-aware feature elicitation (CATE). In CATE, data points are clustered based on a set of zero-cost features and the optimal, fixed set of features is computed for each cluster. A new partially-observed data point is then attributed to a cluster, and the corresponding optimal feature set is acquired for it.

\textit{RL-based AFA policies: }
As the AFA problem is inherently a sequential decision process,  one can address it using RL. Model-based RL approaches that leverage the special AFA structure learn an imputation model as a state-transition function \cite{yoon_deep_2018, yin_reinforcement_2020, li_active_2021, li_dynamic_2021, ma_eddi_2019}.
The imputation model is then used at deployment to simulate possible outcomes of a feature acquisition and derive desired acquisition strategies. 
Alternatively, model-free RL approaches do not require learning a state-transition function. One variant, Q-learning, relies on modeling the expected cost of particular acquisition decisions \cite{chang_dynamic_2019, janisch_classification_2020, shim_joint_2018}. As an example, Shim \textit{et. al} \cite{shim_joint_2018} use double Q-learning for the AFA agent with a deep neural network that shares network layers with the subsequent classification neural network.

\section{Review of Semi-parametric Theory}
\label{app_semiparametric_theory}

\noindent 
We provide a brief overview of fundamental concepts in semi-parametric theory. For more detailed explanations, refer to  \cite{tsiatis_semiparametric_2006,bickel_efficient_1993,kennedy_semiparametric_2017}.
Semi-parametric theory seeks data-efficient estimators for a target parameter $J=J(p)$ without imposing overly restrictive assumptions on $p$. We assume access to independent and identically distributed sample $Z_1$,...,$Z_n$ from the random variable $Z$ which is sampled from $p$. 

In many cases, it's possible to obtain $\sqrt{n}$-consistent estimators for $J$ without making numerous assumptions. This makes it often easier to estimate $J$ than to model the entire distribution $p$. 
A key component of semi-parametric theory is the use of influence functions, which characterize asymptotically linear estimators. An estimator $J_{est}$ is considered asymptotically linear with an influence function $\Psi$ (with zero mean and finite variance) if it satisfies the equality 
\cite{tsiatis_semiparametric_2006}:
\begin{align}
\label{eq_asymptotically_linear}
    J_\textit{est}(n) - J = \frac{1}{n} \sum_{i=1}^n \Psi(Z_i) + o_p(\frac{1}{\sqrt{n}})
\end{align}
Due to the central limit theorem, $J_{est}$ is asymptotically normally distributed \cite{tsiatis_semiparametric_2006}: 
\begin{align*}
    \frac{1}{\sqrt{n}}(J_\textit{est}(n) - J) \rightsquigarrow \mathcal{N}\left(0,\mathbb{E}[\Psi^2]\right)
\end{align*}
with  $\rightsquigarrow$ denoting convergence in distribution.

One possible approach to derive an influence function is the path-derivative method. 
It relies on the equality \cite{tsiatis_semiparametric_2006}:
\begin{align*}
    \nabla_\theta J = \mathbb{E}[\Psi S].
\end{align*}
\noindent 
Here, $\theta$ indexes a parametric submodel $p_\theta$ ($\theta=0$ corresponds to the true $p$), $\nabla_\theta J$ is the gradient of the target parameter, and $S = \nabla_\theta \text{log } p_\theta(\mathcal{D})$ is the score function (i.e. the derivative of the log-likelihood over the dataset $\mathcal{D}$).

In some semi-parametric settings, including ours, the influence function is linearly dependent on the target parameter, taking the form $\Psi = f(Z) + J$ for some function $f$. In such cases, a "1-step" estimator can be easily derived using Eq. \ref{eq_asymptotically_linear}:
\begin{align*}
  J_{est} \equiv - J + \frac{1}{n} \sum_{i=1}^n \Psi(Z_i) = -J + \frac{1}{n} \sum_{i=1}^n f(Z_i) + J = \frac{1}{n} \sum_{i=1}^n f(Z_i).
\end{align*}

\clearpage
\section{Glossary of Terms and Symbols}
\label{Appendix_glossary}

\begin{longtable}{@{} p{0.3\textwidth} p{0.67\textwidth} @{}} 
\textbf{Term}           
& 
\textbf{Description}         
\\
\toprule
\textit{AFAPE} &  Active feature acquisition performance evaluation: The problem of estimating the counterfactual cost that would arise if an AFA agent was deployed.
\\
\midrule
\textit{NDE assumption} &  No direct effect assumption: States that the action of measuring a feature does not impact the values of any features or the label.  \\
\midrule
\textit{Semi-offline RL} & Novel framework that allows an agent to interact with the environment (the online part), but forbids the exploration of certain actions (the offline part).
\\
\midrule
\textit{DTR} &  Dynamic treatment regimes 
\\
\midrule
\textit{G-formula} &  Identification formula from causal inference \cite{robins_new_1986}
\\
\midrule
\textit{Plug-in of the G-formula} &  Estimation formula from causal inference that replaces unknown densities in the G-formula with estimated versions \cite{robins_new_1986}.
\\
\midrule
\textit{IPW} &  Inverse probability weighting: Estimator that is also known as importance sampling or the Horvitz-Thompson estimator.
\\
\midrule
\textit{DM} &  Direct method: Estimator based on a Q-function.
\\
\midrule
\textit{DRL} &  Double reinforcement learning: Double robust estimator that uses IPW weights and a Q-function. 
\\
\midrule
\textit{m-graph} &  Missing data graph: Graph to visualize assumptions in missing data problems.
\\
\midrule
\textit{MI} &  Multiple imputation: Estimator for missing data problems that is a special case of the plug-in of the G-formula. 
\\
\midrule
\textit{influence function} & Function of mean zero and finite variance that is used to analyze the asymptotic properties of regular and asymptotically linear (RAL) estimators.
\\
\midrule
\textit{MCAR assumption} &  Missing-completely-at-random assumption: States that the reason for missingness of certain features does not depend on any feature values. \\
\midrule
\textit{MAR assumption} &  Missing-at-random assumption: States that the reason for missingness of certain features does only depend on observed feature values.\\
\midrule
\textit{MNAR assumption} &  Missing-not-at-random assumption: States that the reason for missingness of certain features may depend on feature values that are not observed.
\\
\midrule
\textit{nuisance function } &
Function that needs to be trained from data in order to use a corresponding estimator. Examples are the propensity score model and the Q-function. 
\\
\bottomrule
\end{longtable}

\begin{longtable}{@{} p{0.25\textwidth} p{0.72\textwidth} @{}} 
\textbf{Symbol}           
& \textbf{Description}         
\\
\toprule
\endhead
$t \in (0,...,T)$ &  Time-steps \\
\midrule
$X$ &  Observed feature values  \\
\midrule
$X_{(1)}$ &  Unobserved counterfactual feature values \\
\midrule
$R$ &   Missingness indicator  \\
\midrule
$Y$ &  Label \\
\midrule
$Y^*$ &  Predicted label \\
\midrule
$C_a^t$ &  Acquisition cost for action $A^t$  \\
\midrule
$C_{mc}$ &  Misclassification cost (if $Y$ and $Y^*$ differ)  \\
\midrule
$\pi_\beta$ &  Missingness mechanism (retrospective data) \\
\midrule
$\pi_\alpha$ &  AFA policy \\
\midrule
$C_{mc,(\pi_\alpha)}$ &  Counterfactual misclassification cost had  $\pi_\alpha$ instead of $\pi_\beta$ been applied  \\
\midrule
$g(.)$
& known deterministic distribution\\
\midrule
$g(Y^*|\underline{X}^{T-1},\underline{A}^T)$ &  Classifier predicting $Y^*$ \\
\midrule
$J$ / $J_{mc}$ &  Expected misclassification cost under the AFA policy and classifier \\
\midrule
$J_a$ &  Expected acquisition cost under the AFA policy and classifier \\
\midrule
$\phi_1^*$ , $\phi_2^*$ &  Sets of parameters that parameterize the AFA policy and the classifier, respectively. \\
\midrule
$q(.)$
& counterfactual distribution\\
 \midrule
$\pi'$ &  Blocked policy 
\\
\midrule
$\pi'_{\textit{sim}}$ &  (Blocked) simulation policy \\
\midrule
$p'(.)$
& Simulated distribution\\
\midrule
$C',Y'^*,X',A'$
& Simulated cost, prediction, features and actions\\
\midrule
$R'^t$
& Set of acquired features in simulation until step $t$/ set version of $\underline{A}'^t$\\
 \midrule
$\mathcal{D}$
& Retrospective dataset \\
\midrule
$\mathcal{D}'$
& Simulated dataset \\
\midrule
$Q^t_{\textit{Semi}}$
& State-action value function from semi-offline RL (at time t)
 \\
 \midrule
$V^t_{\textit{Semi}}$
& State value function from semi-offline RL (at time t) \\
 \midrule
$q'(.)$
& counterfactual simulated distribution\\
\midrule
$\Psi$
& Influence function
 \\
\bottomrule
\end{longtable}

\section{Review of Identification in Missing Data Problems }
\label{app_missing_data}

In this appendix, we provide a short review and give an example of how identification can be performed in missing data scenarios, with a focus on MNAR settings. 
The goal in missing data problems is to estimate some function of $p(X_{(1)})$, such as a parameter, using the i.i.d. samples from the observed distribution $p(X,R)$. 
Parameters of interest must be identified, i.e. unique functions of $p(X,R)$, in order to make the estimation problem well-posed.

We will restrict attention to a special type of missing data models where 
identifiability restrictions are represented by a directed acyclic graph (DAG) factorization of the distribution. The restrictions are then formalized as $p(X_{(1)},R) = \prod_{V \in X_{(1)} \cup R} p(V | \text{pa}_{\cal G}(V))$ for some graph ${\cal G}$ where $\text{pa}_{\cal G}(V)$ selects the parents of the node $V$ in $\mathcal{G}$. The graph $\mathcal{G}$ is termed the \emph{missing data graph} (m-graph) \cite{mohan_graphical_2013, shpitser_missing_2015}. Methods of causal inference may then be applied to achieve identification. 
Specific to the medical AFA setting and the process of step-by-step observations, the pattern graph framework~\cite{chen_pattern_2020} is shown to be effective in addressing the missing data problem~\cite{zamanian_assessable_2023}.

Figure \ref{graph_MNAR_scenario} shows the m-graph for a simple MNAR scenario as an example to be discussed next.

\subsection{Identification Example for a Simple MNAR Scenario}

We give an instructive example of how to perform identification for the propensity score $p(R |X_{(1)})$ in a simple MNAR scenario shown in Figure  \ref{graph_MNAR_scenario}. The missing data graph shows all conditional independence assumptions on the individual feature level. We assume here that $X_{(1),2}$ and $X_{(1),3}$ may be missing with missingness indicators $R_2$ and $R_3$, while $X_{(1),1}=X_1$ is fully observed. The underlying missingness scenario is MNAR, because the missingness indicator $R_3$ depends on feature $X_{(1),2}$ which may be missing itself. 

\begin{figure}[h]
\centering
\includegraphics[width=0.55\textwidth]{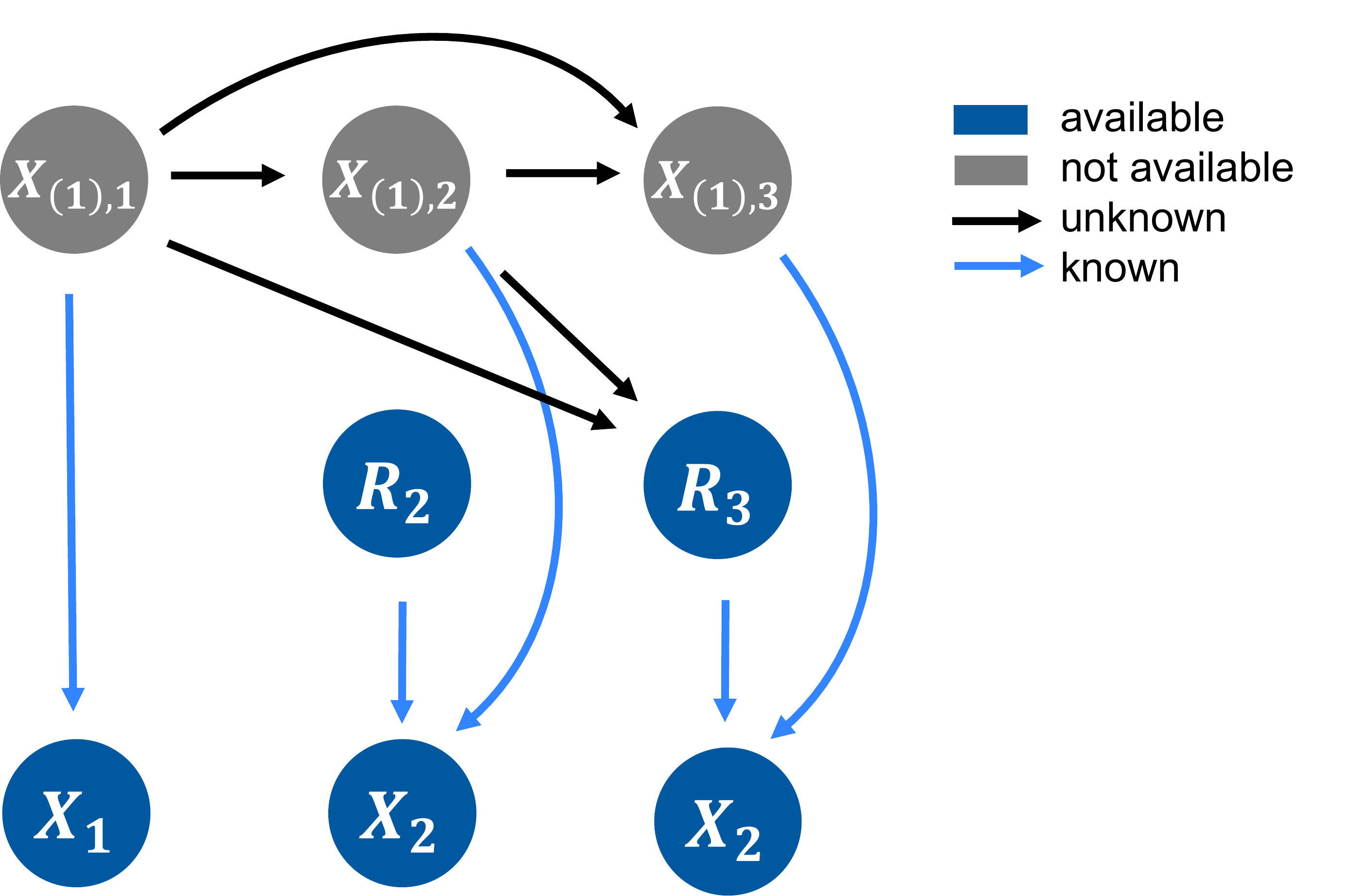}
    \caption{ 
    Missing data graph (m-graph) for a simple MNAR scenario. The scenario is MNAR, because $R_3$ depends on the value of $X_{(1),2}$ that is potentially missing. $X_1$ is presumed to be fully observed and does therefore not have a missingness indicator. 
    }
    \label{graph_MNAR_scenario}
\end{figure}

\vspace{10pt}
\noindent 
\textit{Identification of $p(R|X_{(1)})$:}
\newline
The propensity score is identified by  
\begin{align*}
    p(   R |X_{(1)}) 
    & =  
    p( R_2, R_3|X_1, X_{(1),2}) \nonumber \\
    & = p(R_2|X_1) 
    p(R_3|X_1, X_{(1),2}) \nonumber \\
    & \overset{*_1}{=}
    p(R_2|X_1)  p(R_3|X_1, X_{(1),2}, R_2 = 1)
    \nonumber \\
    & =
    p(R_2|X_1)  p(R_3|X_1, X_{2}, R_2 = 1)
\end{align*}
where we used in $1*)$ the conditional independence $R_3 \indep R_2|X_1,X_{(1),2}$. The final expression is a function of only observed data and is thus identified. It can, however, only be evaluated on datapoints where $R_2= \vec{1}$.

\section{Proof of Theorems \ref{theorem_identification_semi_offline_RL} and \ref{theorem_Bellman_equation}}
\label{app_theorem_identification}

This Appendix contains the proofs for Theorems \ref{theorem_identification_semi_offline_RL} and \ref{theorem_Bellman_equation}. It also shows the factorization of the "observational" (simulated) distribution $p'$ given in Remark $\ref{remark_semi_offline_RL_observational}$ and why the stated positivity assumption is needed for identification. 

\vspace{10pt}
\begin{proof}
We do so by factorizing the counterfactual distribution, denoted by $q'$, in a step-by-step fashion. We also split each time-step in two parts to show how the two parts of the semi-offline RL version of the Bellman equation arises. 
To help with the identification, we also replicate Figure \ref{graph_semi_offline_RL}
of the causal graph 
describing the simulation process in Figure \ref{graph_semi_offline_RL_with_interventions}A). 
Figure \ref{graph_semi_offline_RL_with_interventions}B) contains the counterfactual graph for identification step $t=1$.

\begin{figure}
\centering
\vspace{-5pt}
\includegraphics[width=1\textwidth]
    {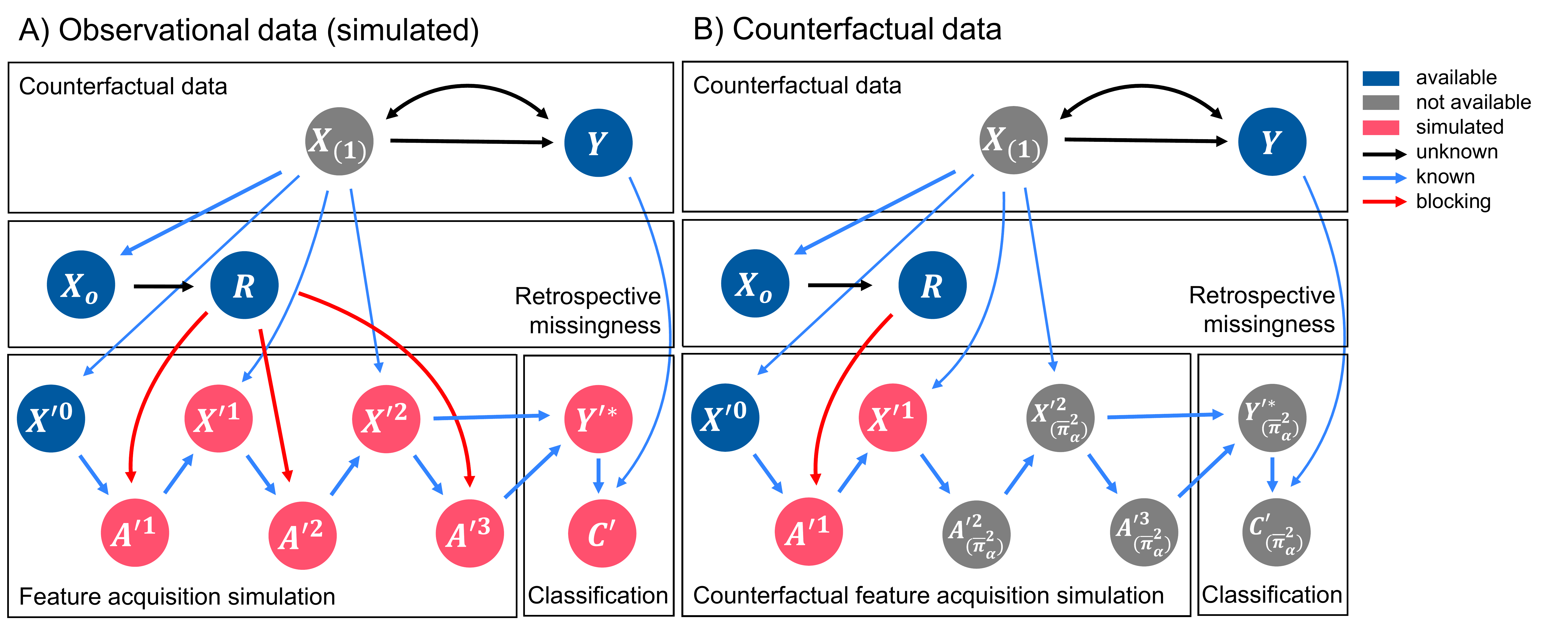}
    \caption{ 
    A) Causal graph depicting the semi-offline RL sampling distribution. B) Causal graph for the counterfactual distribution under an intervention from step $t=2$ onwards. 
   The following edges showing long-term dependencies were excluded from the graph for visual clarity: 
$\underline{X}'^{t-1}/
\underline{X}'^{t-1}_{(\overline{\pi}_{\alpha}^2)}
, \underline{A}'^{t-1} / 
\underline{A}'^{t-1}_{(\overline{\pi}_{\alpha}^2)}
\rightarrow 
A'^{t} / 
A'^{t}_{(\overline{\pi}_{\alpha}^2)}
$; and 
$\underline{X}'^{T-1}/
\underline{X}'^{T-1}_{(\overline{\pi}_{\alpha}^2)}
, \underline{A}'^{T} / 
\underline{A}'^{T}_{(\overline{\pi}_{\alpha}^2)}
\rightarrow 
Y'^* / 
Y'^*_{(\overline{\pi}_{\alpha}^t)}$.
    }
    \label{graph_semi_offline_RL_with_interventions}
\vspace{-1 pt}
\end{figure}

\vspace{10pt}
\noindent 
\textbf{Step 0}

\noindent 
\textit{Counterfactual factorization (step $t = 0$, part 1):}
\begin{align*}   
p'   
\left( 
C'_{(\pi_{\alpha})}
\right)  
\equiv
p'   
\left( 
C'_{(\overline{\pi}^1_{\alpha})}
\right)  
=  & 
\sum_{X_{o} ,X'^0} 
p'
\left(   
C'_{(
\overline{\pi}^1_{\alpha})}
\Big\vert                   
X'^0,
X_{o}
\right)   
p'(X'^0|X_{o})
p(X_{o} ) 
\end{align*}   

\noindent 
where we denote $C'_{(\overline{\pi}^1_{\alpha})}$ as the counterfactual $C'$ under an intervention of $\pi_\alpha$ from step $t=1$ onwards.
Furthermore, we expanded the expression by $X'^0$ and $X_{o}$ which are used for adjustment. 
Further note that 
$p'(
X'^0| X_o 
)= p(
X_{(1),  r'^0}|X_o 
)$  where we let $X_{(1), r'^0}$ denote the $X_{(1)}$ at the indexes that are always revealed at step $0$.

\vspace{10pt}
\noindent 
\textit{Counterfactual factorization (step $t = 0$, part 2):}
\begin{align*}   
p'
\left(   
C'_{(   
\overline{\pi}^1_{\alpha})
}
\Big\vert                   
X'^0,
X_{o}                             
\right) 
& 
= 
\boldsymbol{\sum_{a'^1} }
p'
\left(   
C'_{\boldsymbol{(   
\overline{\pi}^2_{\alpha},
a'^1)}
}        
\Big\vert                        X'^0,
X_{o}                      
\right) 
\boldsymbol{\pi_{\alpha}(
a'^1 | X'^0) }
\\ & 
\overset{*_1}{=}
\sum_{a'^1} 
p'
\left(   
C'_{(   
\overline{\pi}^2_{\alpha},
a'^1,
\boldsymbol{\pi_{{id}}})
}        
\Big\vert                   
X'^0,
X_{o}                             
\right) 
\pi_{\alpha}(
a'^1 | X'^0)  
\\ & 
= 
\sum_{a'^1, \boldsymbol{r}} 
p'
\left(   
C'_{(   
\overline{\pi}^2_{\alpha},
a'^1,
\boldsymbol{r})
}        
\Big\vert                   
X'^0,
X_{o}                             
\right) 
\boldsymbol{\pi_{id}(
r | X_o, a'^1) }  
\pi_{\alpha}(
a'^1 | X'^0)   
\\ & 
\overset{*_2}{=} 
\sum_{a'^1, r} 
p'
\left(   
C'_{(   
\overline{\pi}^2_{\alpha},
a'^1,
r)
}        
\Big\vert                   
X'^0,
X_{o}                             
\right) 
\boldsymbol{\pi_{\beta}(
r | r \geq r'^1, X_o)}   
\pi_{\alpha}(
a'^1 | X'^0)   
\\ &
\overset{*_3}{=}
\sum_{a'^1, r} 
p'
\left(   
C'_{(   
\overline{\pi}^2_{\alpha},
a'^1,
r)
}        
\Big\vert                   
X'^0,
X_{o}, 
\boldsymbol{r}
\right) 
\pi_{\beta}(
r | r \geq r'^1, X_o)   
\pi_{\alpha}(
a'^1 | X'^0)  
\\ &
\overset{*_4}{=}
\sum_{a'^1, r} 
p'
\left(   
C'_{\boldsymbol{(   
\overline{\pi}^2_{\alpha},
a'^1)}
}        
\Big\vert                   
X'^0,
X_{o}, 
r
\right) 
\pi_{\beta}(
r | r \geq r'^1, X_o)   
\pi_{\alpha}(
a'^1 | X'^0)  
\\ &
\overset{*_5}{=}
\sum_{a'^1, r} 
p'
\left(   
C'_{(   
\overline{\pi}^2_{\alpha},
a'^1)
}        
\Big\vert                   
X'^0,
\boldsymbol{a'^1},
X_{o}, 
r
\right) 
\pi_{\beta}(
r | r \geq r'^1, X_o)   
\pi_{\alpha}(
a'^1 | X'^0) 
\\ &
\overset{*_6}{=}
\sum_{a'^1, r} 
p'
\left(   
C'_{\boldsymbol{(   
\overline{\pi}^2_{\alpha})}
}        
\Big\vert                   
X'^0,
a'^1,
X_{o}, 
r
\right) 
\pi_{\beta}(
r | r \geq r'^1, X_o)   
\pi_{\alpha}(
a'^1 | X'^0) 
\\ &
\overset{*_7}{=}
\sum_{a'^1, r} 
p'
\left(   
C'_{(   
\overline{\pi}^2_{\alpha})
}        
\Big\vert                   
\boldsymbol{X'^0,
a'^1,
X_{o}}
\right) 
\pi_{\beta}(
r | r \geq r'^1, X_o)   
\pi_{\alpha}(
a'^1 | X'^0) 
\\ &
=
\sum_{\boldsymbol{a'^1}} 
p'
\left(   
C'_{(   
\overline{\pi}^2_{\alpha})
}        
\Big\vert                   
X'^0,
a'^1,
X_{o}
\right)  
\pi_{\alpha}(
a'^1 | X'^0) 
\end{align*}

\noindent 
where we consider only acquisition actions $a'^1 \leq d_x$, i.e. not the "stop \& predict" action. We will cover that action at the end. 
The following points justify the steps in more detail:

\begin{itemize}
    \item $*1)$: We notice that $C'_{(\overline{\pi}^2_\alpha,a'^1)} $ is independent of any interventions $\pi_{id}$ on $R$. 
    \item $*2)$: We choose 
    $\pi_{id}(r | X_o, a'^1)  = \pi_{\beta}(
    r | r \geq r'^1, X_o)$ where we let $r'^1$ denote the set notation of $a'^1$. This means $r'^1$ denotes the missingness indicator for the acquired features until step 1.  This choice for $\pi_{id}$ will prevent positivity violations. 
    \item $*3)$: 
    We use the exchangeability 
    $C'_{(\overline{\pi}^2_{\alpha},
    a'^1, r)} \perp\!\!\!\perp  
    R |
    X'^0, X_o $ which follows from the MAR assumption. 
    \item $*4)$: 
    We use the consistency assumption:    \newline
    $p'
    \left(  
    C'_{(\overline{\pi}^2_{\alpha}, 
    a'^1, r)}     
    \Big\vert                   
    X'^0, 
    X_o,
    r
    \right) 
    = 
    p'
    \left(  
    C'_{(\overline{\pi}^2_{\alpha}, 
    a'^1)}     
    \Big\vert                   
    X'^0, 
    X_o,
    r
    \right)$ 
    \item $*5)$: We use the exchangeability: 
    $C'_{(\overline{\pi}^2_{\alpha},
    a'^1)} \perp\!\!\!\perp  
    A'^1 |
    X'^0, X_o, R$ 
    \item $*6)$: We use the consistency assumption: 
    \newline
    $p'
    \left(  
    C'_{(\overline{\pi}^2_{\alpha}, 
    a'^1)}     
    \Big\vert                   
    X'^0, 
    a'^1,
    X_o,
    r 
    \right) 
    = 
    p'
    \left(  
    C'_{(\overline{\pi}^2_{\alpha})}     
    \Big\vert                   
    X'^0,
    a'^1,
    X_o,
    r
    \right)$ 
    \item $*7)$: We use the conditional independence 
    $C'_{(\overline{\pi}^2_{\alpha})} \perp\!\!\!\perp  
    R |
    X'^0, A'^1, X_o$
\end{itemize}

\noindent 
The identification step also requires that the conditioning on $X'^0, A'^1, X_o, R$ in $*5)$ is well specified. To state the necessary positivity assumption, we start by factorizing the "observational" (i.e. simulated) distribution for step $t=0$. 
In particular, we examine the following distribution which excludes an intervention on $A'^1$. 

\vspace{10pt}
\noindent 
\textit{Observational factorization (step $t = 0$):}
\begin{align*}  
p'   
\left( 
C'_{(\overline{\pi}^2_{\alpha})}
\right)  
= 
\sum_{X'^0,X_o, R, A'^1}
p'
\left(   
C'_{(   
\overline{\pi}^2_{\alpha})
}        
\Big\vert                   
X'^0,
A'^1 , 
X_o
\right) 
\underbrace{
\pi_{sim}'(
A'^1
|  
X'^0,
R
)}  
_{\text{simulation policy}} 
\underbrace{
\pi_{\beta}(
R
| 
X_o)
}_{\text{retro. missingness}}  
p(X'^0| X_o)
p(X_o)
\end{align*}

\noindent 
The following positivity assumption arises: 
\begin{flalign}
\text{if }  \hspace{11 pt} \quad \quad \quad \quad \quad 
&q'(  x'^0,a'^1, r, x_o)  =  
\nonumber
p(x'^0,x_o) 
\pi_{\alpha}(
 a'^1 | x'^0)     
\pi_{\beta}(
r | r \geq r', x_o)  
>0
\nonumber
&&
\\
\text{then } \quad \quad \quad \quad \quad 
&
p'( x'^0, a'^1, r, x_o) =  
\nonumber
p(x'^0,x_o) 
\pi_{sim}'(
a'^1 | x'^0, r)  
\pi_{\beta}(
r | x_o)  
>0
\nonumber
&&
\\ 
&
\forall 
x'^0,
a'^1,
x_o,
r
\end{flalign}

\noindent 
Comparing the terms for $A'^1$ shows: 
\begin{flalign*}
 &\text{if } 
\pi_{\alpha}(
a'^1 | x'^0) > 0, \quad \quad \quad \quad 
 \quad \quad  \text{then }
\pi_{sim}'(
a'^1 | x'^0, 
r) > 0,  \quad \quad 
&& \text{if and only if } r'^1 \leq r   
\end{flalign*}
due to the definition of a blocked policy (Definition \ref{def_blocked_policy}). 
This does, however, not lead to any positivity issues, because 
$\pi_{\beta}(
r | r \geq r'^1, x_o) = 0$ if $r'^1 \not \leq r$. 
Furthermore, comparing the terms for $R$, we see that there are no problems, since if  $\pi_{\beta}(
r| r \geq r'^1, x_o) > 0$, then $\pi_{\beta}(
r | x_0 
) > 0$. 
Thus, in order for the positivity assumption to hold, the only requirement is that $\pi_{\beta}(
r| r \geq r'^1, X_o)$ is a valid density, which is the case if $\pi_{\beta}(r \geq r'^1| X_o) > 0$. This condition is fulfilled by the positivity assumption from Eq. \ref{eq_positivity_semi_offline_RL}.

\vspace{10pt}
\noindent 
\textbf{Step t}

\noindent 
In the following, we extent the identification to the general step $t$. 

\vspace{10pt}
\noindent 
\textit{Counterfactual factorization (step $t$, part 1):}
\begin{align}     
\label{eq_bellman_1_derivation}
p'   
\biggl( 
C'_{(
\overline{\pi}^{t+1}_{\alpha})} 
\Big\vert  
\underline{X}'^{t-1}, \underline{A}'^t, X_o
\biggl)  
&= 
\sum_{X'^t} 
p'
\left(   
C'_{(
\overline{\pi}^{t+1}_{\alpha}
)} 
\Big\vert                   
\underline{X}'^t,
\underline{A}'^t,
X_o
\right) 
p'(
X'^t|\underline{X}'^{t-1},A'^t, X_o 
)   
\end{align}   

\noindent
where we expanded the expression with $X'^t$ which is needed for adjustment. 
Note further that 
$p'(
X'^t|\underline{X}'^{t-1},A'^t, X_o 
)= p(
X_{(1), a'^t}|X_{(1),r'^{t-1}}, X_o 
)$  where we let $X_{(1), r'^t}$ denote the variable $X_{(1)}$ at all indices $i$ where $r'^t_i = 1$.

\vspace{10pt}
\noindent 
\textit{Counterfactual factorization (step $t$, part 2):}
\begin{align}   
\label{eq_bellman_2_derivation}
p'
\biggl(   
C'_{(   
\overline{\pi}^{t+1}_{\alpha})
} 
\Big\vert                   
\underline{X}'^t, 
\underline{A}'^t,
X_o                           
\biggl) 
& = 
\boldsymbol{\sum_{a'^{t+1}}  }
p'
\left(   
C'_{\boldsymbol{(   
\overline{\pi}^{t+2}_{\alpha},
a'^{t+1})}
}        
\Big\vert                   
\underline{X}'^t,
\underline{A}'^t,
X_o                            
\right) 
\boldsymbol{\pi_{\alpha}(
a'^{t+1} | \underline{X}'^t,
\underline{A}'^t)  }
\nonumber
\\ & 
\overset{*_1}{=}
\sum_{a'^{t+1}} 
p'
\left(   
C'_{(   
\overline{\pi}^{t+2}_{\alpha},
a'^{t+1},
\boldsymbol{\pi_{id}})
}        
\Big\vert                   
\underline{X}'^t,
\underline{A}'^t,
X_o                                 
\right) 
\pi_{\alpha}(
a'^{t+1} | \underline{X}'^t,
\underline{A}'^t)  
\nonumber
\\ & 
= 
\sum_{a'^{t+1}, \boldsymbol{r}} 
p'
\left(   
C'_{(   
\overline{\pi}^{t+2}_{\alpha}, 
a'^{t+1},
\boldsymbol{r})
}        
\Big\vert                   
\underline{X}'^t,
\underline{A}'^t,
X_o                                   
\right) 
\boldsymbol{\pi_{id}(
r | 
\underline{A}'^t, 
a'^{t+1}, 
X_o)   }
\pi_{\alpha}(
a'^{t+1} | \underline{X}'^t,
\underline{A}'^t)  
\nonumber
\\ & 
\overset{*_2}{=} 
\sum_{a'^{t+1}, r} 
p'
\left(   
C'_{(   
\overline{\pi}^{t+2}_{\alpha},
a'^{t+1},
r)
}        
\Big\vert                   
\underline{X}'^t,
\underline{A}'^t,
X_o                                 
\right) 
\boldsymbol{\pi_{\beta}(
r | r \geq r'^{t+1}, X_o)}   
\pi_{\alpha}(
a'^{t+1} | \underline{X}'^t,
\underline{A}'^t) 
\nonumber
\\ &
\overset{*_3}{=}
\sum_{a'^{t+1}, r} 
p'
\left(   
C'_{\boldsymbol{(   
\overline{\pi}^{t+2}_{\alpha},
a'^{t+1})}
}        
\Big\vert                   
\underline{X}'^t,
\underline{A}'^t,
X_o  ,      
\boldsymbol{r}
\right) 
\pi_{\beta}(
r | r \geq r'^{t+1}, X_o)   
\pi_{\alpha}(
a'^{t+1} | \underline{X}'^t,
\underline{A}'^t) 
\nonumber
\\ &
\overset{*_4}{=}
\sum_{a'^{t+1}, r} 
p'
\left(   
C'_{\boldsymbol{(   
\overline{\pi}^{t+2}_{\alpha})}
}        
\Big\vert                   
\underline{X}'^t,
\underline{A}'^t,
\boldsymbol{a'^{t+1}},
X_o  ,    
r
\right) 
\pi_{\beta}(
r | r \geq r'^{t+1}, X_o)   
\pi_{\alpha}(
a'^{t+1} | \underline{X}'^t,
\underline{A}'^t) 
\nonumber
\\ &
\overset{*_5}{=}
\sum_{a'^{t+1}, r} 
p'
\left(   
C'_{(   
\overline{\pi}^{t+2}_{\alpha})
}        
\Big\vert                   
\boldsymbol{\underline{X}'^t,
\underline{A}'^t,
a'^{t+1},
X_o  }    
\right) 
\pi_{\beta}(
r | r \geq r'^{t+1}, X_o)   
\pi_{\alpha}(
a'^{t+1} | \underline{X}'^t,
\underline{A}'^t) 
\nonumber
\\ &
=
\sum_{\boldsymbol{a'^{t+1}}} 
p'
\left(   
C'_{(   
\overline{\pi}^{t+2}_{\alpha})
}        
\Big\vert                   
\underline{X}'^t,
\underline{A}'^t,
a'^{t+1},
X_o      
\right)  
\pi_{\alpha}(
a'^{t+1} | \underline{X}'^t,
\underline{A}'^t) 
\end{align}

\noindent 
where we again only consider only acquisition actions $a'^{t+t} \leq d_x$, i.e. not the "stop \& predict" action which we will cover at step $t = T$. 
The following points justify the steps in more detail:

\begin{itemize}
    \item $*1)$: We use that $C'_{(\overline{\pi}^{t+2}_\alpha,a'^{t+1})} $ is independent of any interventions $\pi_{id}$ on $R$. 
    \item $*2)$: We choose 
    $\pi_{id}(r | \underline{a}'^{t+1}, X_o)  = \pi_{\beta}(
    r | r \geq r'^{t+1}, X_o)$ where we let $r'^{t+1}$ denote the set notation of $\underline{a}'^{t+1}$, corresponding to the missingness indicator for all acquired features until step $t+1$. 
    \item $*3)$: 
    We use exchangeability: 
    $C'_{(\overline{\pi}^{t+2}_{\alpha},
    a'^{t+1}, r)} \perp\!\!\!\perp  
    R |
    \underline{X}'^t,\underline{A}'^t , X_o $ and consistency: \newline
    $p'
    \left(  
    C'_{(\overline{\pi}^{t+2}_{\alpha}, 
    a'^{t+1}, r)}     
    \Big\vert                   
    \underline{X}'^t, 
    \underline{A}'^t, 
    X_o,
    r
    \right) 
    = 
    p'
    \left(  
    C'_{(\overline{\pi}^{t+2}_{\alpha}, 
    a'^{t+1})}     
    \Big\vert                   
    \underline{X}'^t, 
    \underline{A}'^t,  
    X_o,
    r
    \right)$. 
    \item $*4)$: We use exchangeability: 
    $C'_{(\overline{\pi}^{t+2}_{\alpha},
    a'^{t+1})} \perp\!\!\!\perp  
    A'^{t+1} |
    \underline{X}'^t, 
    \underline{A}'^t, X_o, R$ and consistency: 
    \newline
    $p'
    \left(  
    C'_{(\overline{\pi}^{t+2}_{\alpha}, 
    a'^{t+1})}     
    \Big\vert                   
    \underline{X}'^t, 
    \underline{A}'^t,  
    a'^{t+1},
    X_o,
    r 
    \right) 
    = 
    p'
    \left(  
    C'_{(\overline{\pi}^{t+2}_{\alpha})}     
    \Big\vert                   
    \underline{X}'^t, 
    \underline{A}'^t, 
    a'^{t+1},
    X_o,
    r
    \right)$ .
    \item $*5)$: We use the conditional independence 
    $C'_{(\overline{\pi}^{t+2}_{\alpha})} \perp\!\!\!\perp  
    R |
    \underline{X}'^t, 
    \underline{A}'^t, A'^{t+1}, X_o$
\end{itemize}

\noindent 
To ensure in $4*)$ that $p'
\left(  
C'_{(\overline{\pi}^{t+2}_{\alpha}, 
a'^{t+1})}     
\Big\vert                   
\underline{X}'^t, 
\underline{A}'^t,
a'^t, 
X_o,
r 
\right) 
$, i.e. conditioning on $\underline{X}'^t, 
\underline{A}'^{t+1},  
X_o,
r $ is well specified, a positivity assumption must hold. 
We again start by factorizing the "observational" (i.e. simulated) distribution for step $t$:

\vspace{10pt}
\noindent 
\textit{Observational factorization (step $t $):}

\begin{align*}  
& p' 
\big( 
C'_{(\overline{\pi}^{t+2}_{\alpha})}          
\big) 
= 
\sum_{
\underline{X}'^t, 
\underline{A}'^t, 
A'^{t+1}, 
R}
p'   
\big(
C'_{(\overline{\pi}^{t+2}_{\alpha})} 
\Big\vert                   
\underline{X}'^t,
\underline{A}'^{t+1},
R,    
X_o
\big)
p'(\underline{X}'^t,
\underline{A}'^{t+1},
R , 
X_o) 
\end{align*}
\noindent 
where 
\begin{align*}  
p'( \underline{X}'^t,& 
\underline{A}'^{t+1},
R , 
X_o) =
\prod_{t=1}^{t+1}
\pi_{sim}'(
A'^{t}
|  
\underline{X}'^{t-1},
\underline{A}'^{t-1},
R
)  
\prod_{t=0}^t
p(X'^t|\underline{X}'^{t-1}, A'^t, X_o) 
\pi_\beta(R|X_o)
p(X_o) 
\end{align*}



\noindent 
By comparing the observational and counterfactual factorizations, we see that the following positivity assumption is required: 
\begin{flalign*}
\text{if }  \hspace{11 pt} \quad \quad \quad 
&
q'(  \underline{x}'^{t},\underline{a}'^{{t+1}},  r, x_o)  =  
\nonumber
\prod_{\tau=1}^{t+1}
\pi_{\alpha}(
a'^\tau
|  
\underline{x}'^{\tau-1},
\underline{x}'^{\tau-1}
)  
\prod_{\tau=0}^t
p(x'^\tau|\underline{x}'^{\tau-1}, a'^\tau, x_o) 
\pi_\beta(r|r \geq r'^{t+1}, x_o)
p(x_o) 
>0
\nonumber
&&
\\
\text{then } \quad \quad \quad 
&
p'(\underline{x}'^{t},\underline{a}'^{{t+1}},  r, x_o) =  
\nonumber
\prod_{\tau=1}^{t+1}
\pi_{sim}'(
a'^\tau
|  
\underline{x}'^{\tau-1},
\underline{a}'^{\tau-1},
r
)  
\prod_{\tau=0}^t
p(x'^\tau|\underline{x}'^{\tau-1}, a'^\tau, x_o) 
\pi_\beta(r|x_o)
p(x_o) 
>0
\nonumber
&&
\\ 
&
\forall 
\underline{x}'^t,
\underline{a}'^{t+1},
x_o,
r
\end{flalign*}

\noindent 
Comparing the terms for $A'^\tau$ shows: 
\begin{flalign*}
 &\text{if } 
\pi_{\alpha}(
a'^\tau | \underline{x}'^{\tau-1}, \underline{a}'^{\tau-1}) > 0, \quad \quad
 \text{then }
\pi_{sim}'(
a'^\tau | \underline{x}'^{\tau-1}, \underline{a}'^{\tau-1}, 
r) > 0,  \quad \quad 
&& \text{if and only if } r'^\tau \leq r   
\end{flalign*}
due to the definition of a blocked policy (Definition \ref{def_blocked_policy}). 
This does again not lead to any positivity issues, because 
$\pi_{\beta}(
r | r \geq r'^{t+1}, x_o) = 0$ if $r'^{t+1} \not \leq r$ and since this holds for $r'^{t+1}$, it will hold also for all $r'^\tau$ with $\tau < {t+1}$, because these contain less feature acquisitions.

Furthermore, comparing the terms for $R$, we see that there are no problems, since if  $\pi_{\beta}(
r| r \geq r'^{t+1}, x_o) > 0$, then $\pi_{\beta}(
r | x_o 
) > 0$. 
Thus, in order for the positivity assumption to hold, the only requirement is that $\pi_{\beta}(
r| r \geq r'^{t+1}, x_o)$ is a valid density, which is the case if $\pi_{\beta}(r \geq r'^{t+1}| x_o) > 0$. This requirement becomes strictly stronger as $t$ increases and more features are acquired. It is maximal at step $t = T-1$ where it leads to the positivity assumption stated in Eq. \ref{eq_positivity_semi_offline_RL}.

\vspace{10pt}
\noindent 
\textbf{Step T}

\noindent 
Now, we finish the identification with the last step $T$. We start from part 2 of the factorization for step $t$ and assume $a'^T = d_x +1 $, i.e. the "stop \& predict" action.

\vspace{10pt}
\noindent 
\textit{Counterfactual factorization (step $t$, part 2):}

\begin{align*}   
p'\left(C'_{(   
\overline{\pi}^{T+1}_{\alpha},
A'^{T} = {d_x+1})
}        
\Big\vert                   
\underline{X}'^{T-1},
\underline{A}'^{T-1},
X_o                            
\right) 
=
p'
\left(C'     
\Big\vert                   
\underline{X}'^{T-1},
\underline{A}'^{T-1},
A'^{T} = {d_x+1},
X_o                            
\right) 
\end{align*}

\noindent 
which holds again due to exchangeability and consistency. There is no need for other confounding adjustment. 

\vspace{10pt}
\noindent
\textbf{Full factorization}

\noindent
Bringing all time-steps $t = 0, ... , T$ together, and expanding for $Y$, yields the full factorization of the identifying distribution $q'$:
\begin{align*} 
q' & (C', Y, X', A', R,    
X_o )  =  g(C'| Y , \underline{X}'^{T-1}, \underline{A}'^T)           
p'(Y| \underline{X}'^{T-1}, \underline{A}'^{T-1} ,   X_o)
q'(
\underline{X}'^{T-1},          
\underline{A}'^{T},  
R,    
X_o 
) 
\end{align*}
where 
\begin{align*} 
q'(\underline{X}'^{T-1},        
\underline{A}'^{T},        
R,    
X_o )=
\pi_\beta(R| 
R \geq R', 
X_o)
\prod_{t=1}^{T}  
\pi_{\alpha}(A'^t|             
\underline{X}'^{t-1},
\underline{A}'^{t-1})  
\prod_{t=0}^{T-1}
p'(X'^t|   
\underline{X}'^{t-1},        
A'^{t}, 
X_o
) 
p(X_o)
\nonumber
\end{align*}

\noindent where 
$R' = R'^T$ denotes the final set of acquired features. 
In comparison, the full observational factorization (given in Remark \ref{remark_semi_offline_RL_observational}) is: 

\begin{align*} 
p' (C', Y, X', A',R, X_o)  =  g(C'| Y , \underline{X}'^{T-1}, \underline{A}'^T)           
p'(Y| \underline{X}'^{T-1}, \underline{A}'^{T-1},  X_o)
p'(
\underline{X}'^{T-1},          
\underline{A}'^{T},         
R, X_o) 
\end{align*}
where 
\begin{align*} 
p'(\underline{X}'^{T-1},        
\underline{A}'^{T},   
R, X_o)=
\prod_{t=1}^{T}  
\pi_{sim}'(A'^t|             
\underline{X}'^{t-1},
\underline{A}'^{t-1}, R)  
\prod_{t=0}^{T-1}
p'(X'^t|   
\underline{X}'^{t-1},        
A'^{t}, 
X_o
) 
\pi_\beta(R| 
X_o)
p(X_o)
\nonumber
\end{align*}

\noindent 
The following final positivity requirement arises (given by Eq \ref{eq_positivity_semi_offline_RL}): 

\begin{flalign*}
\nonumber
\text{if }  \hspace{11 pt} \quad \quad \quad
& q'(
\underline{x}'^{T-1},
\underline{a}'^T, 
x_o, 
r
) = 
\prod_{t=1}^{T}  
\pi_{\alpha}(a'^t|             
\underline{x}'^{t-1},
a'^{t-1})  
\prod_{t=0}^{T-1}
p'(x'^t|   
\underline{x}'^{t-1},        
\underline{a}'^{t},
x_o
) 
p(x_o)  
>0
\nonumber
&&
\\
\text{then } \quad \quad \quad
& 
\pi_\beta(r \geq r'| 
x_o) 
>0
\nonumber
&&
\\ 
&
\forall 
\underline{x}'^{T-1},
\underline{a}'^T / r'^T,
x_o,
r
\end{flalign*}

%
%
%
\noindent 
Theorem \ref{theorem_identification_semi_offline_RL} is now proven. 


%


\vspace{10pt}
\noindent 
\textbf{Bellman equation}

\noindent 
By extending Eqs. \ref{eq_bellman_1_derivation} and \ref{eq_bellman_2_derivation} to expected values, one arrives at the semi-offline RL Bellman equation:

\begin{align} 
\mathbb{E}   
\biggl[ 
C'_{(\overline{\pi}^{t+1}_{\alpha})} 
 \Big\vert  
\underline{X}'^{t-1}, 
\underline{A}'^{t}, 
X_o
\biggl] 
=
&   
\sum_{X'^t} 
\mathbb{E}
\left[   
C'_{(\overline{\pi}^{t+1}_{\alpha})} 
\Big\vert                   
\underline{X}'^t,
\underline{A}'^t,
X_o
\right]
p'(X'^t
|\underline{X}'^{t-1},A'^t,
X_o) 
\nonumber  
\\
\mathbb{E} \biggl[   
C'_{(
\overline{\pi}^{t+1}_{\alpha}
)} 
\Big\vert                   
\underline{X}'^t,
\underline{A}'^t,
X_o
\biggl] 
= 
&   
\nonumber
\sum_{
A'^{t+1}
} 
\mathbb{E}
\left[   
C'_{(\overline{\pi}^{t+2}_{\alpha} )} 
\Big\vert                   
\underline{X}'^t,
\underline{A}'^{t+1},
X_o
\right]
\pi_{\alpha}(A'^{t+1}
| \underline{X}'^t,
\underline{A}'^t) 
\end{align}

\noindent 
which concludes the proof of Theorem \ref{theorem_Bellman_equation}.
\end{proof}

\section{Proof of Theorem \ref{theorem_double_robustness}}
\label{app_theorem_double_robustness}

In this appendix, we proof Theorem \ref{theorem_double_robustness} stating the double robustness property of the semi-offline DRL estimator. 
\vspace{10pt}
\begin{proof}
We decompose 
$J_{\textit{DRL-Semi}}$ for the two different scenarios of one of the two nuisance functions being misspecified:

\vspace{10pt}
\noindent 
\textit{Scenario 1:} First, we look at the scenario where  $\hat{\pi}_\beta$ is correctly specified. The following decomposition holds: 
\begin{align*}                  
J_{\textit{DRL-Semi}} 
& = 
\underbrace{\mathbb{E}_{p'}[\rho_\textit{Semi}^T 
C'] }_{=J}  
+ 
\sum_{t=1}^{T} 
\underbrace{\mathbb{E}_{p'}
\left[
-\rho_{\textit{Semi}}^{t}  
 Q_{\textit{Semi}}^{t}
 +
 \rho_{\textit{Semi}}^{t-1}  
 V_{\textit{Semi}}^{t-1}
 \right]}_{=0}.  
\end{align*} 

\noindent 
When $\hat{\pi}_\beta$ is correctly specified, the first term, equal to the semi-offline RL IPW estimator, will consistently estimate $J$. 
The second term becomes 0 as shown in more detail in the following:
\begin{align*} 
\mathbb{E}_{p'}
\biggl[ 
-\rho_{\textit{Semi}}^{t}  
 Q_{\textit{Semi}}^{t}
 +& 
 \rho_{\textit{Semi}}^{t-1}  
 V_{\textit{Semi}}^{t-1}
 \biggl] 
 = 
 \\ & 
 \overset{*_1}{=} 
 \mathbb{E}_{p'}
\biggl[
\rho_{\textit{Semi}}^{\boldsymbol{t-1}}
 \biggl(
 - 
    \boldsymbol{\frac{
    \pi_\alpha(A'^{t}| \underline{X}'^{t-1}, \underline{A}'^{t-1}) 
    }
    {
    \pi'_{sim}(
    A'^{t}| \underline{X}'^{t-1}, \underline{A}'^{t-1},
    R
    )
    }
    \frac{
    \mathbb{I}(R \geq R'^t)
    }
    {
    \hat{\pi}_\beta(
    R \geq R'^t| R \geq R'^{t-1},
    X_o
    )
    }}
Q_{\textit{Semi}}^t
 \\ & \quad \quad  \quad \quad \quad  \quad  +
 \boldsymbol{\sum_{A'^t}
  \pi_{\alpha}(A'^t|
 \underline{X}'^{t-1},
 \underline{A}'^{t-1})
 Q_{\textit{\textbf{Semi}}}^t}
 \biggl)
 \biggl] 
  \\ & 
 \overset{*_2}{=} 
 \mathbb{E}_{p'}
 \biggl[
\rho_{\textit{Semi}}^{t-1}
  \biggl(
 - \boldsymbol{\sum_{A'^t} 
    \pi'_{sim}(
    A'^{t}| \underline{X}'^{t-1}, \underline{A}'^{t-1},
    R
    )}
    \frac{
    \pi_\alpha(A'^{t}| \underline{X}'^{t-1}, \underline{A}'^{t-1}) 
    }
    {
    \pi'_{sim}(
    A'^{t}| \underline{X}'^{t-1}, \underline{A}'^{t-1},
    R
    )
    }
     \cdot  \\ & \quad \quad \quad \quad \quad \quad \cdot
    \frac{
    \mathbb{I}(R \geq R'^t)
    }
    {
    \hat{\pi}_\beta(
    R \geq R'^t| R \geq R'^{t-1},
    X_o
    )
    }
Q_{\textit{Semi}}^t +
 \sum_{A'^t}
  \pi_{\alpha}(A'^t|
 \underline{X}'^{t-1},
 \underline{A}'^{t-1})
 Q_{\textit{Semi}}^t
 \biggl)
  \biggl] 
    \\ & 
 \overset{*_3}{=} 
 \mathbb{E}_{p'}
 \biggl[
\rho_{\textit{Semi}}^{t-1}
  \biggl(
 - \boldsymbol{\sum_{A'^t} 
 \pi_{\alpha}(A'^t|
 \underline{X}'^{t-1},
 \underline{A}'^{t-1})}
Q_{\textit{Semi}}^t
  + 
 \sum_{A'^t}
  \pi_{\alpha}(A'^t|
 \underline{X}'^{t-1},
 \underline{A}'^{t-1})
 Q_{\textit{Semi}}^t
 \biggl)
  \biggl] = 0
\end{align*} 
with the following more detailed explanations: 
\begin{itemize}
    \item $*1)$: We replace $V_{\textit{Semi}}$ using $V_{\textit{Semi}}^{t-1} = \mathbb{E}_{\pi_{\alpha}}[Q_{\textit{Semi}}^{t}]$. We also decompose  $\rho_{\textit{Semi}}$:
\begin{align*} 
    \rho^t_{\textit{Semi}} & = 
    \prod_{\tau=1}^t
    \frac{
    \pi_\alpha(A'^{\tau}| \underline{X}'^{\tau-1}, \underline{A}'^{\tau-1}) 
    }
    {
    \pi'_{sim}(
    A'^{\tau}| \underline{X}'^{\tau-1}, \underline{A}'^{\tau-1},
    R
    )
    }
    \frac{
    \mathbb{I}(R \geq R'^t)
    }
    {
    \hat{\pi}_\beta(
    R \geq R'^t| 
    X_o
    ).
    }
    \\ 
    & = 
    \prod_{\tau=1}^{t-1}
    \frac{
    \pi_\alpha(A'^{\tau}| \underline{X}'^{\tau-1}, \underline{A}'^{\tau-1}) 
    }
    {
    \pi'_{sim}(
    A'^{\tau}| \underline{X}'^{\tau-1}, \underline{A}'^{\tau-1},
    R
    )
    }
    \frac{
    \mathbb{I}(R \geq R'^{t-1})
    }
    {
    \hat{\pi}_\beta(
    R \geq R'^{t-1}| 
    X_o
    ).
    }
    \\ & 
    \cdot 
    \left( 
        \frac{
    \pi_\alpha(A'^{t}| \underline{X}'^{t-1}, \underline{A}'^{t-1}) 
    }
    {
    \pi'_{sim}(
    A'^{t}| \underline{X}'^{t-1}, \underline{A}'^{t-1},
    R
    )
    }
           \frac{
    \mathbb{I}(R \geq R'^t)
    }
    {
    \hat{\pi}_\beta(
    R \geq R'^t| R \geq R'^{t-1},
    X_o
    )
    }
    \right) 
\end{align*}
\item 
$*2)$: The expected value $\pi'_{sim}(
A'^{t}| \underline{X}'^{t-1}, \underline{A}'^{t-1},
R
)
$ can be brought inside.
\item 
$*3)$: We leverage the independence of $Q^t_{Semi}$ and $R$ given that $R \geq R'^{t-1}$. 
\end{itemize}

\vspace{10pt}
\noindent 
\textit{Scenario 2:} 
We now assume the correct specification of $Q_{\textit{Semi}}$. Therefore, the following decomposition holds:
\begin{align*}                  
J_{\textit{DRL-Semi}} & =              
\mathbb{E}_{                
p'}[ 
V_{\textit{Semi}}^{0}]           
+                   
\mathbb{E}_{p'}
 [ 
\rho_{\textit{Semi}}^{T}
\left(C'-  
 Q_{\textit{Semi}}^{T}
 \right)]                   
 +                               
 \mathbb{E}_{p'}[\sum_{t=1}^{T-1} 
\rho_{\textit{Semi}}^{t}  
 \left(-Q_{\textit{Semi}}^{t}
 +  
 V_{\textit{Semi}}^{t}
 \right)]   
\\ 
&
=                           
\underbrace{                
\mathbb{E}_{                
p'}
\left[ 
V_{\textit{Semi}}^{0}
\right]
}_{=J}           
+                   
\mathbb{E}_{p'}
\left[ 
\rho_{\textit{Semi}}^{T}
\underbrace{
\left(
\sum_{C'}
C'
p'(C'| 
\underline{X}'^{T-1},
 \underline{A}'^{T},
X_o)
-  
 Q_{\textit{Semi}}^{T}
 \right)}_{=0}
 \right]
\\ &            
 +                                
 \mathbb{E}_{p'}
 \biggl[
 \sum_{t=1}^{T-1} 
\rho_{
\textit{Semi}}^{t}  
 \underbrace{
 \biggl(
 -
 Q_{\textit{Semi}}^{t}
+
 \sum_{X'^t}
 V_{\textit{Semi }}^{t}
 p(X'^t|
    \underline{X}'^{t-1}, 
    \underline{A}'^t, X_o
 ) 
 \biggl)}_{=0}
 \biggl].   
 \end{align*} 

\noindent 
$\mathbb{E}_{                
p'} \left[  V_{\textit{Semi}}^{0}
\right]$ is equal to the DM estimator and therefore consistent. The last term is equal to the first component of semi-offline RL Bellman's equation.  
\end{proof}


\section{Proof of Theorem \ref{theorem_efficient_if}}
\label{app_theorem_efficient_if}
This appendix contains the proof for Theorem \ref{theorem_efficient_if}. 

\vspace{10pt}
\begin{proof}
We apply the path-derivative approach described in our review of semi-parametric theory (Appendix \ref{app_semiparametric_theory}),  similar to our derivation of the influence function for the time-series setting \cite{von_kleist_evaluation_2023}. 

The regular parametric submodel for the semi-offline sampling distribution is 
\begin{align*}                      
\biggl\{     
g(C'|
\underline{X}'^{T-1},            
\underline{A}'^{T}, 
Y)
& 
p'_\theta(Y|                   
\underline{X}'^{T-1},            
\underline{A}'^{T}, 
X_o)       
\prod_{t=1}^{T} 
\pi'_{sim}(A'^t|  
\underline{X}'^{t-1}, 
\underline{A}'^{t-1},                           
R)                              
\prod_{t=0}^{T-1} 
p'_\theta( X'^t|                                        
\underline{X}'^{t-1},        
\underline{A}'^{t}, 
X_o) \\ &\cdot
\pi_{\beta,\theta}(R|  
X_o)  
p_\theta(X_o)
\biggl\}
\end{align*}

\noindent 
and it is equal to the true pdf at $\theta = 0$. 
One could further substitute the following equalities which show that the parametric submodel only contains $\theta$-dependent terms that are functions of $p$ instead of $p'$:
\begin{align*}
    p'_\theta(Y|
    \underline{X}'^{T-1},
    \underline{A}'^T, 
    X_o) 
    & = 
    p_\theta(Y|
    X_{r'^T}, 
    X_o) 
    \\
    p'_\theta(X'^t|
    \underline{X}'^{t-1}, 
    \underline{A}'^t, X_o) 
    & = 
    p_\theta(X_{a'^t}|
    X_{r'^{t-1}},
    X_o). 
\end{align*}

\noindent 
To apply the path-derivative approach, we have to first specify the scores: 
\begin{align*} 
 S \equiv 
S_{
C',
Y,
X',
A', 
R,
X_o} &  = 
\nabla_\theta  
\text{log} 
p'_\theta(Y|
\underline{X}'^{T-1},
\underline{A}'^T,
X_o)
+
\sum_{t=0}^{T-1}
\nabla_\theta  
\text{log} 
p'_\theta(X'^t|
\underline{X}'^{t-1},
\underline{A}'^{t}, 
X_o) +
\\ & 
+
\nabla_\theta  
\text{log} 
\pi_{\beta,\theta}(R|
X_o)
+
\nabla_\theta  
\text{log} 
p_{\theta}(
X_o)
\\ & = 
S_{Y|    
\underline{X}'^{T-1},    
\underline{A}'^{T}, 
X_o}
+ 
\sum_{t=0}^{T-1}
S_{               
X'^t|    
\underline{X}'^{t-1},   
\underline{A}'^{t}, 
X_o}
+
S_{               
R|      
X_o}
+ 
S_{                     
X_o}
\end{align*}

\noindent 
The derivative  $\nabla_\theta J$ is:
\begin{align*} 
 \nabla_\theta J  
  &\overset{*_1}{=} 
\nabla_\theta 
\mathbb{E}_{q'}
\left[C'\right] 
= 
\nabla_\theta \left(
\sum_{C',Y, X',A',R,X_o} 
C' 
q'_\theta(C',Y,X',A',R,X_o)  \right)
\\ & = 
\sum_{C',Y,X',A',R,X_o} 
C'
\nabla_\theta 
q'_\theta(C',Y,X',A',R,X_o)   
\\ &  \overset{*_2}{=} 
\sum_{C',Y,X',A',R,X_o} 
C'
q'_\theta(C',Y,X',A',R,X_o)  
\nabla_\theta 
\text{log}
q'_\theta(C',Y,X',A',R,X_o)   
\\ & \overset{*_3}{=}  
\mathbb{E}_{q'}
\biggl[C' 
\biggl(       
S_{Y|          
\underline{X}'^{T},    
\underline{A}'^{T},
X_o}  
+
\sum_{t=0}^{T-1} 
S_{               
X'^t| 
\underline{X}'^{t-1},  
\underline{A}'^{t},
X_o}
+ 
S_{               
X_o}  +
S_{R|R\geq R',               
X_o}  \biggl) \biggl]
\end{align*}

\noindent 
where we simplified the notation in $*1)$: $\mathbb{E}_{q'}[C'] \equiv 
 \mathbb{E}_{q'_\theta}[C']\equiv  \mathbb{E}[C'_{(\pi_\alpha)}]$. Similarly, we will denote 
 $\mathbb{E}[C'_{(\pi_\alpha)}|\underline{X}'^t,\underline{A}'^t,X_o ]$ by $\mathbb{E}_{q'}[C'|\underline{X}'^t,\underline{A}'^t,X_o ]$.
In $*3)$, we used that $\pi_\alpha$, $\pi'_{sim}$ and all deterministic functions $g$ are known and independent of $\theta$. 

We look at each term individually:

\vspace{10pt}
\noindent
\textit{Term 1:}
\begin{align*}
\mathbb{E}_{q'}
[C'
S_{Y|    
\underline{X}'^{T-1},    
\underline{A}'^{T}, 
X_o}                    
]  
& = 
\mathbb{E}_{p'}\left[
\frac{q'(C',Y,X',A',R,X_o) }{p'(C',Y,X',A',R,X_o)}    C'
S_{Y|    
\underline{X}'^{T-1},    
\underline{A}'^{T}
,X_o}   
\right]
 \\ & =   
\mathbb{E}_{p'}\left[
\rho_\textit{Semi}^T                    
C'
S_{Y|    
\underline{X}'^{T-1},    
\underline{A}'^{T}
,X_o}   
\right]
\\ & \overset{*_1}{=}  
\mathbb{E}_{p'}
\left[
\rho_\textit{Semi}^T                    
(C'-\mathbb{E}_{q'}[C'|    
\underline{X}'^{T-1},    
\underline{A}'^{T}
,X_o]) 
S_{Y|    
\underline{X}'^{T-1},    
\underline{A}'^{T}
,X_o}    
\right]
\\ & \overset{*_2}{=}                        
\mathbb{E}_{p'}\left[
\rho_\textit{Semi}^T                    
\left(C'                             
-                         
\mathbb{E}_{q'}
\left[
C'|                     
\underline{X}'^{T-1},    
\underline{A}'^{T}
,X_o                   
\right]
\right) 
S\right]
\end{align*}

\noindent 
$*1)$ holds because  
$\mathbb{E}_{q'}[C'|                     
\underline{X}'^{T-1},       
\underline{A}'^{T}, 
X_o                  
]$ 
$(=\mathbb{E}_{p'}[C'|                     
\underline{X}'^{T-1},       
\underline{A}'^{T}, 
X_o                 
])$ and $\rho_\textit{Semi}^T$ 
are independent of $p'(C',Y|                   
\underline{X}'^{T-1},       
\underline{A}'^{T}, 
R, X_o)$ and because the scores are mean zero, i.e. \newline 
$\mathbb{E}_{p'}\left[S_{Y|    
\underline{X}'^{T-1},    
\underline{A}'^{T}, 
X_o}\right] = 0$.

In $*2)$, we leverage the independence between all other scores ($S_{           
X'^t|      
\underline{X}'^{t-1},   
\underline{A}'^{t}, 
X_o}$, 
$S_{            
R|      
X_o}$ and $S_{X_o}$) and
$p'(C'|                   
\underline{X}'^{T-1},       
\underline{A}'^{T}, 
X_o)$. This allows us to bring that part of the expected value inside to find: 
\begin{align*}
\mathbb{E}_{p'}
\left[
\left(C'                             
-                         
\mathbb{E}_{q'}
\left[
C'|                     
\underline{X}'^{T-1},       
\underline{A}'^{T}, 
X_o                    
\right]
\right)| 
\underline{X}'^{T-1},       
\underline{A}'^{T}, 
X_o    
\right] = 0
\end{align*}

\vspace{10pt}
 \noindent 
\textit{Term 2:}
\begin{align*}
\mathbb{E}_{q'}
\biggl[
C' 
S_{               
X'^t|      
\underline{X}'^{t-1},   
\underline{A}'^{t}, 
X_o}                  
\biggr]  
& \overset{*_1}{=}     
\mathbb{E}_{p'}
\left[
\rho_\textit{Semi}^t                    
\mathbb{E}_{q'}
\left[C'|    
\underline{X}'^{t},   
\underline{A}'^{t}, 
X_o
\right]
S_{               
X'^t|      
\underline{X}'^{t-1},   
\underline{A}'^{t}, 
X_o}
\right]
\\ 
& 
\overset{*_2}{=} 
\mathbb{E}_{p'}\left[
\rho_\textit{Semi}^t                    
\left(
\mathbb{E}_{q'}
\left[
C'|    
\underline{X}'^{t},   
\underline{A}'^{t}, 
X_o
\right]
-
\mathbb{E}_{q'}
\left[C'|    
\underline{X}'^{t-1},   
\underline{A}'^{t}, 
X_o
\right]
\right) 
S_{               
X'^t|      
\underline{X}'^{t-1},   
\underline{A}'^{t}, 
X_o}
\right]
\\ & 
\overset{*_3}{=}                        
\mathbb{E}_{p'}\left[
\rho_\textit{Semi}^t                   
\left(
\mathbb{E}_{q'}[C'|    
\underline{X}'^{t},   
\underline{A}'^{t}, 
X_o]
-
\mathbb{E}_{q'}[C'|    
\underline{X}'^{t-1},   
\underline{A}'^{t}, 
X_o]
\right) 
S\right]
\end{align*}

\noindent 
In $*1)$, we used the following independences of the IPW weights: 
\begin{align*}
\mathbb{E}_{p'}
\biggl[
\rho_\textit{Semi}^T       
\mathbb{E}_{q'}
\biggl[C'|    
\underline{X}'^{t}, &  
\underline{A}'^{t}, 
X_o
\biggl]
\biggl] 
\\ & = 
\mathbb{E}_{p'}
\left[    
\prod_{\tau=1}^T
\frac{
\pi_\alpha(A'^{\tau}| \underline{X}'^{\tau-1}, \underline{A}'^{\tau-1}) 
}
{
\pi'_{sim}(
A'^{\tau}| 
\underline{X}'^{\tau-1}, 
\underline{A}'^{\tau-1},
R
)
}
\frac{
\mathbb{I}(R \geq R')
}
{
\pi_\beta(
R \geq R'| 
X_o
)
}
\mathbb{E}_{q'}
\left[C'|    
\underline{X}'^{t},   
\underline{A}'^{t}, 
X_o
\right]
\right]
= 
\\ & \overset{*_{1.1}}{=}  
\mathbb{E}_{p'}
\left[    
\prod_{\tau=1}^t
\frac{
\pi_\alpha(A'^{\tau}| \underline{X}'^{\tau-1}, \underline{A}'^{\tau-1}) 
}
{
\pi'_{sim}(
A'^{\tau}| 
\underline{X}'^{\tau-1}, 
\underline{A}'^{\tau-1},
R
)
}
\frac{
\mathbb{I}(R \geq R'^t)
}
{
\pi_\beta(
R \geq R'^t| 
X_o
)
}
\mathbb{E}_{q'}
\left[C'|    
\underline{X}'^{t},   
\underline{A}'^{t}, 
X_o
\right]
\right]
\end{align*}

\noindent 
In $*1.1)$ we use the independence of the weights, leveraging the decomposition: 
\begin{align*}
\frac{
\mathbb{I}(R \geq R')
}
{
\pi_\beta(
R \geq R'| 
X_o
)
}
= 
\frac{
\mathbb{I}(R \geq R'^t)
}
{
\pi_\beta(
R \geq R'^t| 
X_o
)
}
\underbrace{\frac{
\mathbb{I}(R \geq R'^T)
}
{
\pi_\beta(
R \geq R'^T| R \geq R'^t,
X_o
)
}}_{\text{indep. of }
q'(C'|\underline{X}'^{t}, 
\underline{A}'^{t}, 
X_o)}.
\end{align*}

\noindent 
In $*2)$, we used that $\mathbb{E}_{q'}[C'|          \underline{X}'^{t-1},   
\underline{A}'^{t}, 
X_o                 
]$ and $\rho_\textit{Semi}^t$ 
are independent of $p(X'^t|                   
\underline{X}'^{t-1},           
\underline{A}'^{t}, 
X_o)$
and mean zero property of the scores, i.e. 
$\mathbb{E}_{p'}\left[S_{X'^t|    
\underline{X}'^{t-1},    
\underline{A}'^{t}, 
X_o}
\right] = 0$.
In $*3)$, we add the scores separately to the equation. 
\begin{itemize}
    \item Firstly, $\mathbb{E}_{p'}[.]$ can be divided into:
    \begin{align*}
    \mathbb{E}_{p'}[.] = 
\mathbb{E}_{
\underline{X}'^{t},
\underline{A}'^{t}, 
X_o}
\left[
\mathbb{E}_{
C',
Y,
\overline{X}'^{t+1},
\overline{A}'^{t+1}|
\underline{X}'^{t},
\underline{A}'^{t}, 
X_o}
\left[.|
\underline{X}'^{t},
\underline{A}'^{t}, 
X_o
\right]\right]
\end{align*}
such that bringing the inner expectation inside and using the fact that scores are mean zero, shows that $S_{Y|    
\underline{X}'^{T-1},    
\underline{A}'^{T}, 
X_o}$ and 
$S_{  
X'^\tau|    
\underline{X}'^{\tau-1},   
\underline{A}'^{\tau}, 
X_o}$ (with $\tau > t$)
can be included. 
\item We use another division of the expected value: 
\begin{align*}
\mathbb{E}_{
\underline{X}'^{t},
\underline{A}'^{t}, 
X_o
}[.]
= 
\mathbb{E}_{
\underline{X}'^{t-1},
\underline{A}'^{t}, 
X_o}
\left[
\mathbb{E}_{
X'^{t}|
\underline{X}'^{t-1},
\underline{A}'^{t}, 
X_o}
\left[.|
\underline{X}'^{t-1},
\underline{A}'^{t}, 
X_o
\right]\right]
\end{align*}
which lets the difference in the brackets $(.)$ be equal to zero when the inner expectation is brought inside. This bringing inside of the inner expectation is allowed for $S_{ X'^\tau|    
    \underline{X}'^{\tau-1},   
    \underline{A}'^{\tau}, X_o}$
(for $\tau < t$),  
$S_{               
R|    
X_o}$ and 
$S_{                   
X_o}$. 
\end{itemize}

\vspace{10pt}
\noindent 
\textit{Term 3:}
\begin{align*}
\mathbb{E}_{q'}
[C' &
S_{               
X_o}                  
]  
=                        
\mathbb{E}_{p'}\left[                    
\left(
\mathbb{E}_{q'}[C'|    
X_o]
-
\mathbb{E}_{q'}[C']
\right) 
S\right]
\end{align*}

\noindent 
where $\mathbb{E}_{q'}[C'] = J$

\vspace{10pt}
\noindent 
\textit{Term 4:}
\begin{align*}
\mathbb{E}_{q'}&
\biggl[
C'  
S_{               
R|R\geq R',      
X_o}   
\biggr] 
= 
0
\end{align*}

\noindent 
since $C'$ is conditionally independent of $R$, given $R\geq R'$ (given it fulfills the positivity assumption) and using the fact that the scores are mean zero.

\vspace{10pt}
\noindent 
\textit{Influence function: }

\noindent 
By substituting our derived expressions for all individual terms, we find 
\begin{align*}
\nabla_\theta J   
  =  & 
\mathbb{E}_{p'}
\biggl[ 
\Bigl(
\rho_\textit{Semi}^T                    
\left(C'                             
-                         
\mathbb{E}_{q'}[C'|                     
\underline{X}'^{T-1},       
\underline{A}'^{T}, 
X_o                     
]
\right) 
\\   & \quad +
\sum_{t=0}^{T-1}                   
\rho_\textit{Semi}^t                    
\left(
\mathbb{E}_{q'}[C'|    
\underline{X}'^{t},       
\underline{A}'^{t}, 
X_o]
-
\mathbb{E}_{q'}[C'|    
\underline{X}'^{t-1},       
\underline{A}'^{t}, 
X_o]
\right) 
\\ & \quad +                                  \mathbb{E}_{q'}[C'|    
X_o]
-
J
\Bigr) 
S 
\biggr]
\\  = & 
\mathbb{E}_{p'}
\biggl[ 
\Bigl(
\rho_\textit{Semi}^T                    
\left(C'                             
-                         
Q_\textit{Semi}^T                     
\right) +
\sum_{t=1}^{T-1}                  
\rho_\textit{Semi}^t                    
\left(
V_\textit{Semi}^t
-
Q_\textit{Semi}^t
\right) +                                   V_\textit{Semi}^0
-
J
\Bigr) 
S 
\biggr]
\\ 
= 
&
\mathbb{E}_{p'}
\biggl[ 
\Bigl(
\rho_\textit{Semi}^T                    
C'                                          +
\sum_{t=1}^T                   
\left(-
\rho_\textit{Semi}^t
Q_\textit{Semi}^t
+
\rho_\textit{Semi}^{t-1}                    
V_\textit{Semi}^{t-1}
\right) 
-
J
\Bigr) 
S 
\biggr].
\end{align*}

\noindent 
The influence function $\Psi$ can now be read off. 
\end{proof}

\section{Estimation of Other Target Parameters from the Semi-offline RL View}
\label{app_other_variants}

This appendix contains the extension of the main theorems and estimators to also include acquisition costs $C_a^t$. 
The newly defined target parameter becomes 
\begin{align*}
    J_\textit{total} = \mathbb{E}
\left[\sum_{t=1}^{T-1} 
C_{a^t,(\pi_\alpha)} + 
C_{mc,(\pi_\alpha)} \right] \equiv \mathbb{E}\left[\sum_{t=1}^T C^t_{(\pi_\alpha)}\right]
\end{align*}
where we let $C^t$ denote $C_a^t$ if $t < T$ and $C_{mc}$ if $t = T$.   

This extended setting only differs slightly from the setting described in the main part and the derived fundamental concepts can still be applied. We thus only state corollaries of the main identification and estimation theorems for this setting and exclude additional proofs.

\subsection{Identification}

Firstly, we extend Theorem \ref{theorem_identification_semi_offline_RL} to the total cost setting.

\begin{corollary}
\label{corollary_perstep_identification_semi_offline_RL} 
(Identification of $J_\textit{total}$ for the semi-offline RL view).
The AFAPE problem of estimating $J_\textit{total}$ under the semi-offline RL view  is under the no direct effect (NDE) assumption, the MAR assumption from Eq. \ref{eq_MAR}, the consistency assumption, the no interference assumption, the static feature assumption assumption and the positivity assumption from Eq. \ref{eq_positivity_semi_offline_RL} is identified by 
\begin{align}
\label{eq_identificiation_semi_offline_RL_1_perstep}
    J  
     = 
    \mathbb{E}_{p'}
    \left[ 
    \sum_{t=1}^T C'_{(\pi_\alpha)}
    \right]
     = 
    \sum_{C',Y, X',A',R,X_o} 
    \left(\sum_{t=1}^T 
    \mathbb{E}[C'^t|    
    \underline{X}'^{t-1},            
    \underline{A}'^{t},         
    X_o ]
    \right) 
    q'(
    \underline{X}'^{T-1},            
    \underline{A}'^{T},         
    R, 
    X_o) 
\end{align}
\noindent
where $q'$ is given by Eq. \ref{eq_identificiation_semi_offline_RL_2} and 
\begin{align*}
   \mathbb{E}[ C'^t|    
    \underline{X}'^{t-1},            
    \underline{A}'^{t},         
    X_o ] 
    = 
   \begin{cases}
\sum_{C'_{mc}, Y,Y'^{*}}
    C'_{mc} g(C'_{mc}|Y,Y'^{*}) p( Y| \underline{X}^{T-1},\underline{A}^T, X_o) g(Y'^{*}| \underline{X}'^{T-1},\underline{A}'^T)
    & \text{ if } t = T\\
\sum_{C'_{a}} C'_{a} g(C'_{a}|A'^t)  & \text{ if } t < T.
\end{cases}
\end{align*}
\end{corollary}

\noindent 
Next, we continue with a corollary that extends Theorem \ref{theorem_Bellman_equation} for the per-step costs setting.
\begin{corollary}
\label{corollary_perstep_Bellman_equation} 
(Bellman equation for semi-offline RL ($J_\textit{total}$)).
The semi-offline RL view admits under the no direct effect (NDE) assumption, the MAR assumption from Eq. \ref{eq_MAR}, the consistency assumption, the no interference assumption, the static feature assumption, and the positivity assumption from Eq. 
\ref{eq_positivity_semi_offline_RL} the following semi-offline RL version of the Bellman equation for the target $J_\textit{total}$:

\begin{align}
\label{eq_bellman_1_perstep}
Q_{\textit{Semi}}(\underline{X}'^{t-1}, &
\underline{A}'^{t},            
X_o ) 
=
\mathbb{E}[C'^t|    
    \underline{X}'^{t-1},            
    \underline{A}'^{t},         
    X_o ] 
+
\sum_{\mathclap{X'^t}}                          
V_{\textit{Semi}}
(              \underline{X}'^{t}, \underline{A}'^{t}, 
X_o) 
p(X'^t|                   
\underline{X}'^{t-1},          
A'^{t}, 
X_o)
\\
 V_{\textit{Semi}}
(\underline{X}'^{t}, 
\underline{A}'^{t},&
X_o)  
=
\sum_{A'^{t+1}}  
Q_{\textit{Semi}}(\underline{X}'^{t}, \underline{A}'^{t+1},      X_o  )
\pi_{\alpha}(A'^{t+1}|
\underline{X}'^{t},
\underline{A}'^{t})
\label{eq_bellman_2_perstep}
\end{align}

\noindent 
with semi-offline RL versions of the state-action value function $Q_{\textit{Semi}}$ and state value function $V_{\textit{Semi}}$:
\begin{align*}
 Q_{\textit{Semi}}^t 
 & \equiv 
 Q_{\textit{Semi}}(\underline{X}'^{t-1},
\underline{A}'^{t},
X_o) 
\equiv  
\mathbb{E}_{p'}\left[\sum_{\tau={t}}^T 
C'^\tau_{(\overline{\pi}^{t+1}_\alpha)}|
\underline{X}'^{t-1},
\underline{A}'^{t},
X_o
\right]  
\\
V_{\textit{Semi}}^t
& \equiv 
V_{\textit{Semi}}(
\underline{X}'^t,
\underline{A}'^{t},
X_o) 
\equiv 
\mathbb{E}_{p'}\left[\sum_{\tau={t+1}}^T C'^\tau_{(\overline{\pi}^{t+1}_\alpha)}|
\underline{X}'^t,
\underline{A}'^{t},
X_o
\right].
\end{align*}
\end{corollary}

\subsection{Estimation}

The semi-offline RL estimators for $J_{total}$ are given in the following. 

\vspace{10pt}
\noindent 
\textit{1) Inverse probability weighting (IPW):} 

\noindent 
The semi-offline RL IPW estimator for $J_{total}$ has the following form: 
\begin{align*} 
    J_{\textit{IPW-Semi}}
    & = 
    \hat{\mathbb{E}}_{n'}
    \left[ \sum_{t=1}^T \rho_{
    \textit{Semi}}^t
    \text{ } 
    C'^t
    \right], 
\end{align*} 
without any adjustment for $\rho_\textit{Semi}^t$.

\vspace{10pt}
\noindent
\textit{2) Direct method (DM):} 

\noindent
The semi-offline RL DM estimator for $J_{total}$ has the following form: 
\begin{align*} 
    J_{\textit{DM-Semi}} = 
    \hat{\mathbb{E}}_{n'}[V_{\textit{Semi}}^0]
\end{align*}

\noindent 
with the adapted $J_{total}$ version of $V_{\textit{Semi}}$ from Corollary \ref{corollary_perstep_Bellman_equation}.

\vspace{10pt}
\noindent 
\textit{3) Double reinforcement learning (DRL):} 

\noindent
The semi-offline DRL estimator for $J_{total}$ has the following form: 
\begin{align*}
J_{\textit{DRL-Semi}} = 
\hat{\mathbb{E}}_{n'}
\left[
\sum_{t=1}^{T} 
\left( 
\rho_\textit{Semi}^t 
C'^t  
- 
\rho_\textit{Semi}^t 
Q_\textit{Semi}^t
+
\rho_{\textit{Semi}}^{t-1}  
V_\textit{Semi}^{t-1}
\right) 
\right].
\end{align*}

\noindent 
where $V_{\textit{Semi}}$ and $Q_{\textit{Semi}}$ come from Corollary \ref{corollary_perstep_Bellman_equation}.


%
%
%
%

\section{Experiment Details}
\label{Appendix_experiments}

This appendix contains the experiment details. 
A detailed list of the parameters and configurations for each experiment can be found in Tables \ref{tab:details synthetic}, \ref{tab:details heloc}, and  \ref{tab:details income}.

\subsection{Data and Costs}
In our experiments, we characterized a "superfeature" as a feature encompassing multiple subfeatures, typically obtained or omitted together and incurring a unified cost. Additionally, we presumed the existence of a subset of features that incur no cost (termed free features) and assigned predetermined acquisition costs ($c_{acq}$) to the remaining superfeatures. We selected misclassification costs such that effective AFA policies need to find a balance between the cost of acquiring features and the predictive value of those features.
To assess the convergence of various estimators, we consider the average cost incurred by running the AFA policy on the ground truth dataset without missingness as the ground truth for $J$. This involves estimating $J$ through Eq. \ref{eq:AFAPE_objective_miss} by employing a Monte Carlo estimate for $\hat{\mathbb{E}}\left[ C_{(\pi_{\alpha})}|X_{(1)},Y \right]$, using the ground truth data without missingness (i.e., samples from $p(X_{(1)},Y)$).

\subsection{Training}

We employed an impute-then-regress classifier \cite{le_morvan_whats_2021} with unconditional mean imputation and a random forest classifier for the classification task. This classifier was trained on both the available data and additional randomly subsampled data, where the probability $p(R_i = 1)$ was set to 0.5. Random acquisition policies were tested, with each costly feature being acquired with a 10\%, 50\% or 90\% probability. Additionally, we assessed a vanilla deep Q-network (DQN) RL agent \cite{mnih_human-level_2015}, which was trained on $\mathcal{D}'$ with the objective of maximizing 
$\mathbb{E}[C']$. 

The datasets were partitioned into a training set, utilized for training both the DQN agent and the classifier, a nuisance function training set, and a test set for evaluating the estimators. The necessity of splitting the dataset into a nuisance function training set and a test set stems from the complexity of the employed nuisance model functions classes \cite{kennedy_semiparametric_2022}. However, this potential loss of efficiency can be circumvented by employing a cross-fitting approach \cite{kennedy_semiparametric_2022}.

\subsection{Synthetic Dataset}

We evaluated different estimators on a synthetic dataset where 
the features were simulated using Scikit-learn's \verb|make_classification| function \cite{pedregosa_scikit-learn_2011}.
The labels were distributed according to 
\begin{align*}
    p(Y=1) = 
    \begin{cases}
    1
    , & \text{if }  \sum_i X_{(1),i}  > 0  \\
   0.3 , & \text{otherwise}.
\end{cases}
\end{align*}

\noindent 
This distribution for  $Y$ was chosen to simulate a setting that makes the classification for some data points more difficult than for others. 

We induced a MAR missingness scenario where missingness depended only on the always observed features. In a second experiment, we induced an MNAR missingness pattern that 
resembles the scenario given in Appendix \ref{app_missing_data} as the missingness of superfeature 3 depends on superfeature 2 which might be missing itself. See Table \ref{tab:details synthetic} for more details on the synthetic data experiments. 

\subsection{Heloc Dataset}
We also tested on the Heloc (Home Equity Line of Credit) dataset which is part of the FICO explainable machine learning (xML) challenge \cite{fico_fico_2018}. The dataset contains credit applications made by homeowners. The ML task was to predict based on information at the application, whether an applicant is able to repay their loan within two years. We assume sensitive information should only be accessed at a cost to formulate the AFA problem. See Table~\ref{tab:details heloc} for full details on the Heloc data experiment.

\subsection{Income Dataset}
We evaluated on the Income dataset from the UCI data repository \cite{newman_uci_1998}. The ML task was to predict whether a person has an income over 50'000\$. We considered private information (such as education, work experience, etc.) as costly. See Table~\ref{tab:details income} for full details.

\begingroup
\setlength{\tabcolsep}{6pt} 
\renewcommand{\arraystretch}{1.2} 
\begin{center}
\begin{table}
\begin{tabular}{| p{0.25\textwidth} | p{0.75\textwidth} |}
    \hline
    \multicolumn{2}{| c |}{\textbf{Data and environment}} \\
    \hline\hline
    Sample size $n_{D}$ & $150'000$ divided into $20\%$ training set (for DQN agent and classifier), $40\%$ nuisance function training set, and $40\%$ test set. \\
    \hline
    Superfeatures & super$X_0$: $[X_0]$, super$X_1$: $[X_1]$, super$X_2$: $[X_2, X_3]$\\
    \hline
    Costs & $c_{acq}=[0 , 1,  1]$ and $c_{mc} = 14$ \\
    \hline\hline
    \multicolumn{2}{| c |}{\textbf{Missingness mechanisms}} \\
    \hline\hline
    MAR & \makecell*[l]{
        $p(R_0 =  1) = 1.0,$ \\
        $p(R_1 = 1) = \sigma(-0.3 + 0.5 X_{(1),0}),$ \\
        $p(R_2 = 1) = \sigma(-0.1 + 0.6 X_{(1),0})$ \\
        Complete cases ratio: $p(A = \vec{1}) \approx 24\%$} \\
        \hline
    MNAR & \makecell*[l]{
        $p(R_0 = 1) = 1.0,$ \\
        $p(R_1 = 1) = 0.7,$ \\
        $p(R_2 = 1) = \sigma(-1.5 + X_{(1),1})$ \\
        Complete cases ratio: $p(A= \vec{1}) \approx 22\%$} \\
    \hline\hline
    \multicolumn{2}{| c |}{\textbf{Models}} \\
    \hline\hline
    Classifier &           RandomForest (max depth: 10,
          number of estimators: 50)
          \\
    \hline
    Agents & \makecell*[l]{
   DQN (learning rate: 0.0001, number of layers: 2, \\
   \quad\quad\quad hidden layer neurons per layer: 16, \\
   \quad\quad\quad
   hidden layer activation function: ReLU) } \\
    \hline
    Nuisance functions & \makecell*[l]{
    $Q_{Semi}$ (learning rate: 0.001, number of layers: 2, 
    \\ \quad\quad\quad
    hidden layer neurons per layer: 16, \\ 
    \quad\quad\quad hidden layer activation function: ReLU  ) } \\
    \hline
\end{tabular}
    \caption{Synthetic data experiment details}
    \label{tab:details synthetic}
\end{table}
\end{center}

\begingroup
\setlength{\tabcolsep}{6pt} 
\renewcommand{\arraystretch}{1.1} 
\begin{center}
\begin{table}
\begin{tabular}{| p{0.25\textwidth} | p{0.75\textwidth} |}
    \hline
    \multicolumn{2}{| c |}{\textbf{Data and environment}} \\
    \hline\hline
    Sample size $n_{D}$ & $35'562$ divided into $40\%$ training set (for DQN agent and classifier), $30\%$ nuisance function training set, and $30\%$ test set. 
    \\
    \hline
    Superfeatures & \makecell*[l]{
        super$X_0$: [`ExternalRiskEstimate'] \vspace*{5pt} \\ 
        super$X_1$: [`MSinceOldestTradeOpen', `MSinceMostRecentTradeOpen', \\
        \hspace*{44pt}`AverageMInFile', `NumSatisfactoryTrades'] \vspace*{5pt} \\
        super$X_2$: [`NumTrades60Ever2DerogPubRec', \\
        \hspace*{44pt}`NumTrades90Ever2DerogPubRec'] \vspace*{5pt} \\
        super$X_3$: [`PercentTradesNeverDelq', `MSinceMostRecentDelq', \\
        \hspace*{44pt}`MaxDelq2PublicRecLast12M'] \vspace*{5pt} \\
        super$X_4$: [`MaxDelqEver', `NumTotalTrades'] \vspace*{5pt} \\
        super$X_5$: [`NumTradesOpeninLast12M', `PercentInstallTrades', \\
        \hspace*{44pt}`MSinceMostRecentInqexcl7days'] \vspace*{5pt} \\
        super$X_6$: [`NumInqLast6M, NumInqLast6Mexcl7days'] \vspace*{5pt} \\
        super$X_7$: [`NetFractionRevolvingBurden', `PercentTradesWBalance'] \vspace*{5pt} \\
        super$X_8$: [`NetFractionInstallBurden', `NumRevolvingTradesWBalance', \\
        \hspace*{44pt}`NumInstallTradesWBalance', \\
        \hspace*{44pt}`NumBank2NatlTradesWHighUtilization']} \\
    \hline
    Label & `RiskPerformance' (class 0: 48\%,  class 1: 52\%)\\
    \hline
    Costs & $c_{acq}=[1, 1, 1 ,  1, 1,  1,  1, 1, 1]$ and $c_{mc}= 20$ 
    \\
    \hline\hline
    \multicolumn{2}{| c |}{\textbf{Missingness mechanisms}} \\
    \hline\hline
    MAR & \makecell*[l]{
        $p(R_i = 1) = 1.0, \quad i \in \{0, 2, 4, 6, 8\}$ \\
        $p(R_j = 1) = \sigma(0.0 + 5.0 \text{ ExternalRiskEstimate})$ \\
        Complete case ratio: $p(\bar{R} = \vec{1}) \approx 35\%$} \\
    \hline\hline
    \multicolumn{2}{| c |}{\textbf{Models}} \\
    \hline\hline
    Classifier & 
    RandomForest (max depth: 5,
          number of estimators: 100) \\
    \hline
        Agents & \makecell*[l]{
   DQN (learning rate: 0.0001, number of layers: 2, \\
   \quad\quad\quad hidden layer neurons per layer: 16, \\
   \quad\quad\quad
   hidden layer activation function: ReLU) } \\
    \hline
    Nuisance functions & \makecell*[l]{
    $Q_{Semi}$ (learning rate: 0.001, number of layers: 2, 
    \\ \quad\quad\quad
    hidden layer neurons per layer: 16, \\ 
    \quad\quad\quad hidden layer activation function: ReLU  ) 
    }\\
    \hline
\end{tabular}
    \caption{Heloc data experiment details}
    \label{tab:details heloc}
\end{table}
\end{center}

\newpage

\begingroup
\setlength{\tabcolsep}{6pt} 
\renewcommand{\arraystretch}{1.1} 
\begin{center}
\begin{table}[h!]
\begin{tabular}{| p{0.25\textwidth} | p{0.75\textwidth} |}
    \hline
    \multicolumn{2}{| c |}{\textbf{Data and environment}} \\
    \hline\hline
    Sample size $n_{D}$ & 32'561 divided into $40\%$ training set (for DQN agent and classifier), $30\%$ nuisance function training set, and $30\%$ test set.
    \\
    \hline
    Superfeatures & \makecell*[l]{
        workclass: [`Federal-gov', `Local-gov', `Never-worked', `Private', \\
        \hspace*{49pt}`Self-emp-inc', `Self-emp-not-inc', `State-gov', `Without-pay'] \vspace*{4pt} \\ 
        education: [`1st-4th', `5th-6th', `7th-8th', `9th', `10th', `11th', `12th', \\
        \hspace*{49pt}`Assoc-acdm', `Assoc-voc', `Bachelors', `Doctorate', \\
        \hspace*{49pt}`HS-grad', `Masters', `Preschool', `Prof-school', \\
        \hspace*{49pt}`Some-college', `-num'] \vspace*{4pt} \\
        marital-status: [`Married-AF-spouse', `Married-civ-spouse', \\
        \hspace*{64pt}`Married-spouse-absent', Never-married', `Separated', \\
        \hspace*{64pt}`Widowed', `relationship Not-in-family', \\
        \hspace*{64pt}`relationship Other-relative', `relationship Own-child', \\
        \hspace*{64pt}`relationship Unmarried', `relationship Wife'] \vspace*{4pt} \\
        occupation: [`Adm-clerical', `Armed-Forces', `Craft-repair', \\
        \hspace*{53pt}`Exec-managerial', `Farming-fishing', `Handlers-cleaners', \\
        \hspace*{53pt}`Machine-op-inspct', `Other-service', `Priv-house-serv', \\
        \hspace*{53pt}`Prof-specialty', `Protective-serv', `Sales', `Tech-support', \\
        \hspace*{53pt}`Transport-moving'] \vspace*{4pt} \\
        race: [`Asian-Pac-Islander', `Black', `Other'], sex: [`sex Male'] \\
        age: [`age'], hours-per-week: ['hours-per-week'] \\
        capital-gain: [`capital-gain'], capital-loss: [`capital-loss']} \\
    \hline
    Label & `income' (class 0: 76.0\%, class 1: 24.0\%)\\
    \hline
    Costs & $c_{acq}=[1,1,1,1,1,1,1,1,1,1]$ and $c_{mc}= 50$
    \\
    \hline\hline
    \multicolumn{2}{| c |}{\textbf{Missingness mechanisms}} \\
    \hline\hline
MAR & \makecell*[l]{
        $p(R_i = 1) = 1.0, \quad i \in \{5, 6\}$ \\
        $p(R_j = 1) = \sigma(1.0 - \text{ Male} + 4.0 \text{ age}), \quad j \in \{0,1, 2, 3, 4\}$ \\
        $p(R_k = 1) = \sigma(\text{ Male} + 3.0 \text{ age}), \quad k \in \{7, 8, 9\}$ \\
        Complete case ratio: $p(\bar{R} = \vec{1}) \approx 57\%$} 
        \\
    \hline\hline
    \multicolumn{2}{| c |}{\textbf{Models}} \\
    \hline\hline
    Classifier & RandomForest (max depth: 10,
          number of estimators: 50) \\
    \hline
        Agents & \makecell*[l]{
   DQN (learning rate: 0.0001, number of layers: 2, \\
   \quad\quad\quad hidden layer neurons per layer: 16, \\
   \quad\quad\quad
   hidden layer activation function: ReLU) } \\
    \hline
   Nuisance functions & \makecell*[l]{
  $Q_{Semi}$ (learning rate: 0.001, number of layers: 2, 
    \\ \quad\quad\quad
 hidden layer neurons per layer: 16, \\ 
   \quad\quad\quad hidden layer activation function: ReLU  ) 
    }\\
    \hline
\end{tabular}
    \caption{Income data experiment details}
    \label{tab:details income}
\end{table}
\end{center}

\clearpage

\bibliographystyle{unsrt}  
\bibliography{references}  

\begin{thebibliography}{10}

\bibitem{von_kleist_evaluation_2023}
Henrik von Kleist, Alireza Zamanian, Ilya Shpitser, and Narges Ahmidi.
\newblock Evaluation of {Active} {Feature} {Acquisition} {Methods} for {Time}-varying {Feature} {Settings}, December 2023.
\newblock arXiv:2312.01530 [cs, stat].

\bibitem{levine_offline_2020}
Sergey Levine, Aviral Kumar, George Tucker, and Justin Fu.
\newblock Offline {Reinforcement} {Learning}: {Tutorial}, {Review}, and {Perspectives} on {Open} {Problems}.
\newblock {\em arXiv:2005.01643 [cs, stat]}, November 2020.
\newblock arXiv: 2005.01643.

\bibitem{kallus_double_2020}
Nathan Kallus and Masatoshi Uehara.
\newblock Double {Reinforcement} {Learning} for {Efficient} {Off}-{Policy} {Evaluation} in {Markov} {Decision} {Processes}.
\newblock {\em Journal of Machine Learning Research}, 21(167):1--63, 2020.

\bibitem{lavalle_cash_1968}
Irving~H. LaValle.
\newblock On cash equivalents and information evaluation in decisions under uncertainty {Part} {II}: {Incremental} information decisions.
\newblock {\em Journal of the American Statistical Association}, 63(321):277--284, 1968.
\newblock Publisher: Taylor \& Francis.

\bibitem{lavalle_cash_1968-1}
Irving~H. LaValle.
\newblock On cash equivalents and information evaluation in decisions under uncertainty part {I}: {Basic} theory.
\newblock {\em Journal of the American Statistical Association}, 63(321):252--276, 1968.
\newblock Publisher: Taylor \& Francis.

\bibitem{gould_risk_1974}
John~P Gould.
\newblock Risk, stochastic preference, and the value of information.
\newblock {\em Journal of Economic Theory}, 8(1):64--84, May 1974.

\bibitem{hilton_determinants_1979}
Ronald~W. Hilton.
\newblock The {Determinants} of {Cost} {Information} {Value}: {An} {Illustrative} {Analysis}.
\newblock {\em Journal of Accounting Research}, 17(2):411--435, 1979.
\newblock Publisher: [Accounting Research Center, Booth School of Business, University of Chicago, Wiley].

\bibitem{hess_risk_1982}
James Hess.
\newblock Risk and the {Gain} from {Information}.
\newblock {\em Journal of Economic Theory}, 27(1):231--238, 1982.

\bibitem{keisler_value_2014}
Jeffrey~M. Keisler, Zachary~A. Collier, Eric Chu, Nina Sinatra, and Igor Linkov.
\newblock Value of information analysis: the state of application.
\newblock {\em Environment Systems and Decisions}, 34(1):3--23, March 2014.

\bibitem{mushlin_is_1992}
Alvin~I. Mushlin and Lou Fintor.
\newblock Is screening for breast cancer cost-effective?
\newblock {\em Cancer}, 69(S7):1957--1962, 1992.
\newblock \_eprint: https://onlinelibrary.wiley.com/doi/pdf/10.1002/1097-0142\%2819920401\%2969\%3A7\%2B\%3C1957\%3A\%3AAID-CNCR2820691716\%3E3.0.CO\%3B2-T.

\bibitem{krahn_screening_1994}
Murray~D. Krahn, John~E. Mahoney, Mark~H. Eckman, John Trachtenberg, Stephen~G. Pauker, and Allan~S. Detsky.
\newblock Screening for {Prostate} {Cancer}: {A} {Decision} {Analytic} {View}.
\newblock {\em JAMA}, 272(10):773--780, September 1994.

\bibitem{botteman_health_2003}
Marc~F. Botteman, Chris~L. Pashos, Alberto Redaelli, Benjamin Laskin, and Robert Hauser.
\newblock The health economics of bladder cancer.
\newblock {\em PharmacoEconomics}, 21(18):1315--1330, December 2003.

\bibitem{force_screening_2009}
US~Preventive Services~Task Force*.
\newblock Screening for breast cancer: {US} {Preventive} {Services} {Task} {Force} recommendation statement.
\newblock {\em Annals of internal medicine}, 151(10):716--726, 2009.
\newblock Publisher: American College of Physicians.

\bibitem{liu_efficient_2021}
Lin Liu, Zach Shahn, James~M. Robins, and Andrea Rotnitzky.
\newblock Efficient {Estimation} of {Optimal} {Regimes} {Under} a {No} {Direct} {Effect} {Assumption}.
\newblock {\em Journal of the American Statistical Association}, 116(533):224--239, January 2021.

\bibitem{robins_estimation_2008}
James Robins, Liliana Orellana, and Andrea Rotnitzky.
\newblock Estimation and extrapolation of optimal treatment and testing strategies.
\newblock {\em Statistics in Medicine}, 27(23):4678--4721, 2008.
\newblock \_eprint: https://onlinelibrary.wiley.com/doi/pdf/10.1002/sim.3301.

\bibitem{neugebauer_identification_2017}
Romain Neugebauer, Julie~A. Schmittdiel, Alyce~S. Adams, Richard~W. Grant, and Mark~J. van~der Laan.
\newblock Identification of the {Joint} {Effect} of a {Dynamic} {Treatment} {Intervention} and a {Stochastic} {Monitoring} {Intervention} {Under} the {No} {Direct} {Effect} {Assumption}.
\newblock {\em Journal of Causal Inference}, 5(1):20160015, September 2017.

\bibitem{kreif_exploiting_2021}
Noémi Kreif, Oleg Sofrygin, Julie~A. Schmittdiel, Alyce~S. Adams, Richard~W. Grant, Zheng Zhu, Mark~J. van~der Laan, and Romain Neugebauer.
\newblock Exploiting nonsystematic covariate monitoring to broaden the scope of evidence about the causal effects of adaptive treatment strategies.
\newblock {\em Biometrics}, 77(1):329--342, 2021.
\newblock \_eprint: https://onlinelibrary.wiley.com/doi/pdf/10.1111/biom.13271.

\bibitem{yoon_asac_2019}
Jinsung Yoon, James Jordon, and Mihaela Schaar.
\newblock {ASAC}: {Active} {Sensing} using {Actor}-{Critic} models.
\newblock In {\em Machine {Learning} for {Healthcare} {Conference}}, pages 451--473. PMLR, October 2019.
\newblock ISSN: 2640-3498.

\bibitem{yoon_deep_2018}
Jinsung Yoon, William~R. Zame, and Mihaela Van Der~Schaar.
\newblock Deep sensing: {Active} sensing using multi-directional recurrent neural networks.
\newblock In {\em International {Conference} on {Learning} {Representations}}, 2018.

\bibitem{tang_adversarial_2020}
Fengyi Tang, Lifan Zeng, Fei Wang, and Jiayu Zhou.
\newblock Adversarial {Precision} {Sensing} with {Healthcare} {Applications}.
\newblock In {\em 2020 {IEEE} {International} {Conference} on {Data} {Mining} ({ICDM})}, pages 521--530, November 2020.

\bibitem{jarrett_inverse_2020}
Daniel Jarrett and Mihaela van~der Schaar.
\newblock Inverse {Active} {Sensing}: {Modeling} and {Understanding} {Timely} {Decision}-{Making}.
\newblock {\em arXiv:2006.14141 [cs, stat]}, June 2020.

\bibitem{natarajan_whom_2018}
Sriraam Natarajan, Srijita Das, Nandini Ramanan, Gautam Kunapuli, and Predrag Radivojac.
\newblock On {Whom} {Should} {I} {Perform} this {Lab} {Test} {Next}? {An} {Active} {Feature} {Elicitation} {Approach}.
\newblock In {\em Proceedings of the {Twenty}-{Seventh} {International} {Joint} {Conference} on {Artificial} {Intelligence}}, pages 3498--3505, Stockholm, Sweden, July 2018. International Joint Conferences on Artificial Intelligence Organization.

\bibitem{das_clustering_2021}
Srijita Das, Rishabh Iyer, and Sriraam Natarajan.
\newblock A {Clustering} based {Selection} {Framework} for {Cost} {Aware} and {Test}-time {Feature} {Elicitation}.
\newblock In {\em 8th {ACM} {IKDD} {CODS} and 26th {COMAD}}, pages 20--28. ACM, January 2021.

\bibitem{li_dynamic_2021}
Yang Li and Junier~B. Oliva.
\newblock Dynamic {Feature} {Acquisition} with {Arbitrary} {Conditional} {Flows}.
\newblock {\em arXiv:2006.07701 [cs, stat]}, March 2021.

\bibitem{zhang_novel_2019}
Pin Zhang.
\newblock A novel feature selection method based on global sensitivity analysis with application in machine learning-based prediction model.
\newblock {\em Applied Soft Computing}, 85:105859, 2019.
\newblock Publisher: Elsevier.

\bibitem{gong_icebreaker_2019}
Wenbo Gong, Sebastian Tschiatschek, Sebastian Nowozin, Richard~E Turner, José~Miguel Hernández-Lobato, and Cheng Zhang.
\newblock Icebreaker: {Element}-wise {Efficient} {Information} {Acquisition} with a {Bayesian} {Deep} {Latent} {Gaussian} {Model}.
\newblock In {\em Advances in {Neural} {Information} {Processing} {Systems}}, 2019.

\bibitem{janisch_classification_2020}
Jaromír Janisch, Tomáš Pevný, and Viliam Lisý.
\newblock Classification with costly features as a sequential decision-making problem.
\newblock {\em Machine Learning}, 109(8):1587--1615, August 2020.

\bibitem{xiaoyong_chai_test-cost_2004}
{Xiaoyong Chai}, {Lin Deng}, {Qiang Yang}, and C.~X. Ling.
\newblock Test-cost sensitive naive {Bayes} classification.
\newblock In {\em Fourth {IEEE} {International} {Conference} on {Data} {Mining} ({ICDM}'04)}, pages 51--58, November 2004.

\bibitem{chang_dynamic_2019}
Chun-Hao Chang, Mingjie Mai, and Anna Goldenberg.
\newblock Dynamic {Measurement} {Scheduling} for {Event} {Forecasting} using {Deep} {RL}.
\newblock In {\em Proceedings of the 36th {International} {Conference} on {Machine} {Learning}}, pages 951--960. PMLR, May 2019.

\bibitem{cheng_optimal_2018}
Li-Fang Cheng, Niranjani Prasad, and Barbara~E. Engelhardt.
\newblock An {Optimal} {Policy} for {Patient} {Laboratory} {Tests} in {Intensive} {Care} {Units}.
\newblock In {\em Biocomputing 2019}, pages 320--331. WORLD SCIENTIFIC, October 2018.

\bibitem{an_reinforcement_2022}
Chaojie An, Qifeng Zhou, and Shen Yang.
\newblock A reinforcement learning guided adaptive cost-sensitive feature acquisition method.
\newblock {\em Applied Soft Computing}, page 108437, January 2022.

\bibitem{erion_coai_2021}
Gabriel Erion, Joseph~D. Janizek, Carly Hudelson, Richard~B. Utarnachitt, Andrew~M. McCoy, Michael~R. Sayre, Nathan~J. White, and Su-In Lee.
\newblock {CoAI}: {Cost}-{Aware} {Artificial} {Intelligence} for {Health} {Care}.
\newblock Technical report, medRxiv, January 2021.

\bibitem{bhattacharya_identification_2020}
Rohit Bhattacharya, Razieh Nabi, Ilya Shpitser, and James~M. Robins.
\newblock Identification {In} {Missing} {Data} {Models} {Represented} {By} {Directed} {Acyclic} {Graphs}.
\newblock In {\em Proceedings of {The} 35th {Uncertainty} in {Artificial} {Intelligence} {Conference}}, pages 1149--1158. PMLR, August 2020.

\bibitem{nabi_full_2020}
Razieh Nabi, Rohit Bhattacharya, and Ilya Shpitser.
\newblock Full {Law} {Identification} {In} {Graphical} {Models} {Of} {Missing} {Data}: {Completeness} {Results}.
\newblock {\em arXiv:2004.04872 [cs, stat]}, August 2020.

\bibitem{seaman_review_2013}
Shaun~R. Seaman and Ian~R. White.
\newblock Review of inverse probability weighting for dealing with missing data.
\newblock {\em Statistical methods in medical research}, 22(3):278--295, 2013.
\newblock Publisher: Sage Publications Sage UK: London, England.

\bibitem{hernan_causal_2020}
Miguel~A Hernán and James~M Robins.
\newblock {\em Causal {Inference}: {What} {If}}.
\newblock CRC Boca Raton, FL, 2020.

\bibitem{sterne_multiple_2009}
Jonathan A.~C. Sterne, Ian~R. White, John~B. Carlin, Michael Spratt, Patrick Royston, Michael~G. Kenward, Angela~M. Wood, and James~R. Carpenter.
\newblock Multiple imputation for missing data in epidemiological and clinical research: potential and pitfalls.
\newblock {\em BMJ}, 338:b2393, June 2009.
\newblock Publisher: British Medical Journal Publishing Group Section: Research Methods \&amp; Reporting.

\bibitem{fico_fico_2018}
FICO.
\newblock {FICO} {Explainable} {Machine} {Learning} {Challenge}. https://community.fico.com/s/explainable-machine-learning-challenge, May 2018.
\newblock Last accessed May 2022.

\bibitem{newman_uci_1998}
D.~J. Newman, S.~Hettich, C.~L. Blake, and C.~J. Merz.
\newblock {UCI} {Repository} of machine learning databases, 1998.

\bibitem{mnih_human-level_2015}
Volodymyr Mnih, Koray Kavukcuoglu, David Silver, Andrei~A. Rusu, Joel Veness, Marc~G. Bellemare, Alex Graves, Martin Riedmiller, Andreas~K. Fidjeland, Georg Ostrovski, Stig Petersen, Charles Beattie, Amir Sadik, Ioannis Antonoglou, Helen King, Dharshan Kumaran, Daan Wierstra, Shane Legg, and Demis Hassabis.
\newblock Human-level control through deep reinforcement learning.
\newblock {\em Nature}, 518(7540):529--533, February 2015.

\bibitem{le_morvan_whats_2021}
Marine Le~Morvan, Julie Josse, Erwan Scornet, and Gael Varoquaux.
\newblock What’s a good imputation to predict with missing values?
\newblock In {\em Advances in {Neural} {Information} {Processing} {Systems}}, volume~34, pages 11530--11540, 2021.

\bibitem{ling_decision_2004}
Charles~X. Ling, Qiang Yang, Jianning Wang, and Shichao Zhang.
\newblock Decision trees with minimal costs.
\newblock In {\em Twenty-first international conference on {Machine} learning - {ICML} '04}, page~69, 2004.

\bibitem{sheng_feature_2006}
Victor~S. Sheng and Charles~X. Ling.
\newblock Feature value acquisition in testing: a sequential batch test algorithm.
\newblock In {\em Proceedings of the 23rd international conference on {Machine} learning}, pages 809--816, 2006.

\bibitem{yin_reinforcement_2020}
Haiyan Yin, Yingzhen Li, Sinno~Jialin Pan, Cheng Zhang, and Sebastian Tschiatschek.
\newblock Reinforcement {Learning} with {Efficient} {Active} {Feature} {Acquisition}.
\newblock {\em arXiv:2011.00825 [cs]}, November 2020.

\bibitem{li_active_2021}
Yang Li and Junier Oliva.
\newblock Active {Feature} {Acquisition} with {Generative} {Surrogate} {Models}.
\newblock In {\em Proceedings of the 38th {International} {Conference} on {Machine} {Learning}}, pages 6450--6459. PMLR, July 2021.
\newblock ISSN: 2640-3498.

\bibitem{ma_eddi_2019}
Chao Ma, Sebastian Tschiatschek, Konstantina Palla, Jose~Miguel Hernandez-Lobato, Sebastian Nowozin, and Cheng Zhang.
\newblock {EDDI}: {Efficient} {Dynamic} {Discovery} of {High}-{Value} {Information} with {Partial} {VAE}.
\newblock In {\em Proceedings of the 36th {International} {Conference} on {Machine} {Learning}}, pages 4234--4243. PMLR, May 2019.

\bibitem{shim_joint_2018}
Hajin Shim, Sung~Ju Hwang, and Eunho Yang.
\newblock Joint {Active} {Feature} {Acquisition} and {Classification} with {Variable}-{Size} {Set} {Encoding}.
\newblock {\em Advances in Neural Information Processing Systems}, 31, 2018.

\bibitem{tsiatis_semiparametric_2006}
Anastasios~A. Tsiatis.
\newblock {\em Semiparametric theory and missing data}.
\newblock Springer series in statistics. Springer, New York, 2006.

\bibitem{bickel_efficient_1993}
Peter~J. Bickel, Chris~AJ Klaassen, Peter~J. Bickel, Ya’acov Ritov, J.~Klaassen, Jon~A. Wellner, and YA'Acov Ritov.
\newblock {\em Efficient and adaptive estimation for semiparametric models}, volume~4.
\newblock Springer, 1993.

\bibitem{kennedy_semiparametric_2017}
Edward~H. Kennedy.
\newblock Semiparametric theory.
\newblock {\em arXiv:1709.06418 [stat]}, September 2017.
\newblock arXiv: 1709.06418.

\bibitem{robins_new_1986}
James Robins.
\newblock A new approach to causal inference in mortality studies with a sustained exposure period—application to control of the healthy worker survivor effect.
\newblock {\em Mathematical Modelling}, 7(9):1393--1512, January 1986.

\bibitem{mohan_graphical_2013}
Karthika Mohan, Judea Pearl, and Jin Tian.
\newblock Graphical {Models} for {Inference} with {Missing} {Data}.
\newblock In {\em Advances in {Neural} {Information} {Processing} {Systems}}, volume~26, 2013.

\bibitem{shpitser_missing_2015}
Ilya Shpitser, Karthika Mohan, and Judea Pearl.
\newblock Missing data as a causal and probabilistic problem.
\newblock In {\em Proceedings of the {Thirty}-{First} {Conference} on {Uncertainty} in {Artificial} {Intelligence}}, {UAI}'15, pages 802--811, Arlington, Virginia, USA, July 2015. AUAI Press.

\bibitem{chen_pattern_2020}
Yen-Chi Chen.
\newblock Pattern graphs: a graphical approach to nonmonotone missing data, December 2020.
\newblock arXiv:2004.00744 [math, stat].

\bibitem{zamanian_assessable_2023}
Alireza Zamanian, Narges Ahmidi, and Mathias Drton.
\newblock Assessable and interpretable sensitivity analysis in the pattern graph framework for nonignorable missingness mechanisms.
\newblock {\em Statistics in Medicine}, 42(29):5419--5450, 2023.
\newblock \_eprint: https://onlinelibrary.wiley.com/doi/pdf/10.1002/sim.9920.

\bibitem{kennedy_semiparametric_2022}
Edward~H. Kennedy.
\newblock Semiparametric doubly robust targeted double machine learning: a review.
\newblock {\em arXiv:2203.06469 [stat]}, March 2022.
\newblock arXiv: 2203.06469.

\bibitem{pedregosa_scikit-learn_2011}
F.~Pedregosa, G.~Varoquaux, A.~Gramfort, V.~Michel, B.~Thirion, O.~Grisel, M.~Blondel, P.~Prettenhofer, R.~Weiss, V.~Dubourg, J.~Vanderplas, A.~Passos, D.~Cournapeau, M.~Brucher, M.~Perrot, and E.~Duchesnay.
\newblock Scikit-learn: {Machine} {Learning} in {Python}.
\newblock {\em Journal of Machine Learning Research}, 12:2825--2830, 2011.

\end{thebibliography}


\end{document}